\documentclass[lettersize,journal]{IEEEtran}
\usepackage{amsmath,amsfonts}
\usepackage{array}
\usepackage[caption=false,font=normalsize,labelfont=sf,textfont=sf]{subfig}
\usepackage{textcomp}
\usepackage{float} 
\usepackage{url}
\usepackage{verbatim}
\usepackage{cite}
\usepackage{multirow}
\usepackage{amsthm}
\usepackage{times}
\usepackage{epsfig}
\usepackage{amsmath}
\usepackage{amssymb}
\usepackage{bm}
\newtheorem{corollary}{Corollary}
\newtheorem{theorem}{Theorem}

\usepackage{threeparttable}
\usepackage{tikz}
\usepackage{subcaption,graphicx}
\usepackage{verbatim}
\usepackage{algorithm}
 \usepackage{algpseudocode}
 \usepackage{algorithmicx}
  \algdef{SE}[DOWHILE]{Do}{doWhile}{\algorithmicdo}[1]{\algorithmicwhile\ #1}%
\definecolor{Red}{rgb}{1,0,0}

\newtheorem{remark}{Remark}

\usepackage{hyperref}

\algrenewcommand\algorithmicrequire{\textbf{Input:}}
\algrenewcommand\algorithmicensure{\textbf{Output:}}

\begin{document}

\title{Non-Rigid Structure-from-Motion via Differential Geometry with Recoverable Conformal Scale}

\author{Yongbo Chen,~\IEEEmembership{Member,~IEEE}, Yanhao Zhang, Shaifali Parashar, Liang Zhao,~\IEEEmembership{Member,~IEEE}, Shoudong Huang,~\IEEEmembership{Senior Member,~IEEE}
\thanks{Manuscript received January 14, 2025; revised May 26, 2025; accepted September 18, 2025. Date of publication; date of current version. This work was supported by the Australia Research Council (ARC) Discovery grant Project DP120102786. This paper was recommended for publication by Editor Civera, Javier and Editor-in-Chief Burgard, Wolfram upon evaluation of the reviewers' comments. (Corresponding author: Yanhao Zhang.)}

\thanks{Yongbo Chen was with the Robotics Institute, University of Technology Sydney, Australia, and is now with the School of Automation and Intelligent Sensing, Shanghai Jiao Tong University, Shanghai, 200240, People’s Republic of China (e-mail: shjtdx\_cyb@sjtu.edu.cn).}

\thanks{Yanhao Zhang and Shoudong Huang are with the Robotics Institute, University of Technology Sydney, Ultimo, NSW 2007, Australia (e-mail:
yanhaozhang1991@gmail.com, Shoudong.Huang@uts.edu.au).}

\thanks{Shaifali Parashar is now with Institut National des Sciences Appliqu\'{e}es de Lyon (LIRIS, INSA-Lyon), 69100 Villeurbanne, France. (e-mail: shaifali.parashar@liris.cnrs.fr)}

\thanks{Liang Zhao is with the School of Informatics, University of Edinburgh, Edinburgh, UK (e-mail: Liang.Zhao@ed.ac.uk).}

\thanks{This work has been accepted for publication in the IEEE Transactions on Robotics (T-RO).}%
\thanks{\copyright~2025 IEEE. Personal use of this material is permitted.
    Permission from IEEE must be obtained for all other uses, in any current or
    future media, including reprinting/republishing this material for advertising
    or promotional purposes, creating new collective works, for resale or
    redistribution to servers or lists, or reuse of any copyrighted component of
    this work in other works.}%

\thanks{This article has supplementary material provided by the authors and color versions of one or more figures available at https://doi.org.}
\thanks{Digital Object Identifier.}
}

\markboth{IEEE TRANSACTIONS ON ROBOTICS,~Vol.~, No.~, September~2025}%
{Shell \MakeLowercase{\textit{et al.}}: A Sample Article Using IEEEtran.cls for IEEE Journals}


\maketitle

\begin{abstract}
Non-rigid structure-from-motion (NRSfM), a promising technique for addressing the mapping challenges in monocular visual deformable simultaneous localization and mapping (SLAM), has attracted growing attention. We introduce a novel method, called Con-NRSfM, for NRSfM under conformal deformations, encompassing isometric deformations as a subset. Our approach performs point-wise reconstruction using 2D selected image warps optimized through a graph-based framework. Unlike existing methods that rely on strict assumptions, such as locally planar surfaces or locally linear deformations, and fail to recover the conformal scale, our method eliminates these constraints and accurately computes the local conformal scale. Additionally, our framework decouples constraints on depth and conformal scale, which are inseparable in other approaches, enabling more precise depth estimation. To address the sensitivity of the formulated problem, we employ a parallel separable iterative optimization strategy. Furthermore, a self-supervised learning framework, utilizing an encoder-decoder network, is incorporated to generate dense 3D point clouds with texture. Simulation and experimental results using both synthetic and real datasets demonstrate that our method surpasses existing approaches in terms of reconstruction accuracy and robustness. The code for the proposed method will be made publicly available on the project website: \url{https://sites.google.com/view/con-nrsfm}.
\end{abstract}


\section{Introduction}
\IEEEPARstart{N}{on-rigid} structure-from-motion (NRSfM) addresses the challenge of reconstructing the 3D shapes of deforming objects from multiple calibrated monocular images. It is a fundamental problem in 3D computer vision, with applications ranging from entertainment~\cite{1add} to modern surgery~\cite{8}. The NRSfM algorithm reconstructs 3D shapes in local camera coordinates, inherently intertwining camera motion with object deformations. This concept closely aligns with deformable visual SLAM, which aims to localize a robot and map its environment, even in dynamically deforming scenarios. NRSfM has the potential to play a pivotal role in overcoming the mapping challenges associated with deformable visual SLAM~\cite{8}. By integrating tools for robot pose estimation, NRSfM can enhance mapping consistency through information fusion, paving the way for significant advancements in deformable visual SLAM research~\cite{8add, 8add1}.

Naturally, the NRSfM problem is unsolvable without introducing constraints or assumptions on deformations. Common approaches to deformation modeling include statistical constraints, such as low-rank shapes~\cite{1} or trajectory basis~\cite{2}, and physical constraints, such as isometry (preserving local geodesics)~\cite{3}, inextensibility~\cite{4}, and local rigidity~\cite{5}. Among these, physical methods generally outperform statistical ones in reconstructing highly deformable objects.

Recent advancements have introduced local formulations of physical constraints, including isometry (distance-preserving)\cite{6}, conformality (angle-preserving)\cite{7}, and smoothness~\cite{10}. These methods assume surfaces to be locally planar (LP) and deformations to be locally linear (LL), expressing physical constraints in terms of local depth derivatives as unknowns. This approach offers several advantages: it is often more accurate, computationally efficient, and robust to missing data. However, there are two major limitations: \emph{1) indirect depth computation:} all constraints are derived on local normals (expressed with depth derivatives) and the depth is only obtained by interpolating them. Such a formulation leads to weaker constraints on the surface geometry which affects their performance. \emph{2) LP and LL assumptions:} While LP and LL assumptions are valid for continuous surfaces, these methods often operate on sparsely sampled points (100–200 points) on the surface. This sparse sampling compromises the accuracy of local derivative computations, leading to degraded performance, especially in scenarios involving strong deformations, such as surface bending. These limitations emphasize the need for further refinement to improve accuracy and robustness in handling complex deformations.

\IEEEpubidadjcol

In this paper, we extend the existing local formulations of isometry/conformality to overcome these limitations and develop a novel NRSfM method tailored for conformal deformations. The main contributions of this paper are:

\begin{itemize}
    \item \textbf{Rotational invariance under conformal deformation}: We establish that connections and moving frames across surfaces preserve distinct invariance properties under different types of deformations. Critically, we prove that connections under conformal deformation preserve rotational invariance, enabling the decoupling of conformal scale and depth estimation. This introduces a novel and strict theoretical constraint on the geometry of deformable surfaces. As a result, we derive physical constraints in terms of conformal scale, depth, and normals (expressed as depth derivatives).
    \item \textbf{Relaxation of LL and LP assumptions:} Unlike prior approaches, we do not impose LL or LP assumptions. Instead, we define physical constraints that hold up to the second-order derivatives of local depth, introducing additional variables to better align the formulation with real-world scenarios. To solve this, we propose a parallel, separable, iterative optimization algorithm that independently recovers depth and normals (depth derivatives). The algorithm is robust to initialization and offers acceptable computational complexity.
\end{itemize}
Meanwhile, we develop an encoder-decoder neural network trained using a self-supervised learning approach on simulated datasets. This network reconstructs 3D dense point clouds with texture from the normal field, offering a more comprehensive representation of deformable surfaces. Our experiments show that our proposed method, Con-NRSfM, outperforms the existing state-of-the-art (SOTA) methods, especially when the deformations are non-isometric and/or cause strong bending.

\section{Related work}\label{s2}
\subsection{NRSfM methods}
We review the existing NRSfM methods in terms of the deformation modeling they impose.

\textbf{i) Statistical constraints}:
Starting from~\cite{5,13}, the researches focus on the factorization-based methods with low-rank assumptions. They consider instantaneous 3D shape of a deforming object to be a linear combination of a small size shape \cite{19, 19add}, trajectory~\cite{18}, or force basis~\cite{20,21}. Assuming that deforming shapes are viewed under affine projections~\cite{4}, the additional low-dimensional priors include the trajectory basis priors~\cite{18}, non-linear modeling~\cite{14}, spatial smoothness~\cite{15}, spatio-temporal smoothness~\cite{16}, temporal smoothness~\cite{16add}, and a quadratic deformation model~\cite{17}.  
Recently,~\cite{56, 57, 57add} extend this framework to neural networks which could learn the shape-basis from a small set of images. However, these methods work on sparse data such as simple wireframe objects.
~\cite{21add} introduced the first unsupervised, end-to-end NRSfM approach which uses the low-rank constraints to reconstruct dense data. These methods perform effectively on objects with simple deformations, but struggle to handle more complex or severe deformations.




\textbf{ii)~Global physical constraints}: Global physical methods aim to preserve the physical properties of the entire point cloud and its corresponding surface. For instance, inextensibility—a convex relaxation of isometry—has been used to develop a convex second-order cone programming (SOCP) formulation for isometric NRSfM using point parameterization~\cite{4}. Similarly, the NRSfM problem has been formulated as a convex semi-definite program (SDP)~\cite{25}, incorporating constraints based on the cosine law with edge parameterization to minimize non-rigidity. While these global physical methods can achieve satisfactory accuracy in certain scenarios, they often suffer from high computational complexity and limited robustness, making them less practical for  real-world applications.

\begin{table*}[!ht]
	\small
	\caption{{Comparison of our method with existing local methods}}
	\label{t1zz112}
	\begin{center}
		\begin{tabular}{|c|c|c|c|c|c|}
			\hline
			{\textbf{Method}}  &{\textbf{Deformation modeling}}  &{\textbf{Variables}}  & {\textbf{Assumptions}}&{\textbf{Constraints}}  & {\textbf{Surface representation}} \\
			\cline{1-6}
			\cite{6}&Isometry&Normals&LP&MT, CC&Inverse-depth\\
			\cline{1-6}
			\cite{7}& \begin{tabular}{@{}c@{}}Isometry, Conformality, \\ Equiareality\end{tabular}&Normals&LP+LL&MT, CC&Inverse-depth\\
			\cline{1-6}
			\cite{10}&Diffeomorphism &Normals&LL&CC&Depth\\
			\cline{1-6}
			{\cite{31}}&{Isometry, Conformality} &{Normals}&{LP+LL}&{MT, CC}&{Inverse-depth}\\
            \cline{1-6}
			Con-NRSfM&Isometry, Conformality&Depth, Normals, Conformal scale&None&MT, CC&Depth\\
			\hline
		\end{tabular}
	\end{center}
\end{table*}

\textbf{iii)~Piece-wise physical constraints}: Piece-wise physical methods approximate the shape of small surface regions using simple models, recovering depth for each region and subsequently stitching the reconstructions together to maintain surface continuity. For example, in~\cite{23} and~\cite{24}, NRSfM problems are addressed using orthographic and perspective camera models, respectively, based on piece-wise rigidity. However, these methods face significant challenges due to the high computational cost of region segmentation. Additionally, defining optimal segmentation for generic surfaces is difficult, leading to inefficiencies and reduced accuracy~\cite{4}. More recently,\cite{4add} introduced a method that utilizes specular highlights in images as geometric and photometric cues, adding constraints to the previous isometric-inextensible NRSfM model\cite{4}. This approach enhances the robustness of the reconstruction process by incorporating additional visual information.

\textbf{iv)~Local physical constraints}: Local physical methods represent an object's 3D deformable shape in each image as a smooth Riemannian manifold, assuming inter-image registration (image warp) is available. These methods enforce physical constraints, including isometry \cite{9,22}, conformality, equiareality~\cite{7}, and diffeomorphism~\cite{10}, to restrict deformations between surfaces. They leverage metric tensors (MT) to measure local distances, angles, and areas, and Cartan’s connections (CC) to assess local curvatures. Using linear relationships between local normals across surfaces (LL and LP), these methods reduce the number of variables, enabling rapid computation of local normals. The surfaces are then reconstructed by integrating the local normals, which is computationally more intensive. The local formulation also efficiently handles noisy or missing data, reconstructing various deformable objects from videos or sparse images. Con-NRSfM extends this approach by imposing conformal constraints on local depth in addition to normals, without relying on the surface assumptions of LL and LP. A recent work~\cite{31} addresses the conformal case by relying on LP and LL assumptions using existing connection constraints to derive surface normals via a closed-form formulation. In contrast, our method theoretically relaxes these assumptions, introduces a novel connection constraint, and employs a more robust and scalable optimization framework.
Table~\ref{t1zz112} highlights the differences between Con-NRSfM and other methods. Furthermore, Con-NRSfM introduces a parallel, separable optimization framework and extends this strategy using a deep learning network capable of recovering 3D shapes with texture.
\subsection{Related applications}
In this section, we review and discuss the challenges and recent adaptations of  NRSfM in practical scenarios, such as monocular deformable SLAM and visual endoluminal robotic navigation, to highlight its real-world significance. Monocular deformable SLAM is a novel and challenging problem where most classical SLAM methods exhibit low accuracy or even fail. This is primarily due to the absence of direct depth measurements and the violation of the rigid and static scene assumptions traditionally used for observed features. Def-SLAM~\cite{8} is the first monocular SLAM system explicitly designed for deformable environments. It constructs and incrementally updates a deformable map by incorporating key modifications into the ORB-SLAM framework~\cite{9adds} to handle non-rigid scenes. An isometric NRSfM approach~\cite{6} is embedded within the mapping module to generate surface templates at keyframe intervals. These templates are then aligned in the tracking module using a Shape-from-Template (SfT) method and a deformation energy function~\cite{9adds1111}, enabling simultaneous estimation of both camera trajectory and surface deformations in real time. SfT methods, which reconstruct the deformed shape of an object from a single monocular image using a known 3D template in its rest configuration, are conceptually aligned with NRSfM and often achieve faster performance. As demonstrated in Def-SLAM, SfT can serve as a complementary tool to NRSfM by refining deformation estimates when the template is available through NRSfM over time. Building on similar techniques, SD-DefSLAM~\cite{9} incorporates optical flow to address the data association problem in the SLAM pipeline, significantly improving robustness and accuracy. This enhancement makes it particularly suitable for visual endoluminal robotic navigation in monocular medical imaging. More recently, in minimally invasive procedures, the work in~\cite{9adds3333} combines an embedded deformation graph (EDG) with isometric NRSfM to enable centralized optimization of both map points and camera motion over time. This approach further improves scene reconstruction accuracy in deformable soft tissue using monocular endoscopy.

\section{Mathematical Modeling Preliminary}\label{s3}
\subsection{Notations}
Throughout this paper, unless otherwise {stated}, bold lowercase, and bold uppercase letters are reserved for
vectors and matrices, respectively. Sets and spaces are shown by uppercase letters with mathcal font.  $\bm{I}_{n\times n}$ is the  identity matrix with $n$ dimension. $SO(n)$ (special orthogonal group) is defined as: $SO(n)\triangleq\{\bm{R}\in \mathbb{R}^{n\times n}: \bm{R}^\top\bm{R} = \bm{I}_{n\times n},~\det(\bm{R})=1\}$. $(\star)^\top$ means the transpose of   matrix and vector $\star$.
$\det(\star)$ represents the determinant of the matrix $\star$.  $\star \times \bullet$ means the direct product group of $\star$ and $\bullet$. For connection $\Gamma$, its superscript, first subscript, and second subscript respectively represent the row, column, and block numbers. If any of the indices is written as a variable (e.g., $i,~j,~k$), it indicates that we are considering the entire sub-vector or sub-matrix corresponding to all possible values of that index (all choices along this script), like $\Gamma^1_{j1} \in \mathbb{R}^{3\times1}\subset\Gamma^i_{j1} \in \mathbb{R}^{3\times3}\subset\Gamma^i_{jk} = \Gamma\in \mathbb{R}^{3\times6}$. If the indices are specified as integers (e.g., 1,~2), they indicate the specific, fixed part of the group or matrix being referenced. For a mapping  $\bullet$, $\star (\bullet)$ means the corresponding concepts, like connection $\Gamma$ and moving frame $E$, applied on  mapping $\bullet$.


\subsection{Perspective Projection and Image Embedding}\label{s3b}
	\begin{figure}[!htb]
	\begin{center}
		\includegraphics[width=0.8\linewidth]{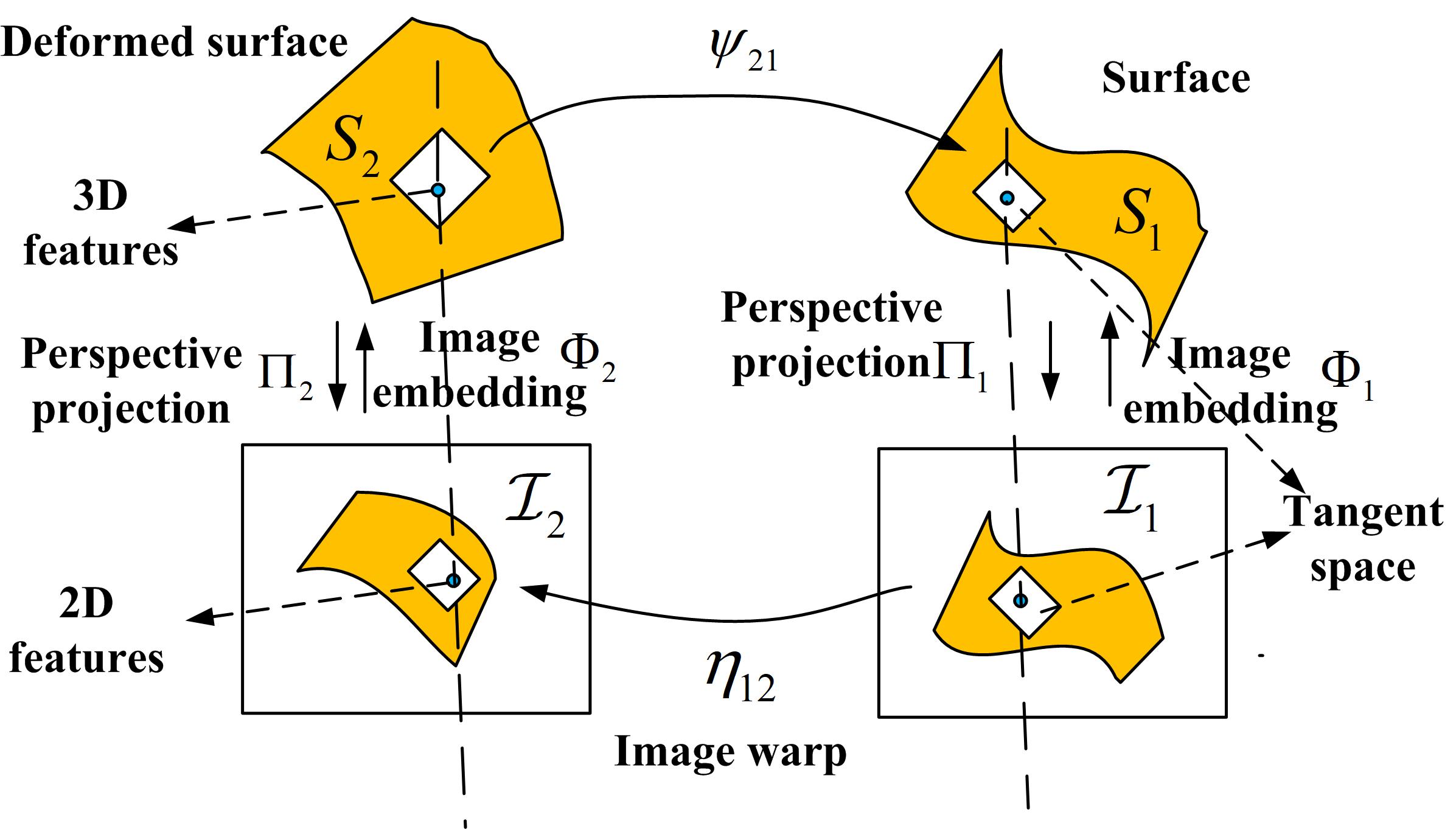}
	\end{center}
	\caption{A 2-view model for NRSfM. }
	\label{figurelabe1}
\end{figure}
As shown in \figurename~\ref{figurelabe1}, assuming that the considered multiple deformation
surfaces are Riemannian manifolds, based on the calibrated monocular camera, we are considering the mappings between $i$-th and $j$-th  images $\mathcal{I}_i,~\mathcal{I}_j$ and their corresponding surfaces $\mathcal{S}_i,~\mathcal{S}_j$. The camera mapping from the surfaces to the images is the perspective projection $\Pi_i:~\mathcal{S}_i\in\mathbb{R}^3\rightarrow\mathcal{I}_i\in\mathbb{R}^2 $ and its inverse mapping is the image embedding, which can be assumed as the smooth function of which the variables are the pixels $(u,v)^\top$ on the images $\Phi_i:~\mathcal{I}_i\rightarrow \mathcal{S}_i$. For the perspective camera, the perspective projection $\Pi$ and the image embedding $\Phi$ are respectively denoted as\footnote{For brevity, we have ignored the subscript $i$ which means the index of the surface and the image. }:
\begin{equation}\label{E1}
	\begin{aligned}
		\bm{x}&=(u,v)^\top=\Pi(\bm{z})=(z_1/z_3,z_2/z_3)^\top,\\
		\bm{z}&=(z_1,z_2,z_3)^\top=\Phi(\bm{x})={\beta(u,v)}(u, v,1)^\top,
	\end{aligned}
\end{equation}
where $\bm{z}=(z_1,z_2,z_3)^\top,~z_3>0$ is a 3D feature on the surface, $\beta(u,v)$ means the smooth depth function of the pixel $(u,v)^\top$.\footnote{For brevity, we will ignore the variable $(u,v)$ and $\beta(u,v)$ is written as $\beta$ in the following text.} In this paper, the mappings $\Pi$ and $\Phi$ are represented using B-splines~\cite{25a}, which allows us to accurately obtain interpolation value, first- and second-order derivatives of these functions.
\subsection{Deformation Mapping}\label{s32}
The partial goal of the NRSfM method is to obtain the mappings $\Psi_{ij}: \mathcal{S}_i\rightarrow\mathcal{S}_j$ among the deformable surfaces. For the generic deformation, the NRSfM problem is not solvable. So as to limit the solution, the typical local assumptions on the deformation are restricted in the diffeomorphic (isometry, conformality, and equiareality).  According to~\cite{10}, for the sparse matched features, if $\Psi_{ij}$ results in the  diffeomorphic deformation of the surface, its approximated Jacobian matrix follows: $\mathbf{J}_{\Psi_{ij}}=\mbox{diag}(\lambda_1,\lambda_2,\lambda_3)\mathbf{R}$, where $\mbox{diag}(\lambda_1,\lambda_2,\lambda_3)$ means the scale matrix, $\lambda_i$ are scalars, and $\mathbf{R}\in SO(3)$ is a rotation matrix.\footnote{This property follows an approximation definition of the diffeomorphic mapping assuming the off-diagonal elements to be zero. The more general diffeomorphic mapping has six degrees of freedom for the scale matrix.} Following the same approximation, the Jacobian matrices of the isometric and conformal deformations will be degenerated to $\mathbf{J}_{\Psi_{ij}}=\mbox{diag}(1,~1,~1)\mathbf{R}=\mathbf{R}$ and $\mathbf{J}_{\Psi_{ij}}=\mbox{diag}(\lambda,~\lambda,~\lambda)\mathbf{R}=\lambda\mathbf{R}$ respectively, where $\lambda > 0$ is a conformal scale factor. It is noted that the conformal scale $\lambda$ is not the translation scale between different frames for monocular images corresponding to scale ambiguity. The scale ambiguity is solved by Step 3 in  Section~\ref{s52}.

\subsection{Selected Image Warp}\label{s33}
Based on the matched features obtained by the standard image matching methods, such as optical flow~\cite{26}, Scale Invariant Feature Transform (SIFT)~\cite{27}, and Speeded Up Robust Features (SURF)~\cite{28}, we can define a dense geometric transformation functions $\eta_{ij}: \mathbb{R}^2\rightarrow\mathbb{R}^2$ between the pair of images $\mathcal{I}_i$ and $\mathcal{I}_j$, called image warp. Our local physical method belongs to the point-wise method. The main part to share the information connection between different features is the image warp, so its accuracy and robustness are very important for the whole NRSfM framework. Each pair of images is possible to be used to compute the image warp. Considering the large computational cost of the image warp, a complete weighted undirected graph $\mathcal{G}_c=(\mathcal{V}_c,~\mathcal{E}_c,~w_c)$,  of which the nodes, the edges, and the weighted values are the images, the image warp, and the number of the matched features, is introduced to select the well-connected sub-graph and only the image warps corresponding to the edges of the selected sub-graph are computed. The tree-connectivity $t_w(\mathcal{G}) = \det(\bm{\mathcal{L}}_{w}^{\mathcal{G}})$ of the connected weighted undirected sub-graph $\mathcal{G}=(\mathcal{V},~\mathcal{E},~w)\subseteq \mathcal{G}_c$, which is the log-determinant function of the weighted Laplacian matrix, is used as the objective function of the selection method. Given the limited number of the edges $N_e$, the selection problem can be considered as the problem to select maximum spanning tree,\footnote{It is noted that, in order to recover the depth of the features in all the images, $N_e$ needs to satisfy $N_e\geq n_c-1$.} if $N_e= n_c-1$, and the famous $k$-edge selection problem ($k$-ESP,~ $k=N_e-n_c+1$), if $N_e> n_c-1$~\cite{29}. The corresponding edge selection problem is formulated as:
\begin{equation}\label{E10}
	\begin{aligned}
		&\max \limits_{\mathcal{G}}t_w(\mathcal{G}),~s.t.~|\mathcal{G}|=N_e,~\mathcal{E}\subseteq\mathcal{E}_c,~\mathcal{V}=\mathcal{V}_c.
	\end{aligned}
\end{equation}
This sub-graph selection problem is solved by the combination of  Kruskal’s maximum spanning tree algorithm (for maximum spanning tree problem, obtain $\mathcal{T}_{opt}$) and the greedy-based method with the rank-1 update (for $k$-ESP, obtain $\mathcal{G}_{opt}$)~\cite{3,38}. As a sub-modular maximization problem with a cardinality constraint, the greedy-based method for the $k$-ESP has some performance guarantee. Then, a well-connected sub-graph $\mathcal{G}_{opt}$ is obtained and its corresponding edges are used to calculate the image warps based on 2D Schwarzian derivatives under the infinitesimal planarity assumption~\cite{5}. The obtained results provide us the first-/second-order derivatives of the image warps and their inverse mappings.

\subsection{Moving Frames and Connections}
For the image embedding, a local frame of reference, \emph{moving frame}, $E=(\bm{e}_1,\bm{e}_2,\bm{e}_3)$, at surface $\mathcal{S}$ expressed in terms of $\Phi$ is given by
\begin{equation}\label{eq3}
	E=\left(\bm{e}_1=\frac{\partial~\Phi}{\partial~u},~\bm{e}_2=\frac{\partial~\Phi}{\partial~v},~\bm{e}_3=\bm{e}_1\times \bm{e}_2\right).
\end{equation}

Introducing the definition~\eqref{E1} of the image embedding $\Phi$ into the above equation, we have:
\begin{equation}\label{eq3add}
	\begin{aligned}
	&\bm{e}_1=\beta(y_1u+1,y_1v,y_1)^\top,\\
	&\bm{e}_2=\beta(y_2u,y_2v+1,y_2)^\top,\\
	&\bm{e}_3=\beta^2(-y_1,-y_2,y_1u+y_2v+1)^\top,
	\end{aligned}
\end{equation}
 where $y_1=\frac{1}{\beta}\frac{\partial~\beta}{\partial~u}$ and $y_2=\frac{1}{\beta}\frac{\partial~\beta}{\partial~v}$.

The local change in \emph{moving frames} is a linear function. The  coefficients $\Gamma^i_{jk}$, known as \emph{connections}, describe the fundamental  properties of the surfaces. 
\begin{equation}\label{eq1}
	\begin{aligned}
		\frac{\partial~\bm{e}_j}{\partial~u}=&\Gamma^1_{j1}\bm{e}_1+\Gamma^2_{j1}\bm{e}_2+\Gamma^3_{j1}\bm{e}_3,~j=1,2,3\\\frac{\partial~\bm{e}_j}{\partial~v}=&\Gamma^1_{j2}\bm{e}_1+\Gamma^2_{j2}\bm{e}_2+\Gamma^3_{j2}\bm{e}_3.
	\end{aligned}
\end{equation}

\subsection{Metric Preservation under Conformal Deformation }
In \figurename~\ref{figurelabe1}, considering two features $\bar{\bm{x}}\in\mathcal{I}_1$ and $\bm{x}\in\mathcal{I}_2$ on two images $\mathcal{I}_1$ and $\mathcal{I}_2$, their corresponding 3D features $\bar{\bm{z}}\in\mathcal{S}_1$ and ${\bm{z}}\in\mathcal{S}_2$ on surfaces $\mathcal{S}_1$ and $\mathcal{S}_2$, and the mappings (image embedding $\Phi_1:~\bar{\bm{x}}\rightarrow \bar{\bm{z}}$ and $\Phi_2:~{\bm{x}}\rightarrow {\bm{z}}$; image warp $\eta_{12}:\bar{\bm{x}}\rightarrow{\bm{x}}$; deformation mapping $\Psi_{21}: {\bm{z}}\rightarrow\bar{\bm{z}}$) between them. We can get the maps and their Jacobian matrices following:
\begin{equation}\label{eq6}
	\begin{aligned}
	\Phi_1=\Psi_{21}\circ\Phi_2\circ\eta_{12}, \quad \mathbf{J}_{\Phi_1}=\mathbf{J}_{\Psi_{21}}\mathbf{J}_{\Phi_2}\mathbf{J}_{\eta_{12}}.
	\end{aligned}
\end{equation}
Under a local conformal deformation, we have $\mathbf{J}_{\Psi_{21}} = \lambda \mathbf{R}$. If $\lambda = 1$, the deformations are isometric. The 2D part of the moving frames of the mappings $\Phi_1$ and $\Psi_{21}\circ\Phi_2\circ\eta_{12}$ satisfy the following metric preservation equation~\cite{10}:

\begin{equation}\label{eq7a}
	\begin{aligned}
		&{\mathbf{J}_{\Phi_1}}^\top {\mathbf{J}_{\Phi_1}}=\lambda^2{\mathbf{J}_{\eta_{12}}}^\top{\mathbf{J}_{\Phi_2}}^\top\mathbf{J}_{\Phi_2}\mathbf{J}_{\eta_{12}}.\\
	\end{aligned}
\end{equation}

\section{Connections under Conformal Deformation}\label{s4}
In this section, we present all the novel conclusions, which are different from the ones in the literature, about the connection under conformal deformation.

Assuming that the depth function $\beta$ is second-order differentiable at every pixel, based on the definitions \eqref{eq3add} and \eqref{eq1}, we have:
\begin{equation}\label{eq5}
	\begin{aligned}
		&\left(\begin{array}{cccccc}\Gamma^1_{11}&\Gamma^1_{21}&\Gamma^1_{31}&\Gamma^1_{12}&\Gamma^1_{22}&\Gamma^1_{32}\\\Gamma^2_{11}&\Gamma^2_{21}&\Gamma^2_{31}&\Gamma^2_{12}&\Gamma^2_{22}&\Gamma^2_{32}\\\Gamma^3_{11}&\Gamma^3_{21}&\Gamma^3_{31}&\Gamma^3_{12}&\Gamma^3_{22}&\Gamma^3_{32}\end{array}\right)=\\
		&\left(\begin{array}{ccc}\beta_1u+\beta&\beta_2u&-\beta\beta_1\\\beta_1v&\beta+\beta_2v&-\beta\beta_2\\\beta_1&\beta_2&\beta\beta_1u+\beta\beta_2v+\beta^2\end{array}\right)^{-1}\\
			\end{aligned}
\end{equation}
\begin{equation}\nonumber
\begin{aligned}
		&\left(\begin{array}{ccc}\beta_{11}u+2\beta_1&\beta_2+\beta_{12}u&-\beta_{1}^2-\beta\beta_{11}\\\beta_{11}v&\beta_{12}v+\beta_{1}&-\beta_{1}\beta_{2}-\beta\beta_{12}\\\beta_{11}&\beta_{12}&T_1\end{array}\right.\\&\left.\begin{array}{ccc}\beta_{12}u+\beta_{2}&\beta_{22}u&-\beta\beta_{12}-\beta_1\beta_2\\\beta_{12}v+\beta_{1}&\beta_{22}v+2\beta_{2}&-\beta\beta_{22}-\beta_{2}^2\\\beta_{12}&\beta_{22}&T_2\end{array}\right),
	\end{aligned}
\end{equation}
where

\begin{equation}\label{eq3add_new11}
	\begin{aligned}&T_1=3\beta\beta_1+u\beta_1^2+u\beta\beta_{11}+v\beta_{12}\beta+v\beta_1\beta_2,\\ &T_2=u\beta_2\beta_1+u\beta\beta_{12}+3\beta\beta_2+v\beta_2^2+v\beta\beta_{22},\\
    &\beta_1=\frac{\partial~\beta}{\partial~u}=\beta y_1,~\beta_2=\frac{\partial~\beta}{\partial~v}=\beta y_2,\\
    &\beta_{11}=\frac{\partial^2~\beta}{\partial^2~u}=\beta y_{11},~\beta_{22}=\frac{\partial^2~\beta}{\partial^2~v}=\beta y_{22},\\
    &\beta_{12}=\frac{\partial^2~\beta}{\partial~u~\partial~v}=\frac{\partial^2~\beta}{\partial~v~\partial~u}=\beta y_{12}=\beta y_{21}.
	\end{aligned}
\end{equation}

Unlike the approaches in \cite{10} and \cite{23}, which assume infinitesimally planar surfaces and simplify their formulations by neglecting second-order derivatives, our method retains these terms to enable exact computation.

 Given  the moving frames {$\overline{E}(\Phi_1)={E}(\Psi_{21}\circ\Phi_2\circ\eta_{12})$ $=(\bar{\bm{e}}_1^*,\bar{\bm{e}}_2^*,\bar{\bm{e}}_3^*),~E(\Phi_2\circ\eta_{12})=(\bm{e}_1^*,\bm{e}_2^*,\bm{e}_3^*),~{E}(\Phi_2)=({\bm{e}}_1,$ ${\bm{e}}_2,{\bm{e}}_3)$} on $\mathcal{S}_1$ and $\mathcal{S}_2$ with image coordinates $\bar{\bm{x}}=(\overline{u},\overline{v})$ and ${\bm{x}}=(u,v)$ respectively,  we write $\mathbf{J}_{\Phi_1} = (\bar{\bm{e}}_1^*,\bar{\bm{e}}_2^*)$ and $\mathbf{J}_{\Phi_2} =$ $ ({\bm{e}}_1,{\bm{e}}_2)$. Using~\eqref{eq3}, the relation between moving frames is:
\begin{equation}\label{eq7}
	\begin{aligned}
		&(\bm{e}_1^*,\bm{e}_2^*)=\mathbf{J}_{\Phi_2}\mathbf{J}_{\eta_{12}},~\bm{e}_3^*=\bm{e}_1^*\times \bm{e}_2^*=\det(\mathbf{J}_{\eta_{12}})\bm{e}_3,\\
		&E(\Phi_2\circ\eta_{12})=(\mathbf{J}_{\Phi_2}\mathbf{J}_{\eta_{12}},\bm{e}_3\det(\mathbf{J}_{\eta_{12}}))=E(\Phi_2){\mathbf{J}_{\eta}}_{3\times 3},\\
		&(\bar{\bm{e}}_1^*,\bar{\bm{e}}_2^*)=\lambda\mathbf{R}(\bm{e}_1^*, \bm{e}_2^*),~\bar{\bm{e}}_3^*=\bar{\bm{e}}_1^*\times\bar{\bm{e}}_2^*=\mathbf{R}\lambda^2 \bm{e}_3^*,\\
		&\overline{E}(\Phi_1)=\mathbf{R}E(\Phi_2\circ\eta_{12})\Lambda=\mathbf{R}E(\Phi_2){\mathbf{J}_{\eta}}_{3\times 3}\Lambda,
	\end{aligned}
\end{equation}
where ${\mathbf{J}_{\eta}}_{3\times 3} =\mbox{diag}(\mathbf{J}_{\eta_{12}},\det(\mathbf{J}_{\eta_{12}}))$ and $\Lambda = \mbox{diag}(\lambda,\lambda,\lambda^2)$. \emph{It is noted that these constraints (especially last equation in~\eqref{eq7}) are different from the constraints (6) in the related work~\cite{10}.}\footnote{\cite{10} assumes deformations to be  infinitesimally linear. In the formulation, $\mathbf{J}_{\Psi_{21}} =\Upsilon \mathbf{R}$, $\Upsilon=\mbox{diag}({\lambda_1,\lambda_2,\lambda_3})$, $\lambda_1 \neq \lambda_2 \neq \lambda_3 $ and $\lambda_1,\lambda_2,\lambda_3 \approx 1$. This can be understood as as-rigid-as-possible transformation. This assumption causes the important conclusions in~\cite{10}: $\bar{e}_3^*=\bar{e}_1^*\times \bar{e}_2^*=\Upsilon\bm{R}{e}_3^*$ and $\overline{E}(\Phi_1)=\Upsilon\bm{R}E(\Phi_2\circ\eta_{12})$, which is different from our formula derivation.}


Under the different assumptions on the deformations ($\mathbf{J}_{\Psi_{21}} = \mbox{diag}({\lambda_1,\lambda_2,\lambda_3})$ $\mathbf{R}$), we investigate the invariant relationship of the connections of the mapping $\Phi_1$ and the composite mapping $\Phi_2\circ\eta_{12}$ and obtain the following results which are different from ~\cite{10}. We write the  connections in equation~\eqref{eq5} as $\Gamma=(\Gamma^i_{j1},\Gamma^i_{j2})$. 


\noindent {\bf Claim 1}. If the deformation mapping $\Psi_{21}$ is more general, such as diffeomorphic, the connection of the mapping $\Phi_1$ is not equal to the one of the mapping $\Phi_2\circ\eta_{12}$ unless additional special conditions hold. 
\begin{equation}
	\begin{aligned}
		\Gamma(\Phi_1)\neq \Gamma(\Phi_2\circ\eta_{12}),
	\end{aligned}
\end{equation}
where $\Gamma(\star)$ means the connections of the mapping $\star$. For the proof process, please refer to Appendix~\ref{App_1}.

Claim 1 shows that the connections of the mapping $\Phi_1$ and $\Phi_2\circ\eta_{12}$ are non-invariant  under the local generic diffeomorphic deformation without any assumptions. Based on the moving frame relations~\eqref{eq7}, we now derive the relation for connections under conformal deformations. We next show that like moving frames, connections across surfaces are related by image warps (up to first and second order) and conformal scale, $\lambda$. They are, however, invariant to the rotation $\mathbf{R}$.


\begin{theorem}\label{t2}
	For a conformal deformation, $\mathbf{J}_{\Psi_{21}}=\lambda \mathbf{R}$, the relation between connections ${\Gamma}(\Phi_1)$ and  $\Gamma(\Phi_2\circ\eta_{12})$ is invariant to rotation $\mathbf{R}$ and is not invariant to $\lambda$, given by
	\begin{align}\label{eq_861}
		\nonumber&~{\mathbf{J}_{\eta}}_{3\times 3}\Lambda{\Gamma}^i_{j1}(\Phi_1)\Lambda^{-1}= \\ \nonumber& \left({\Gamma}^i_{j1}(\Phi_2)\frac{\partial {u}}{\partial\overline{u}}+{\Gamma}^i_{j2}(\Phi_2)\frac{\partial {v}}{\partial \overline{u}}\right){\mathbf{J}_{\eta}}_{3\times 3}
		+\frac{\partial {\mathbf{J}_{\eta}}_{3\times 3}}{\partial \overline{u}},\\
		\nonumber &~{\mathbf{J}_{\eta}}_{3\times 3}\Lambda{\Gamma}^i_{j2}(\Phi_1)\Lambda^{-1}= \\& \left({\Gamma}^i_{j1}(\Phi_2)\frac{\partial {u}}{\partial \overline{v}}+{\Gamma}^i_{j2}(\Phi_2)\frac{\partial {v}}{\partial \overline{v}}\right){\mathbf{J}_{\eta}}_{3\times 3}
		+\frac{\partial {\mathbf{J}_{\eta}}_{3\times 3}}{\partial \overline{v}}.
	\end{align}
\end{theorem}
\begin{proof}
    For the proof process, please refer to Appendix~\ref{App_2}.
\end{proof}

Theorem \ref{t2} shows that the relationship of the connections ${\Gamma}^i_{jk}(\Phi_1)$ and  ${\Gamma}^i_{jk}(\Phi_2\circ\eta_{12})$ is rotational invariant  and is only related to the conformal scale $\lambda$.  We can compute the connections and the conformal scale of  $\mathcal{S}_1$ from those of $\mathcal{S}_2$ using $\eta_{12}$.~\cite{6} used the connections in 2D case to find a joint conformal NRSfM under the local infinitesimally planar approximation without considering the rotation scale $\lambda$. Our goal is, however, to recover $\lambda$ and get a better estimate on the local depth using the full information of the connections in 3D case and without any approximation in the differentiable order.
We will exploit the result in the metric preservation~\eqref{eq7a} and Theorem \ref{t2} to develop our algorithm.

\begin{corollary}\label{c2}
	If the deformation mapping $\Psi_{21}$ is local conformal, the second-order leading principal minors and the last elements of the connections $\Gamma^i_{jk}(\Phi_1)$ and  $\Gamma^i_{jk}(\Phi_2\circ\eta_{12})$ of the mapping $\Phi_1$ and $\Phi_2\circ\eta_{12}$ are invariant.
\end{corollary}

\begin{proof}
    For the proof process, please refer to Appendix~\ref{App_3}.
\end{proof}

\begin{corollary}\label{c3}
	For the conformal deformation $\Psi_{21}$, we can write $\Gamma^i_{j1}(\Phi_1)=(\bar{\bm{T}}^1_{kl})$, $\Gamma^i_{j2}(\Phi_1)=(\bar{\bm{T}}^2_{kl})$, $\Gamma^i_{j1}(\Phi_2)=({\bm{T}}^1_{kl})$, and $\Gamma^i_{j2}(\Phi_2)=({\bm{T}}^2_{kl}),~k,l=1,2$ as $2\times 2$ block matrices. Then, for two frames, only considering the rotational invariance of the connection in~\eqref{eq_861}, the conformal scale has a closed-form solution based on the sum of squares formulation: 
	\begin{equation}\label{eq_23add2}
		\begin{aligned}
		&\lambda=\frac{1}{8}\sum_{i=1}^{4}\sum_{j=1}^{2}b^i_j\\
&b^1_j=\frac{(\frac{\partial u}{\partial \bar{u}}{\bm{T}}^1_{21}+\frac{\partial v}{\partial \bar{u}}{\bm{T}}^2_{21})\mathbf{J}_{\eta_{12}}\bm{a}_{j}}{\det(\mathbf{J}_{\eta_{12}})\bar{\bm{T}}^1_{21}\bm{a}_{j}},\\
&b^2_j=\frac{{\bm{a}_{j}}^\top\mathbf{J}_{\eta_{12}}\bar{\bm{T}}^1_{12}}{{\bm{a}_{j}}^\top(\frac{\partial u}{\partial \bar{u}}{\bm{T}}^1_{12}+\frac{\partial v}{\partial \bar{u}}{\bm{T}}^2_{12})\det(\mathbf{J}_{\eta_{12}})},\\
&b^3_j=\frac{(\frac{\partial u}{\partial \bar{v}}{\bm{T}}^1_{21}+\frac{\partial v}{\partial \bar{v}}{\bm{T}}^2_{21})\mathbf{J}_{\eta_{12}}\bm{a}_{j}}{\det(\mathbf{J}_{\eta_{12}})\bar{\bm{T}}^2_{21}\bm{a}_{j}},\\
&b^4_j=\frac{{\bm{a}_{j}}^\top\mathbf{J}_{\eta_{12}}\bar{\bm{T}}^2_{12}}{{\bm{a}_{j}}^\top(\frac{\partial u}{\partial \bar{v}}{\bm{T}}^1_{12}+\frac{\partial v}{\partial \bar{v}}{\bm{T}}^2_{12})\det(\mathbf{J}_{\eta_{12}})},\\
		\end{aligned}
	\end{equation}
where $b^i_j\in\mathbb{R},~\bm{a}_{1}=(1,~0)^\top,~\bm{a}_{2}=(0,~1)^\top$.

\end{corollary}
\begin{proof}
    For the proof process, please refer to Appendix~\ref{App_4}.
\end{proof}

In the later section, we will consider both the invariance of the connection in~\eqref{eq_861} and metric preservation in~\eqref{eq7a} to recover the conformal instead of using the connection invariance only, which leads to a non-linear equation with a higher dimension.

\section{Con-NRSfM Algorithm}\label{s5}

Our algorithm is built based on the rotational invariance property of the connections and the metric preservation of the moving frame (Theorem~\ref{t2} and~\eqref{eq7a}) using the parallel separable iterative framework and self-supervised convolutional network.
\subsection{Point-wise Solution}
For the $i$-th feature in the $j$-th shape, we can define several variables $y^{(i,j)}_1$, $y^{(i,j)}_2$, $y^{(i,j)}_{11}$, $y^{(i,j)}_{12}$, $y^{(i,j)}_{22}$, $\beta^{(i,j)}$, and $\lambda^{(i,j)}$, corresponding to the first-order terms $y_1$, $y_2$, the second-order terms $y_{11},~y_{12},~y_{22}$, the depth $\beta$, and the conformal scale $\lambda$. Based on Section~\ref{s33} and a given edge number $N_e$, we can generate a well-connected graph $\mathcal{G}_{opt}$~\cite{3}. Assuming that all the features are detected and tracked in all the images (this is only for simplification of the equation, later in Section~\ref{s6} we will show the proposed algorithm is fine for data missing), the NRSfM problem aims to compute the depths of $N_p$ 3D points from $N_m$ monocular images. This problem can be formulated as a graph optimization problem of which the variable dimension is $7N_pN_m$. For the $e$-th edge $(i_e,~j_e)\in\mathcal{G}_{opt}$, considering the $i$-th feature in the $i_e$-th and the $j_e$-th shape, the factors  $\bar{f}_{j'}(i,i_e,j_e)$ of this graph optimization problem are the virtual measurements written by the invariance relationships~\eqref{eq7a} and~\eqref{eq_861}. As shown in \figurename~\ref{figurelabe3}, we obtain a weighted non-linear least squares (NLLS) problem:
\begin{equation}\label{E20}
	\begin{aligned}
		&\min\sum_{(i_e,j_e)\in\mathcal{G}_{opt}}\sum_{i=1}^{N_p}\omega_{e}\sum_{j'=1}^{21}\|\bar{f}_{j'}(i,i_e,j_e)\|^2,\\
	\end{aligned}
\end{equation}
where 
\begin{equation}
	\begin{aligned}
		&\bar{f}_{j'}(i,i_e,j_e)=f_{j'}(\widetilde{\bm{x}}^{(i,i_e)}_{n},\widetilde{\bm{x}}^{(i,j_e)}_{n}, \lambda^{(i,i_e,j_e)}),\\
        &\widetilde{\bm{x}}^{(i,i_e)}_{n}=(y^{(i,i_e)}_1,y^{(i,i_e)}_2,y^{(i,i_e)}_{11},y^{(i,i_e)}_{12},y^{(i,i_e)}_{22},\beta^{(i,i_e)})^\top,\\
        &\widetilde{\bm{x}}^{(i,j_e)}_{n}=(y^{(i,j_e)}_1,y^{(i,j_e)}_2,y^{(i,j_e)}_{11},y^{(i,j_e)}_{12},y^{(i,j_e)}_{22},\beta^{(i,j_e)})^\top
	\end{aligned}
\end{equation} are the virtual measurements, $\omega_{e}$ is the weight of the edge $(i_e,j_e)$ in the selected graph $\mathcal{G}_{opt}$, $j'$ is the equation number of the virtual measurements. Because~\eqref{eq_861} corresponds to two $3\times3$ matrices and~\eqref{eq7a} related to three independent elements in a $2\times2$ matrix, the total number of the virtual measurements for one feature within two images is $2\times9+3=21$.

\begin{figure}[!ht]
	\begin{center}
	\includegraphics[width=0.8\linewidth]{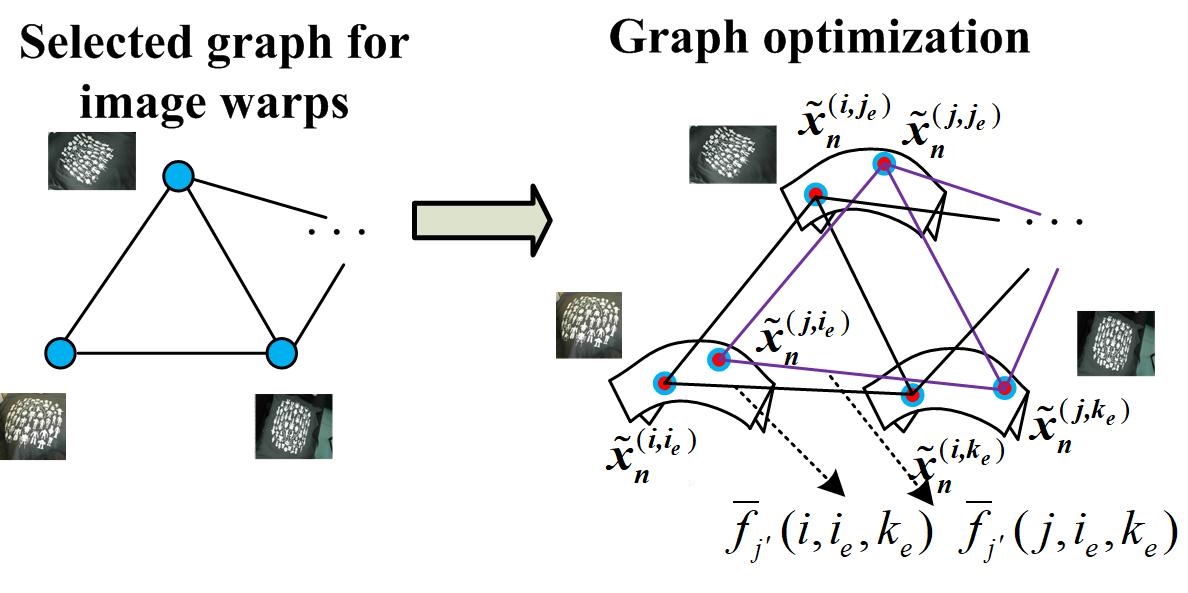}
	\end{center}
	\caption{Point-wise solution using graph optimization.}
	\label{figurelabe3}
\end{figure}

\subsection{Separable Framework using Parallel Strategy}\label{s52}
Typically, a weighted NLLS problem is solved using gradient-based methods~\cite{3}. To mitigate optimization sensitivity, reduce the impact of multiple local minima, and address issues related to unbounded scaling, we propose a separable framework. It decouples the sub-variables and leverages a parallel strategy to enhance efficiency and robustness. 

\textbf{Pre-Step. Variable optimization using isometric assumption}: Under the isometric assumption (a special case of the conformal assumption) and the IP assumption~\cite{5}, for two images, the invariance properties of Christoffel Symbols (CS)\cite{30} and MT from related work\cite{6} are utilized to construct four virtual measurements. The first-order terms, $y^{(i,j)}_1$ and $y^{(i,j)}_2$, are estimated by solving a weighted NLLS problem, where the virtual measurements act as factors. Since the measurements between 3D points are independent of feature IDs, the overall NLLS problem involving $N_p$ 3D points can be decomposed into $N_p$ independent sub-NLLS problems, each corresponding to specific features. These sub-problems are solved in parallel using multiple cores, as illustrated in \figurename~\ref{figurelabe4}. {Each core independently addresses a sub-problem, and the final solution is obtained by concatenating the results from all sub-problems.} The first-order terms are efficiently computed using the trust-region-reflective (TRR) algorithm, aided by the given sparse Jacobian matrix.
\begin{figure}[!ht]
	\begin{center}
		\includegraphics[width=0.9\linewidth]{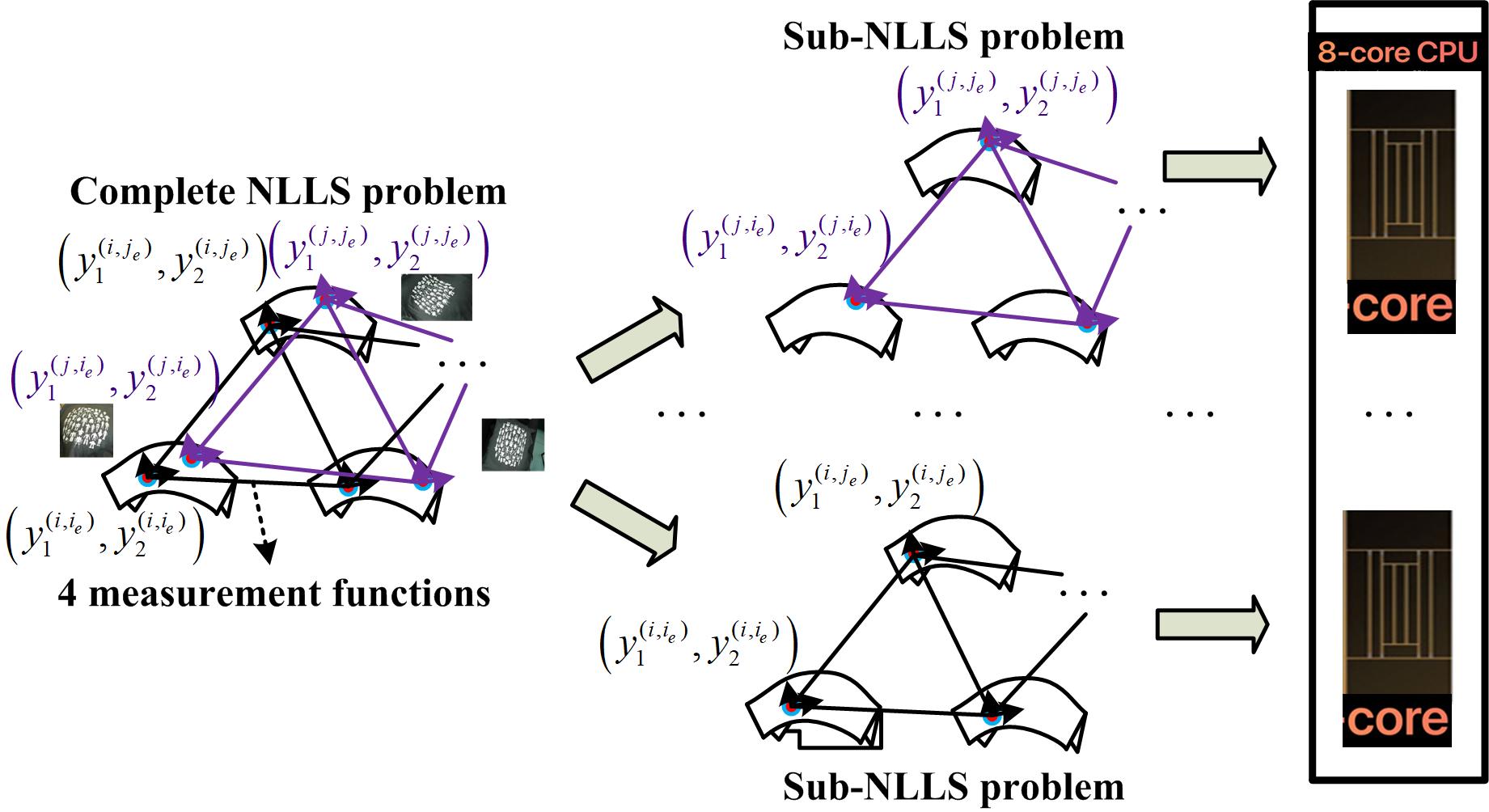}
	\end{center}
	\caption{Parallel method to solve the point-based problems.}
	\label{figurelabe4}
\end{figure}

  
  

By introducing the obtained first-order terms into the Jacobian matrix $\mathbf{J}_{\Phi_i}$ of the image embedding $\Phi_i$, the normal fields are obtained by normalizing the cross-product of two columns of the Jacobian matrix. Then, by integrating the normal fields~\cite{5}, the whole surfaces with the depth $\beta$ are recovered. Different from the parallel strategy in the first-order terms optimization, the depths of the points belonging to one frame are recovered together, which means that dividing the complete problem based on the points is impossible. Hence, each core runs the depth recovery corresponding to one frame, instead of a feature, as shown in \figurename~\ref{figurelabe5}. Then, the conformal scales $\lambda$ is given by the average of 11 equations in~\eqref{eq7a} and~\eqref{eq_23add2}.

\begin{figure}[!ht]
	\begin{center}
	\includegraphics[width=0.7\linewidth]{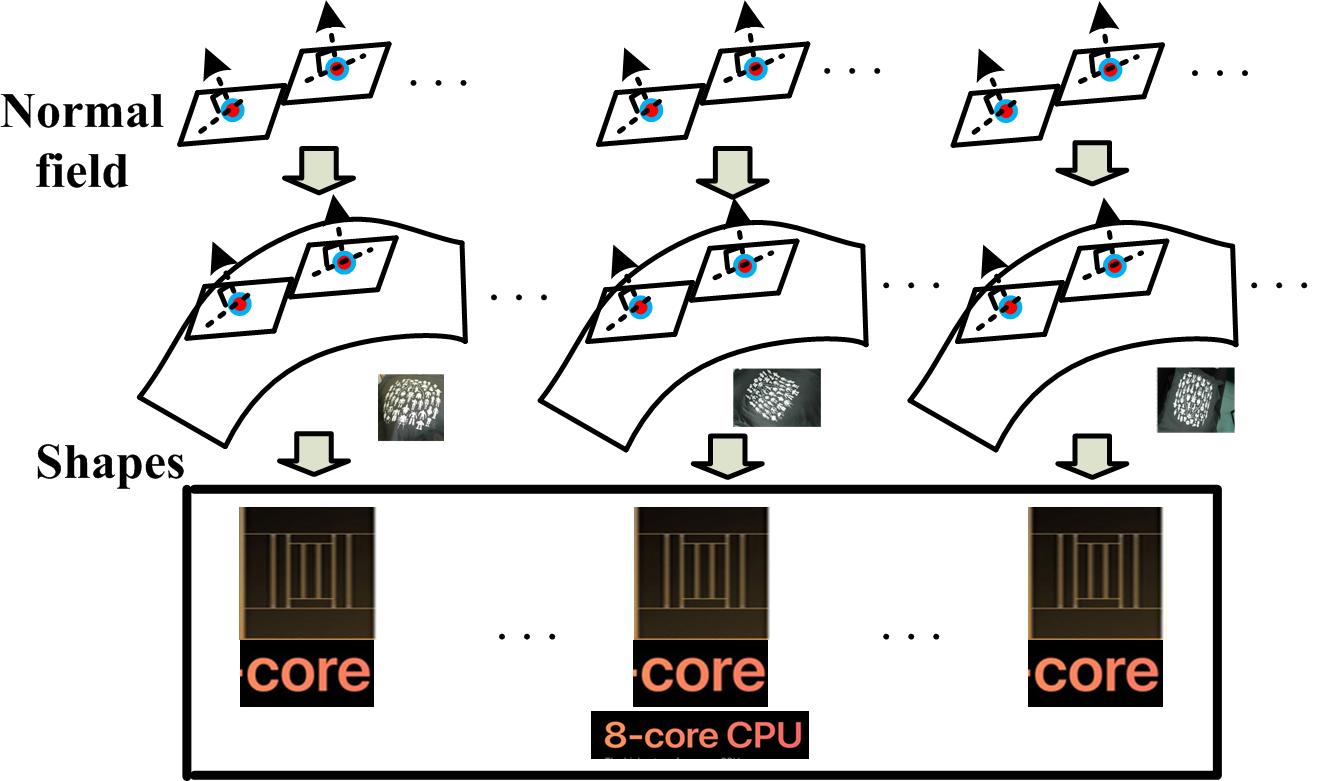}
	\end{center}
	\caption{Parallel method to solve the depth recovery.}
	\label{figurelabe5}
\end{figure}

\textbf{Step 1. Second-order terms optimization}:
At this stage, the first-order terms $y_1$, $y_2$, the depth $\beta$, and the conformal scale $\lambda$ as known constant values. The second-order terms $y^{(i,j)}_{11}$, $y^{(i,j)}_{12}$, and $y^{(i,j)}_{22}$ are considered as variables. Initially, these second-order terms are set to zero in the first iteration, while in subsequent iterations, the solution from the previous iteration is used as the starting value. Since metric preservation does not involve second-order information, only the rotational invariance property is applied here. For every pair of features corresponding to edges in the selected tree $\mathcal{G}_{opt}$, the connection~\eqref{eq5} is incorporated into the rotational invariance property~\eqref{eq_861}, yielding 18 virtual measurements. These measurements depend on the second-order terms as variables and the elements of equation~\eqref{eq_861} as factors. Using a parallel approach similar to that in \figurename~\ref{figurelabe4}, point-based sub-NLLS problems are solved efficiently. The second-order terms are computed using the TRR algorithm, with the optimization iterations capped at $N_t=3$. Subsequent steps for solving problems use the same optimization strategy and parameters.

\textbf{Step 2. First-order terms optimization using conformal assumption}: In this step, only the first-order terms $y^{(i,j)}_1$, $y^{(i,j)}_2$ are regarded as the variables in problem~\eqref{E20} and the other variables are considered as the known constant values based on the conformal assumption. The variables $y^{(i,j)}_1$, $y^{(i,j)}_2$ are obtained based the point-based sub-NLLS problems.


\textbf{Step 3. Depth computation}: We express local normals using first-order terms and integrate the normal field along the surface to recover depth up to a scale factor. This process is performed in parallel across different frames, similar to  Fig.~\ref{figurelabe5}.

\textbf{Step 4. Conformal scale optimization}: In the Pre-Step, the conformal scale is calculated using an approximate formulation. However, this solution is unstable and only valid when all factors in \eqref{eq_861} and \eqref{eq7a} are perfectly satisfied, which is unrealistic in real-world cases. In this step, similar to Steps 2 and 3, all previously computed terms are treated as constants, while the conformal scale $\lambda^{(i,i_e,j_e)}$ becomes the variable. Multiple sub-NLLS problems\footnote{Due to Corollary~\ref{c2}, second-order leading principal minors and the last elements of the connections are excluded.} are constructed using different features and solved efficiently.

\textbf{Terminal conditions}: At each iteration, we go through Step 1 to Step 4 and get a set of optimized parameters $\bm{x}^{(i,i_e)}_{n}(k)$ and  $\lambda^{(i,i_e,j_e)}(k)$, where $k$ is the iteration number. Then, we will check its convergence index $Terminal(k)=\sum_{i}\sum_{i_e}\|\bm{x}^{(i,i_e)}_{n}(k)-\bm{x}^{(i,i_e)}_{n}(k-1)\|+\sum_{i}\sum_{(i_e,j_e)\in\mathcal{G}_{opt}}\|\lambda^{(i,i_e,j_e)}(k)-\lambda^{(i,i_e,j_e)}(k-1)\|<\sigma$ using the results from the last two iterations. If the condition is met, the depth  $\beta^{(i,i_e)}$ is finalized as the solution.

\begin{remark}
{Our separable optimization framework is inspired by the principle of alternating optimization~\cite{41add}. Given that the underlying problem is a non-convex NLLS formulation, we cannot theoretically guarantee the convergency of using separable framework. However, in practice, the framework exhibits favorable numerical properties, including better conditioning of individual sub-problems, reduced memory usage, and empirically improved convergence behavior. We now consider the applied factors $\bar{f}_{j'}(i, i_e, j_e)$ in the NLLS formulation. The fixed coefficients associated with these factors are precomputed from image warping and encapsulate shared surface information. Once computed, these fixed coefficients enable each factor to impose only the connection constraint~\eqref{eq7a} and the metric tensor constraint~\eqref{eq_861} for the same spatial point observed across different frames, without introducing measurement constraints between different points within the same frame. Consequently, several key optimization steps in our framework, including the initialization in the Pre-step, second-order term optimization in Step 1, first-order term optimization in Step 2, and conformal scale adjustment in Step 4, are inherently independent across different points. This natural independence permits efficient parallelization across multiple cores without affecting convergence or final estimation accuracy. With the given derivative terms $y^{(i,j_e)}_1,y^{(i,j_e)}_2,y^{(i,j_e)}_{11},y^{(i,j_e)}_{12},y^{(i,j_e)}_{22}$ fused along the surface, the normal field of the whole surface is obtained.  The subsequent depth recovery, performed up to a scale factor using the estimated normal field, is also frame-independent. Thus, parallelization in Step 3 across different frames is theoretically sound and empirically valid. Overall, our parallel, separable framework achieves high accuracy, efficiency, and robustness. As shown in Section~\ref{s6}, it solves the problem reliably, even with zero or random initialization.}
\end{remark}

\subsection{Dense 3D point cloud with texture}\label{s62}
In certain NRSfM tasks, the desired output is a dense 3D point cloud with texture. To achieve this, a self-supervised convolutional neural network can be employed to generate the dense point cloud from a sparse normal field and the corresponding input images\footnote{Alternatively, if only a sparse point cloud is required, the classical integration method described in Step 3 can be used to directly recover the sparse point cloud.}.


To apply the self-supervised approach for automatic data generation, we randomly sample sparse 3D point clouds with a specified number of points, $n_p$, within a predefined range, including normalized pixel coordinates and depth values. Tangent vectors are computed along pixel directions for all sampled 3D features, and their corresponding normal vectors are derived, forming the normal field using the BBS tool\footnote{The BBS tool is a toolbox for efficiently handling bicubic B-splines, enabling interpolation of 3D point clouds~\cite{32}.}. This process allows the generation of numerous surfaces with their normal fields as inputs and depth values as outputs. To create smoother surfaces, we further employ fifth-degree polynomial surfaces to fit the randomly sampled point cloud and extract new points along the fitted surface.

 Convolutional neural networks are commonly used with image datasets, so the next step involves encoding the normal field inputs and depth outputs into RGB and depth images, respectively. Consider the obtained normal field $\{\bm{n}_{i}\in\mathbb{R}^3\},~i=1,2,\cdots,n_p$, corresponding to the 2D features $(u_i,v_i),~i=1,2,\cdots,n_p$. By combining the three components of the normal vector with the 2D features in a single frame $(u_i,v_i)~i=1,2,\cdots,n_p$, we can construct 3D point clouds. These point clouds can then be integrated into a dense 3D surface using the BBS tool. The three resulting surfaces corresponding to one frame are treated as the RGB channels of an image by interpolating all pixel positions  $(u_i,v_i)~i=1,2,\cdots,n_p$. In this way, the normal fields are encoded as multiple RGB images, which serve as inputs for the convolutional neural network. Similarly, the depth data is encoded into RGB images by assigning the same values to all three channels. Alternatively, the depth values corresponding to the 2D features can also be encoded into depth-specific images. To improve robustness, data augmentation is applied by adding Gaussian noise to the components of the input normal vectors. This improves the model's ability to generalize to noisy data during practical applications.

 \begin{figure}[!ht]
 	\begin{center}
 		\includegraphics[width=0.9\linewidth]{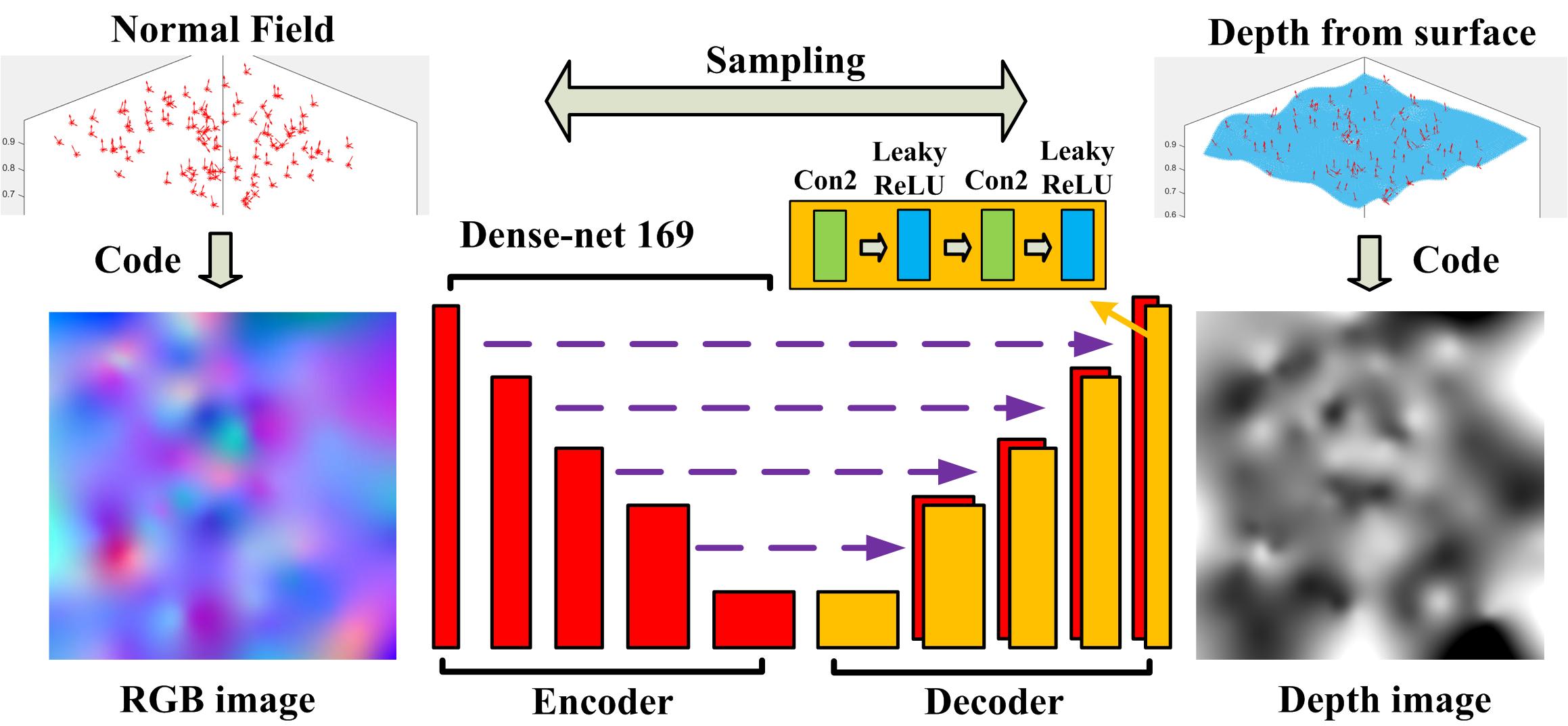}
 	\end{center}
 	\caption{Self-supervised network for depth recovery.}
 	\label{figurelabe6}
 \end{figure}
 

\figurename~\ref{figurelabe6} provides an overview of our framework for depth recovery from the normal field using a straightforward encoder-decoder network with skip connections. The architecture is a slightly modified version of a well-known network for monocular image-based depth estimation~\cite{32}. The encoder is based on DenseNet-169~\cite{33}, but unlike~\cite{32}, which employs transfer learning, our network does not use pre-trained weights from ImageNet~\cite{34}. This decision stems from the fact that our input images, derived from normal fields, differ significantly from traditional monocular images of real-world objects, rendering ImageNet pre-trained weights unsuitable for feature extraction in our case. The decoder consists of five basic blocks, each comprising convolutional layers, a leaky rectified linear unit (Leaky ReLU), and a 2$\times$ bilinear upsampling layer with concatenation. Our network adopts the same simple architecture approach as~\cite{32}, avoiding Batch Normalization~\cite{35} or other advanced layers described in~\cite{36}. Additionally, we enhance the architecture by appending an upsampling layer, a 2D convolutional layer, and a Leaky ReLU to align the output depth image with the input RGB size.

For the loss function, we only use the point-wise L1 loss defined as the differences between the ground truth depth values and the predicted depth values:

\begin{equation}\label{E18}
	\begin{aligned}
		L(\bm{p},\bar{\bm{p}})=\frac{1}{n_a}\sum_{i=1}^{n_a} |\bm{p}_i-\bar{\bm{p}}_i|_1,
	\end{aligned}
\end{equation}
where $n_a$ is the number of depth image pixels, $\bm{p}_i$ and $\bar{\bm{p}}_i$ are respectively the ground truth and prediction of the $i$-th pixel.

The primary advantage of the trained network is its ability to directly recover a dense point cloud, rather than being limited to a sparse solution. This makes our self-supervised convolutional neural network significantly faster and more robust compared to traditional integration and texturing methods~\cite{3}.
\subsection{Algorithm Summary}
Our Con-NRSfM method integrates differential geometry, a separable parallel framework, and a learning-based neural network. Using the edge selection approach described in Section~\ref{s33}, we identify a well-connected subgraph from a complete weighted graph. For each pair of connected images, the image warp is computed and used to formulate a point-wise weighted NLLS problem based on virtual measurements derived from metric preservation \eqref{eq7a} and rotational invariance of the connection \eqref{eq_861}. To address this large-scale NLLS problem, we employ a separable framework that decouples the first-order terms $y_1$, $y_2$, second-order terms $y_{11},~y_{12},~y_{22}$, depth $\beta$, and conformal scale $\lambda$, improving both robustness and solution accuracy. The process begins with simplifying assumptions to solve an isometric NRSfM efficiently, providing the first-order terms. Depth values are then recovered using an integration approach, and conformal scales are quickly estimated using an approximate closed-form solution. Next, we apply a point-based parallel iterative strategy to optimize the first-order terms, second-order terms, and conformal scale sequentially using the proposed NLLS formulation \eqref{E20}. Finally, the depth and corresponding 3D dense point cloud with texture are recovered from the normal field, leveraging the first-order terms and a pre-trained self-supervised network described in Section~\ref{s62}. The complete pipeline is summarized in Algorithm~\ref{alg:euclid} and illustrated in the flow chart in Fig.~\ref{figurelabexx}.

\begin{algorithm}[!ht]  
	\caption{Con-NRSfM Algorithm}  
	\label{alg:euclid}  
	\small{\begin{algorithmic}[1]  
			\Require{2D sparse point clouds $(u,v)$ with feature correspondences in different frames or multiple monocular images $\mathcal{I}_i$, the given edge number $N_e=n_c-1+k$.}
			\Ensure{The 3D point clouds $\bm{z}$ with depth or dense surfaces.}	
\Statex{ \hspace{-2em} \textcolor{magenta}{Pre-Step: Variable optim. using isometric assumption (line 1-7)}}
			\State {Build a complete graph $\mathcal{G}_c=(\mathcal{V}_c,~\mathcal{E}_c,~w_c)$ using feature matching for every image pair;}
			\State {Select a well-connected subgraph $\mathcal{G}_{opt}$ using the Kruskal’s maximum spanning tree algorithm and the greedy-based method;}
			\State {Compute the image warps based on 2D Schwarzian derivatives;}
			\State {Compute first-order terms using the isometric assumption;}
			\State {Obtain the normal field $\{\bm{n}_{i}\in\mathbb{R}^3\}$ by Jacobian matrix $\mathbf{J}_{\Phi_i}$;}
			\State {Recover $\beta$ by integrating the normal field $\{\bm{n}_{i}\}$ in parallel;}
			\State Compute conformal scale $\lambda_c$ using average value;
			\While{$Terminal(k)\geq\sigma$}\Comment{\textcolor{magenta}{Terminal conditions}}
			
			\hspace{-1.15cm}\textcolor{magenta}{Step 1: Second-order terms optim. (line 9-10)}
			
			\State{Get the second-order terms $y_{11}$, $y_{12}$, and $y_{22}$ in parallel given the constant values: the first-order terms $y_{1,c}$, $y_{2,c}$, the depth $\beta_{c}$, and the conformal scale $\lambda_{c}$;}
            \State{$y_{11}$$\rightarrow$$y_{11,c}$, $y_{12}$$\rightarrow$$y_{12,c}$, $y_{22}$$\rightarrow$$y_{22,c}$}

 \hspace{-1.15cm}\textcolor{magenta}{Step 2: First-order optim. using conformal assumption (line 11-12)}
            
			\State{Get the first-order terms $y_{1}$ and $y_{2}$ in parallel by solving~\eqref{E20} given the constant values: the second-order terms $y_{11,c}$, $y_{12,c}$, $y_{22,c}$, the depth $\beta_{c}$, and the conformal scale $\lambda_{c}$;}
 	\State{$y_{1}$$\rightarrow$$y_{1,c}$, $y_{2}$$\rightarrow$$y_{2,c}$}
 	
 \hspace{-1.15cm}\textcolor{magenta}{Step 3: Depth computation (line 13-14)}

			\State{Recover $\beta$ by integrating the normal field $\{\bm{n}_{i}\}$ in parallel;}    	\State{$\beta$$\rightarrow$$\beta_{c}$;}
 	
 \hspace{-1.15cm}\textcolor{magenta}{Step 4: Conformal scale optim. (line 15-16)}
			
			\State{Get the conformal scale $\lambda$ in parallel by solving~\eqref{E20} given the constant values: the first-order terms $y_{1,c}$, $y_{2,c}$, the second-order terms $y_{11,c}$, $y_{12,c}$, $y_{22,c}$, and the depth $\beta_{c}$;}
			\State{ $\lambda$$\rightarrow$$\lambda_{c}$;}
			\EndWhile
			\If{ Images $\mathcal{I}_i$ are available $\bigcap$ require dense point clouds} Recover the dense point clouds using network in Section~\ref{s62};
			\EndIf
			\State \textbf{return} Sparse point cloud or dense surface with texture.
	\end{algorithmic} }
\end{algorithm} 

 \begin{figure}[!ht]
	\begin{center}
		\includegraphics[width=0.8\linewidth]{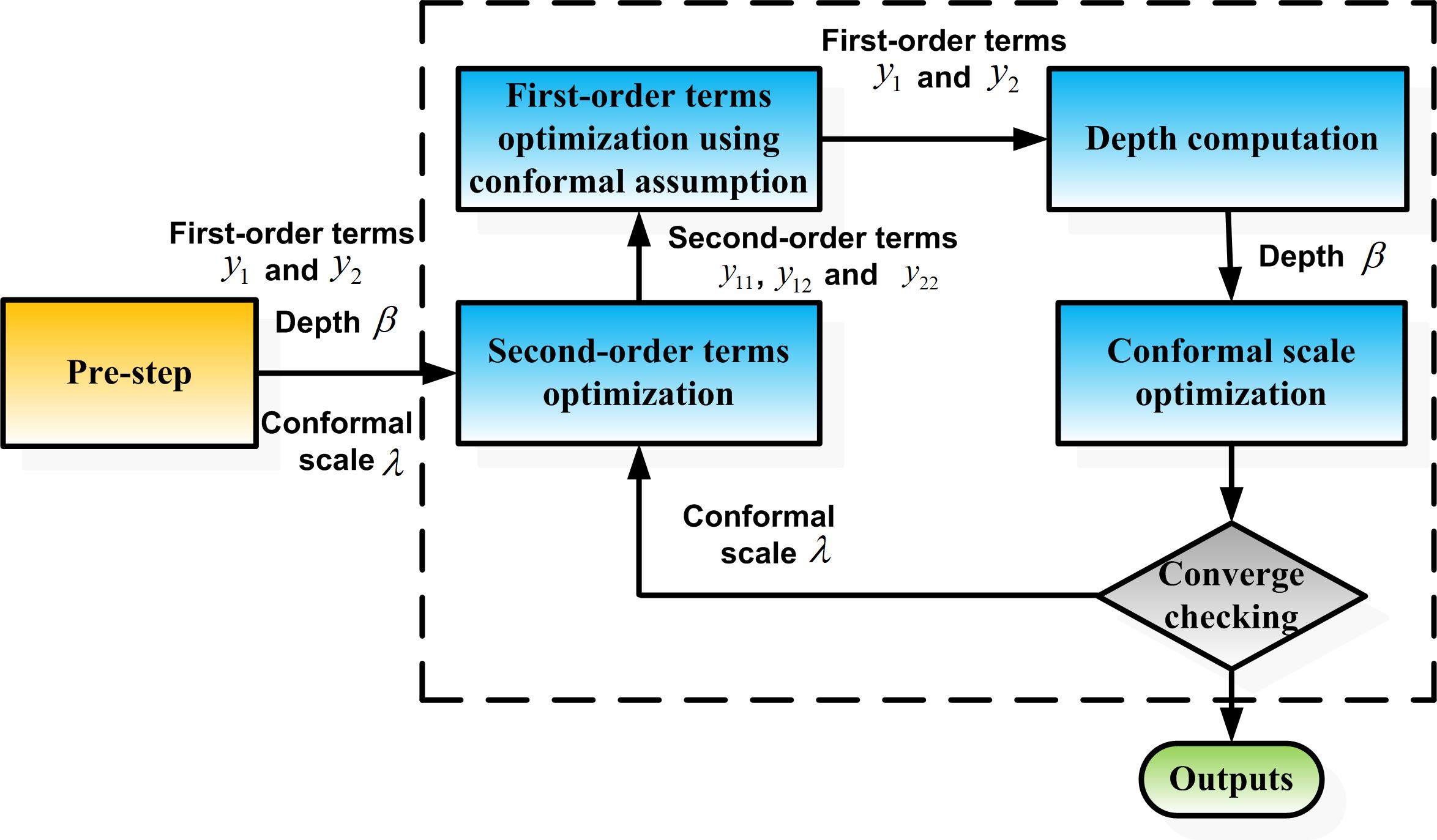}
	\end{center}
	\caption{Flow chart for the core of Con-NRSfM algorithm.}
	\label{figurelabexx}
\end{figure}

\section{Simulation and Experiments}\label{s6}
	In this part, we present simulations and experiments
with synthetic and real datasets using  C++ (Section~\ref{s64}) and MATLAB (other subsections) on a Dell
G5-5500 laptop with an Intel Core i7-10870H 2.20 GHz processor. Our network is pre-trained using Tensorflow 2.0. Our method is compared with some state-of-the-art
NRSfM methods by the shape error and \% 3D
error, which are commonly used in the NRSfM literature~\cite{22,3}, computed as RMSEs between the reconstructed and ground-truth results. All error computation of the recovered shapes is performed after using ABSOR function by finding optimal rotation, scale, and translation~\cite{41_new}.

\subsection{Simulation Dataset}
The theorem proposed in Section~\ref{s4} will be verified and then the performance of the proposed method will be evaluated using two synthetic datasets against other methods, including \textbf{infP}~\cite{6}\footnote{It shows a similar performance with its extended version \textbf{iso}~\cite{7}.}, \textbf{Diff}~\cite{10}, \textbf{Ch17}~\cite{4}, \textbf{SDP17}~\cite{25a}\footnote{Due to the lack of the open source, we re-implement this algorithm.}, and \textbf{Go20}~\cite{3}\footnote{The edge number needs to be set in this method. We use the default edge number with $n_c-1$ for \textbf{Go20} and our framework with no additional edges.}.
\subsubsection{Theorem Verification}
The proposed theorems and corollaries show the important invariance property of the connections undergoing different deformations. Here we mainly focus on the numerical verification of 
the relationship~\eqref{eq_861} in Theorem 1. The conformal deformation is simulated using 11 balls with different radii $R_i$ and central coordinates $(x_i,y_i,z_i)^\top$, where $i$ is the index of the balls. In the local coordinate system with the camera center as origin, all the features $(x_k^f,y_i^f,z_i^f)^\top$ on the partial surface of the ball, the depths $\beta_i=z_i$, and their first/second-order derivatives are written as the function of the pixels $(u_k,v_k)^\top$ using the perspective projection. For the other balls, the matched features $(\bar x^f_i,\bar y^f_i,\bar z^f_i)^\top$ and their corresponding pixels $(\bar u_k,\bar v_k)^\top$ are computed as: 
\begin{equation}
	\begin{aligned}
		&\left(\begin{array}{c}\bar x^f_i\\\bar y^f_i\\\bar z^f_i\end{array}\right)=
		&\left(\begin{array}{c}R_j/R_i(x_k^f-x_i)+ x_j\\R_j/R_i(y_k^f-y_i)+ y_j\\R_j/R_i(z_k^f-z_i)+ z_j\end{array}\right),
	\end{aligned}
\end{equation}
 and  $(\bar u_k,\bar v_k)^\top=(\bar x^f_i/\bar z^f_i,\bar y^f_i/\bar z^f_i)^\top$. Similarly, we can get their depths $\bar \beta_i=\bar z_i$, first/second-order derivatives, and the image warps using the analytical relationship between the different pixels $(\bar u_k,\bar v_k)^\top$ and $( u_k, v_k)^\top$. The conformal scale between two balls is computed as $\lambda_{ij}=R_i/R_j$.

Based on these analytical equations, 100 matched features are generated for each ball. All the obtained parameters are introduced to the rotational invariance property as shown in~\eqref{eq_861}. As an example, the (1,~1)-th elements at both sides of \eqref{eq_861} corresponding to 100 features and two balls are presented in \figurename~\ref{fig:graph81}. In some references~\cite{5,10}, the assumption of the infinitesimal planarity is used, which means that only the first-order terms are considered. Hence, as a comparison, connection~\eqref{eq5} ignoring the second-order terms is also introduced to the invariance property~\eqref{eq_861}. The similar results are also shown in \figurename~\ref{fig:graph82}. We can find that, when second-order terms in connections are considered, the simulation results completely satisfy the rotational invariance shown in Theorem~\ref{t2}. The simulation results only considering the first-order terms only roughly follow the rotational invariance in Theorem~\ref{t2}. We define an index for results using only first-order terms:
\begin{equation}\label{E18x}
	\begin{aligned}
		Index=\sum_{k=1}^{100}\sum_{i=1}^{3}\sum_{j=1}^{3}\frac{1}{900}\frac{|\Theta^l_{ij}(k)-\Theta^r_{ij}(k)|}{\text{max}_{k}~\Theta^l_{ij}(k)-\text{min}_{k}~\Theta^l_{ij}(k)},
	\end{aligned}
\end{equation}
where $\Theta^l_{ij}(k)$ and $\Theta^r_{ij}(k)$ respectively mean the $(i,j)$-th element of the left-hand and right-hand sides of \eqref{eq_861} for the $k$-th feature; $\text{max}_{k}$ and $\text{min}_{k}$ mean the maximal and minimal values for all the features.  Using this definition. we test multiple simulated balls (1 ball with 10 matched balls) with different centers and radii. For the case with first-order terms only, the indexes of different balls are respectively $6.82\%,~7.03\%,~9.69\%,~8.07\%,~9.89\%,~11.6\%,~7.70\%,$ $~6.01\%,~8.98\%,$ and $12.20\%$. For the case using both first-order and second-order terms, all the indexes are equal to 0\% consistently, which validates the correctness of Theorem~\ref{t2}.


\begin{figure}[!htb]
	\minipage{\columnwidth}
	\centering
	\subfloat[\footnotesize First and second-order terms. The result shows that the left-hand side is exactly the same as the right-hand side, when both the first and second-order terms are used.]{\includegraphics[width=0.8\columnwidth]{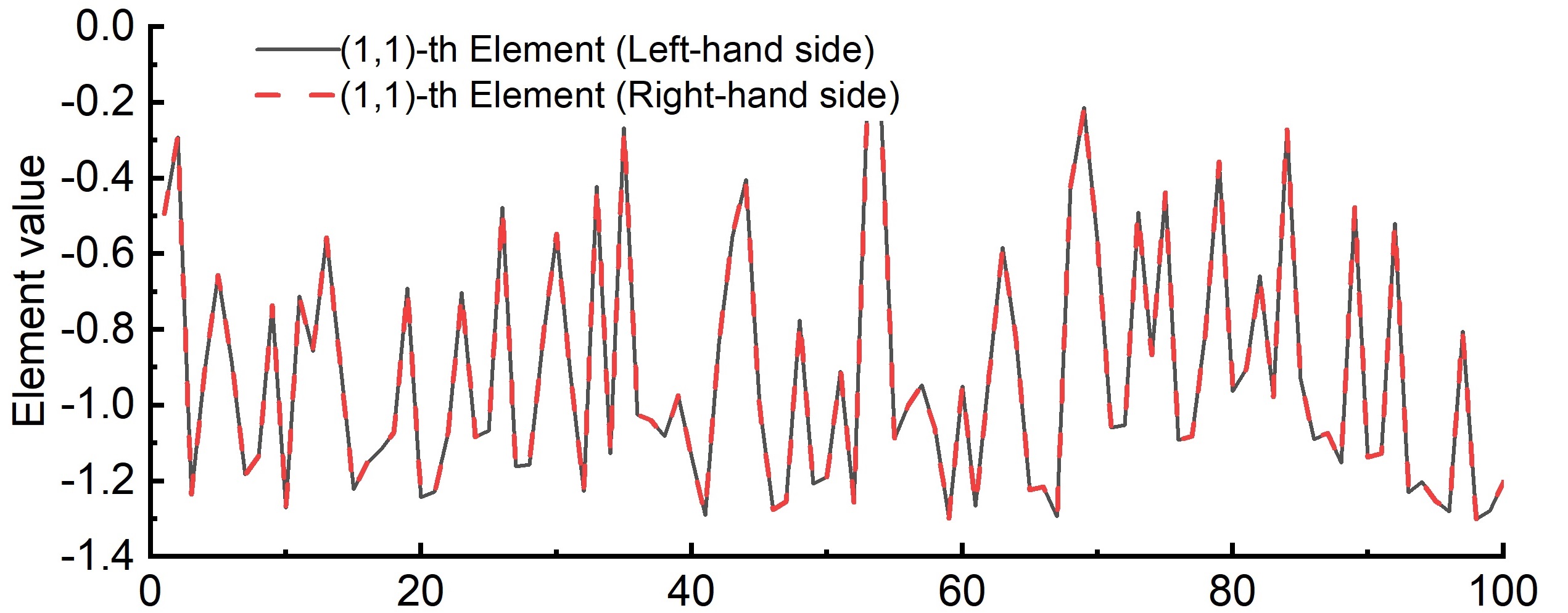}\label{fig:graph81}}
	\\
	\subfloat[\footnotesize{First-order terms only. The result shows that the left-hand side is approximately equal to the right-hand side, when only the first terms are used.}]{\includegraphics[width=0.8\columnwidth]{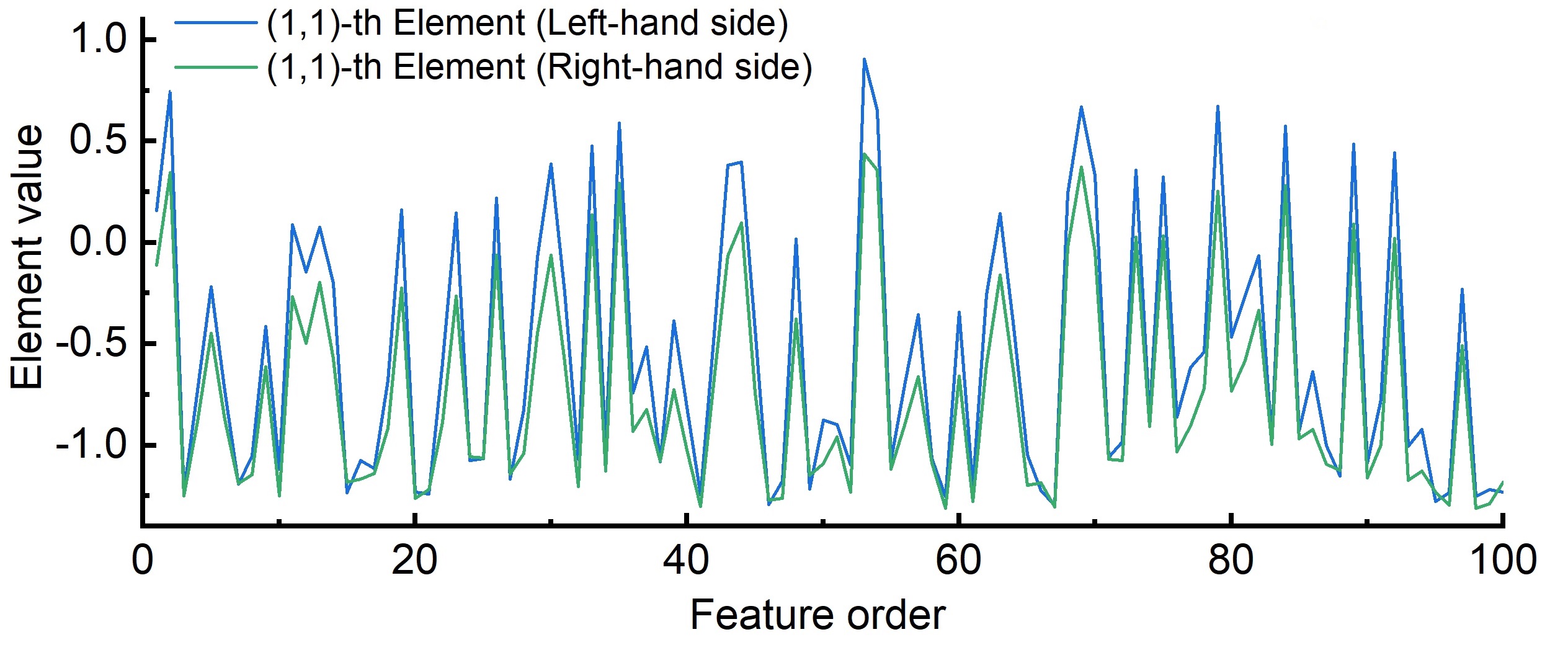}\label{fig:graph82}}
	\caption{{The (1, 1)-th element of the first equation of~\eqref{eq_861} shows the invariance property. }}
	\label{figurelabe1x}
	\endminipage
\end{figure}

\subsubsection{Synthetic Datasets}
We simulate a synthetic conformal dataset
with 7 deforming scenes and 100 features using the perspective projection model to test the performance of our method. This conformal dataset is built based on 7 ball surfaces with different radii and central coordinates. For this dataset, because the calibrated features beyond the vision range $\mathcal{R}_v$\footnote{Because we only sample the random data in a given range, if the vision range $\mathcal{R}_v$ and the deformation range $\mathcal{R}_d$ of the feature measurements are outside the given range, our trained network cannot offer accurate depth estimation results.}, which is set as $[-0.5,0.5]$ to $u$-axis and $[-0.5,0.5]$ to $v$-axis for our  network, we do not use the deep learning network in Section~\ref{s62}. As an example, the reconstructed result for two deforming shapes using our method is shown in \figurename~\ref{fig:graph5and8}.  The mean \%3D errors are respectively \textbf{0.8052}\%(\textbf{Ours}), 4.0454\%(\textbf{Diff}), 1.0643\%(\textbf{infP}), 1.4486\% (\textbf{Ch17}),  1.4545\%(\textbf{SDP17}), and 0.9792\%(\textbf{Go20}). The average shape errors are respectively  $\textbf{6.0822}^{\circ}$(\textbf{Ours}), $24.9474^{\circ}$(\textbf{Diff}), $7.2102^{\circ}$(\textbf{infP}), $13.0426^{\circ}$(\textbf{Ch17}), $12.9205^{\circ}$(\textbf{SDP17}), and $7.4458^{\circ}$(\textbf{Go20}). Based on the obtained results for these datasets, our method shows the best performance. 

\begin{figure}[!htb]
	\minipage{\columnwidth}
	\centering
	\hspace{.1cm}
	\subfloat[\footnotesize{Shape 2}]{\includegraphics[width=0.45\columnwidth,height=1.3in]{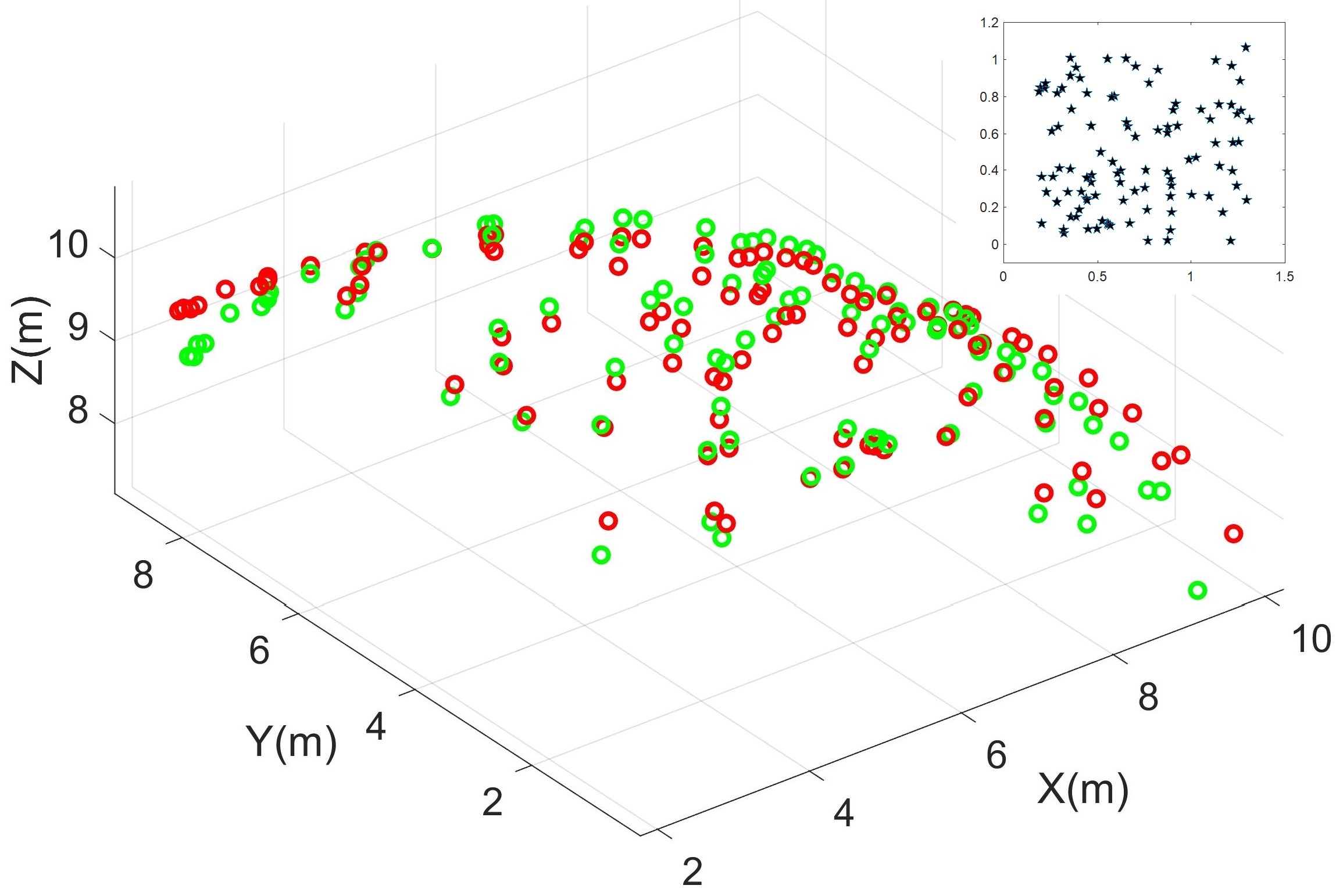}\label{fig:graph8}}
	\hspace{.1cm}
	\subfloat[\footnotesize{Shape 4}]{\includegraphics[width=0.45\columnwidth,height=1.3in]{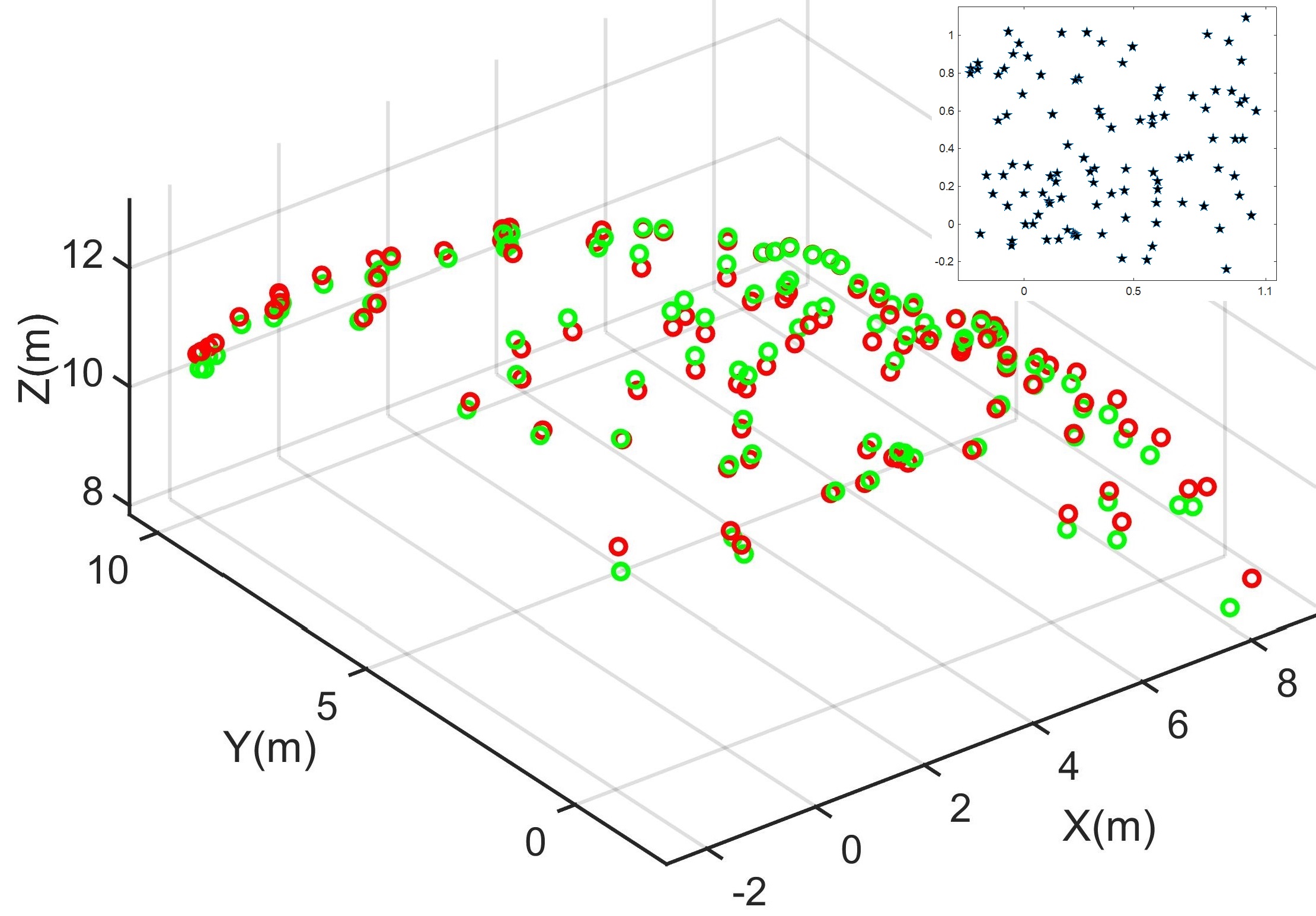}\label{fig:graph9}}
	\caption{{Ground truth (green) and reconstructed results (red) for the synthetic dataset. Right-up sub-graphs are 2D input scenes.}}
	\label{fig:graph5and8}
	\endminipage
\end{figure}

In order to show the robustness of our proposed method, the comparison results with the changing data are presented. Table~\ref{t1zz} shows the mean \%3D error (in \%) and the mean shape error (in $^\circ$) when 0-70\% point data are missing from every image, including the first image. A feature, which is missing in the first image but is included in others will be regarded as an appearing feature\footnote{The \textbf{Diff} and \textbf{infP} methods cannot deal with the case with appearing features, so we ignore these appearing features in the recovered shape and the error computation. Our method and the \textbf{Ch17} method can easily solve the problem with the appearing features, so these features remain.}. The final results in Table~\ref{t1zz} show the high accuracy and good robustness of our proposed method.

\begin{table}[!ht]
	\small
	\caption{{Comparison using missing and appearing data}}
	\label{t1zz}
	\begin{center}
		\begin{tabular}{|c|c|c|c|c|c|c|}
			\hline
			Missing&{\textbf{infP}$^*$}  & {\textbf{Ours}}&{\textbf{Diff}$^{**}$}  \\
			\cline{1-4}
			0\%&1.06\%-$7.21^\circ$&\textbf{0.81}\%-$\textbf{6.08}^\circ$&4.05\%-$24.95^\circ$
			\\
			10\%& 1.23\%-$8.62^\circ$& \textbf{1.10}\%-$\textbf{8.38}^\circ$&- 
			\\
			20\%& 1.47\%-$10.25^\circ$& \textbf{1.18}\%-$\textbf{9.78}^\circ$&- \\
			30\%&2.01\%-$15.39^\circ$& 1.74\%-$\textbf{13.59}^\circ$&- \\
			{40\%}&2.04\%-$15.97^\circ$& \textbf{1.66}\%-$\textbf{12.98}^\circ$&- \\
			{50\%}& 3.07\%-$20.53^\circ$&  \textbf{1.86}\%-$\textbf{13.21}^\circ$&-\\
			{{60\%}}& {3.48\%-$19.85^\circ$}&  {\textbf{1.87}\%-$\textbf{13.94}^\circ$}&{-}\\
			{{70\%}}& {4.05\%-$24.73^\circ$}&  {\textbf{2.39}\%-$\textbf{18.78}^\circ$}&{-}\\
			\hline
			Missing & {\textbf{Ch17}} & {\textbf{SDP17}}&\textbf{Go20}\\
			\cline{1-4}
			0\%& 1.45\%-$13.04^\circ$&1.45\%-$12.92^\circ$&0.98\%-$7.45^\circ$
			\\
			 10\%& 1.53\%-$14.05^\circ$&1.49\%-$14.01^\circ$&1.15\%-$8.44^\circ$
			\\
			20\%& 1.48 \%-$14.01^\circ$&1.51\%-$14.21^\circ$&1.28\%-$10.63^\circ$\\
			30\%& \textbf{1.63}\%-$15.24^\circ$&1.71\%-$15.52^\circ$&1.74\%-$13.67^\circ$\\
			{40\%}& 2.05\%-$16.49^\circ$&2.09\%-$16.37^\circ$&1.77\%-$14.02^\circ$\\
			{50\%} &10.76\%-$34.13^\circ$&8.97\%-$31.78^\circ$&1.99\%-$14.67^\circ$\\
			{{60\%}}& {14.06\%-$58.55^\circ$}&  {13.22\%-$49.31^\circ$}&{1.97\%-$14.98^\circ$}\\
			{{70\%}}& {16.93\%-$75.82^\circ$}&  {16.11\%-$68.99^\circ$}&{{2.44}\%-$19.04^\circ$}\\
			\hline
		\end{tabular}
		\\
	\end{center}
		\footnotesize{$^*$ Due to limitations in the Gloptipoly 3 toolbox~\cite{41xxxxxx} used in \textbf{infP}, errors arise when the synthetic dataset contains one or more features observed fewer than twice, causing the open-source code to fail in producing valid results. This issue does not affect the other compared methods. In cases with a high missing rate (e.g., 70\%), such errors occur frequently, and no valid output is returned. Therefore, we select a specific dataset in which each feature is observed at least twice to get the readable results and ensure a fair comparison.}\\
		\footnotesize{$^{**}$\textbf{Diff} is not very stable for the synthetic dataset with  missing features and no correct result is obtained using the open source code.}
\end{table}

We also simulate a challenging synthetic dataset with 10 deforming scenes and 600 features under both isometric and conformal deformations using perspective projection, as clearly indicated in \figurename~\ref{fig:graph8ass}. Here we only consider the sparse result. Therefore, the dense depth recovery network is not deployed. The reconstruct results for two deforming shapes using our method are shown in \figurename~\ref{fig:graph5and8ass}. Their mean \%3D errors are  \textbf{1.1241}\%(\textbf{Ours}), 2.9611\%(\textbf{Diff}),  1.2130\%(\textbf{infP}),  1.1891\% (\textbf{Ch17}),   1.3871\%(\textbf{SDP17}), and  1.2303\%(\textbf{Go20}). The average shape errors are respectively  $\textbf{9.6309}^{\circ}$(\textbf{Ours}), $16.2807^{\circ}$(\textbf{Diff}), $11.5302^{\circ}$(\textbf{infP}), $9.7148^{\circ}$(\textbf{Ch17}), $14.7723^{\circ}$(\textbf{SDP17}), and $ 11.9024^{\circ}$ (\textbf{Go20}). Results confirm our method’s superiority.

\begin{figure}[!htb]
	\minipage{\columnwidth}
	\centering
	\hspace{.1cm}
	\subfloat[\footnotesize{Shape 1}]{\includegraphics[width=0.45\columnwidth,height=1.3in]{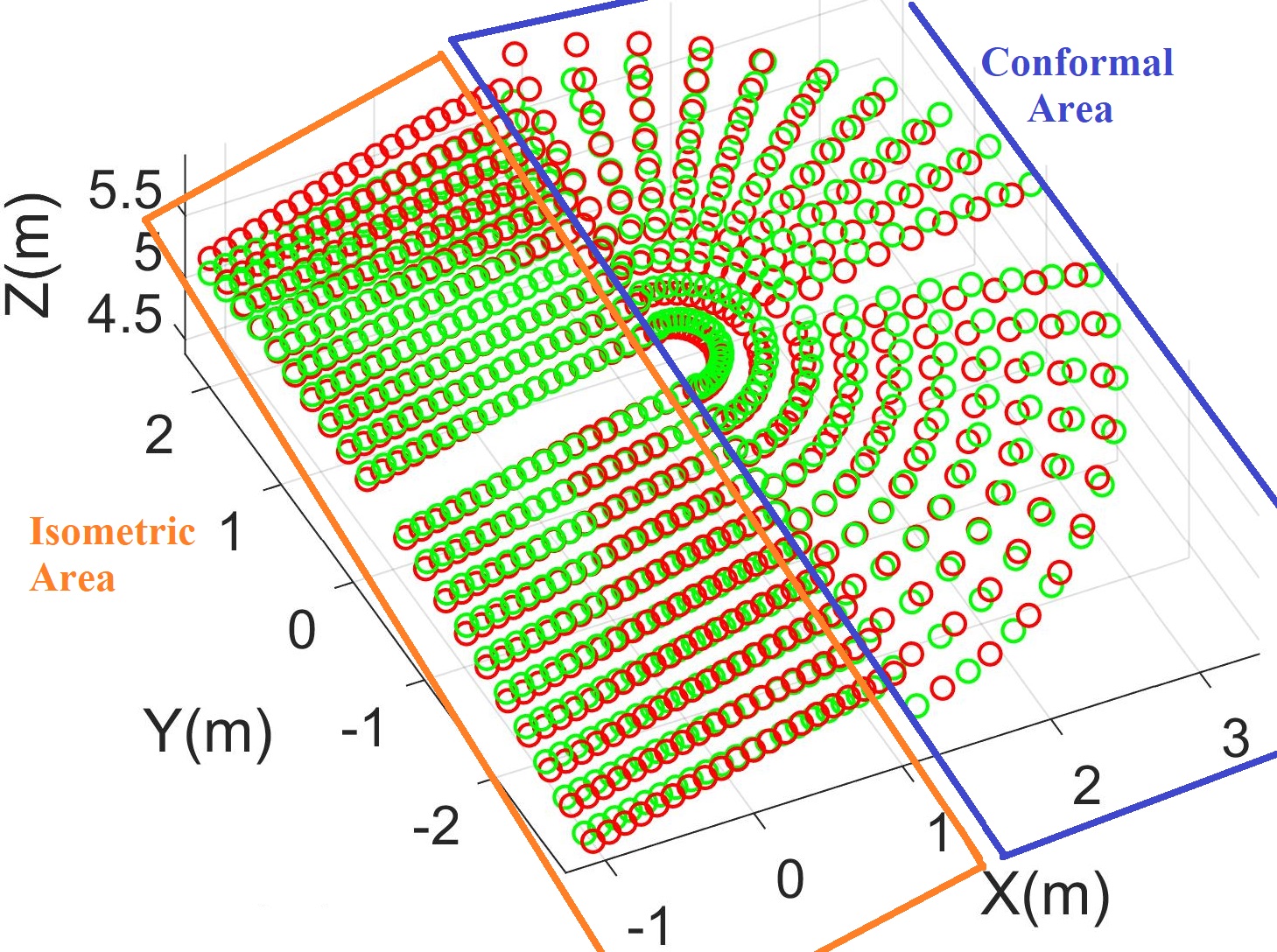}\label{fig:graph8ass}}
	\hspace{.1cm}
	\subfloat[\footnotesize{Shape 5}]{\includegraphics[width=0.45\columnwidth,height=1.3in]{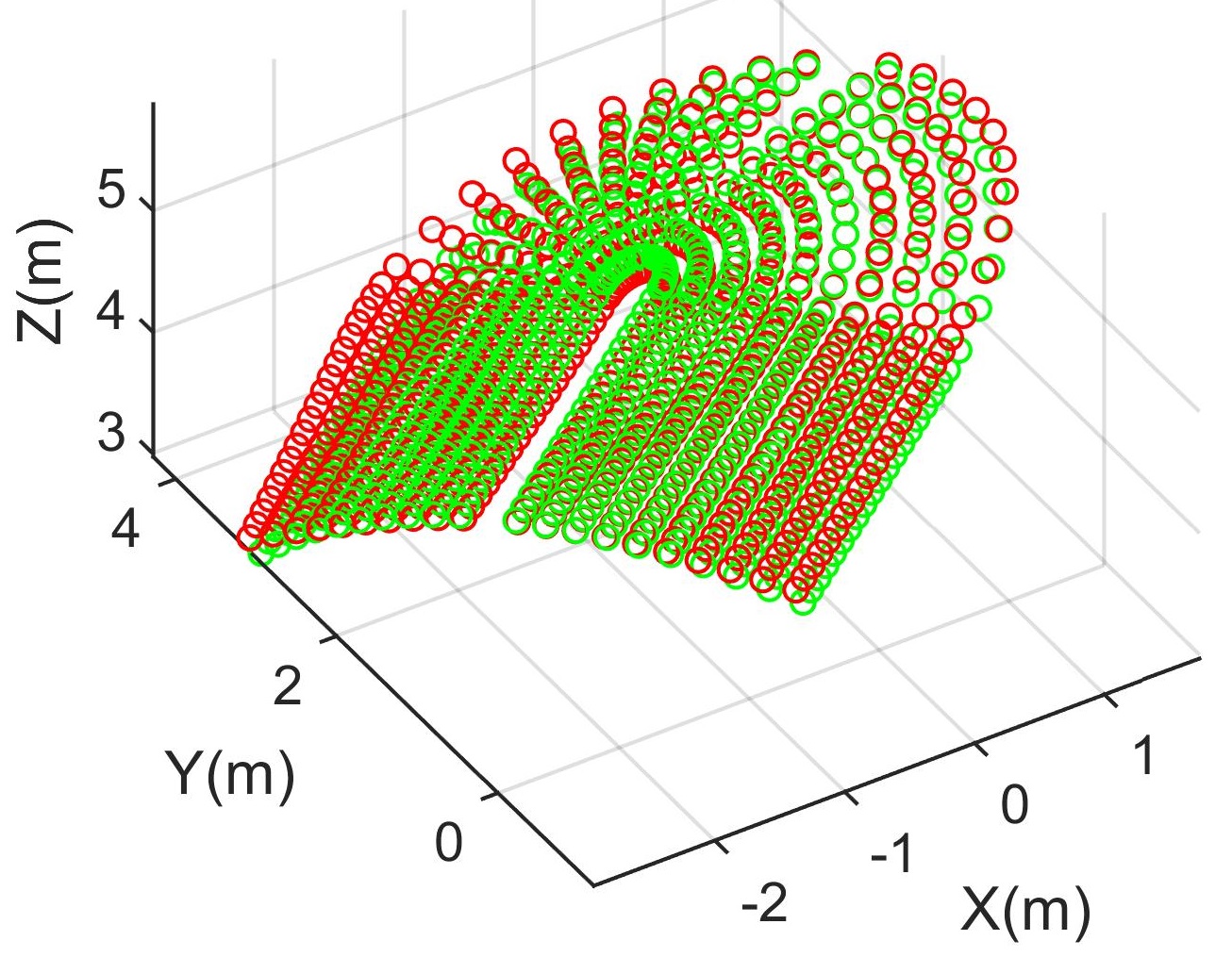}\label{fig:graph9ass}}
	\caption{{NRSfM ground truth (green) data and reconstructed results (red) for the synthetic dataset 2.}}
	\label{fig:graph5and8ass}
	\endminipage
\end{figure}

\subsection{Isometric Real Datasets}\label{62}
In this part, we perform experiments using both short-term and long-term real datasets to compare with the other five methods. The T-shirt dataset with 10 images and the Flag dataset with 30 images are the short-term datasets and the long-term datasets include  Rug and KinectPaper datasets.

\textbf{T-shirt dataset}: The T-shirt dataset~\cite{10} consists of 85 manually set
features correspondences across 10 different images of a T-shirt deforming isometrically. The features belong to the vision range $\mathcal{R}_v$ and the dataset includes the corresponding RGB images, so we implement our proposed method with the network in Section \ref{s62}. Using the full datasets and all the features, we recover the shape, generate, and texture the dense point cloud, as shown in \figurename~\ref{figurelabe2}.
\begin{figure}[!ht]
	\minipage{\columnwidth}
	\begin{center}
		\includegraphics[width=0.95\linewidth]{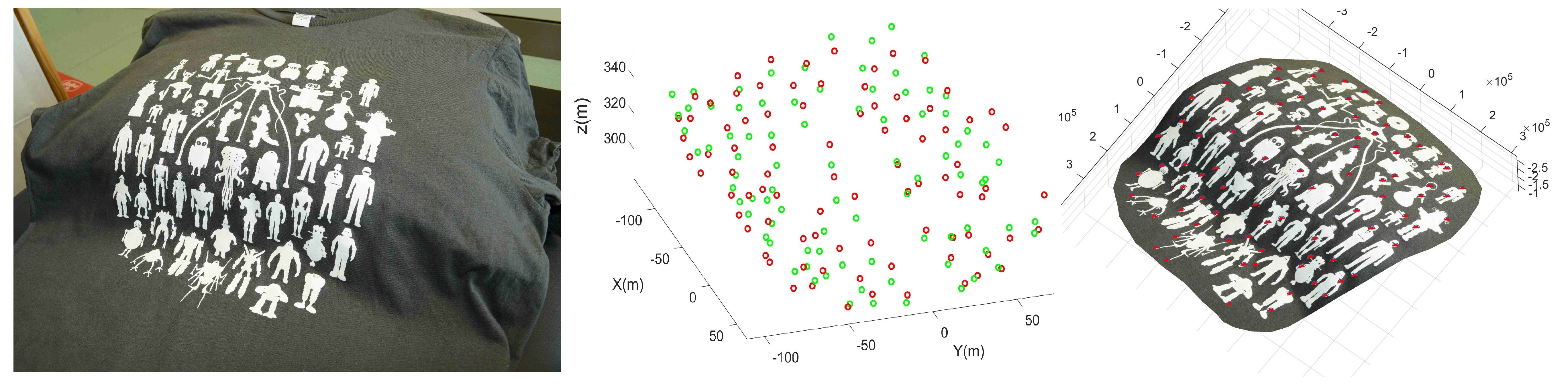}
	\end{center}
	\caption{The reconstructed dense point cloud with texture.}
	\label{figurelabe2}
	\endminipage\\
	\hspace{.1cm}
	\minipage{\columnwidth}
	\centering
	\subfloat[\footnotesize{$\%$ 3D error}]{\includegraphics[width=0.48\columnwidth,height=1.4in]{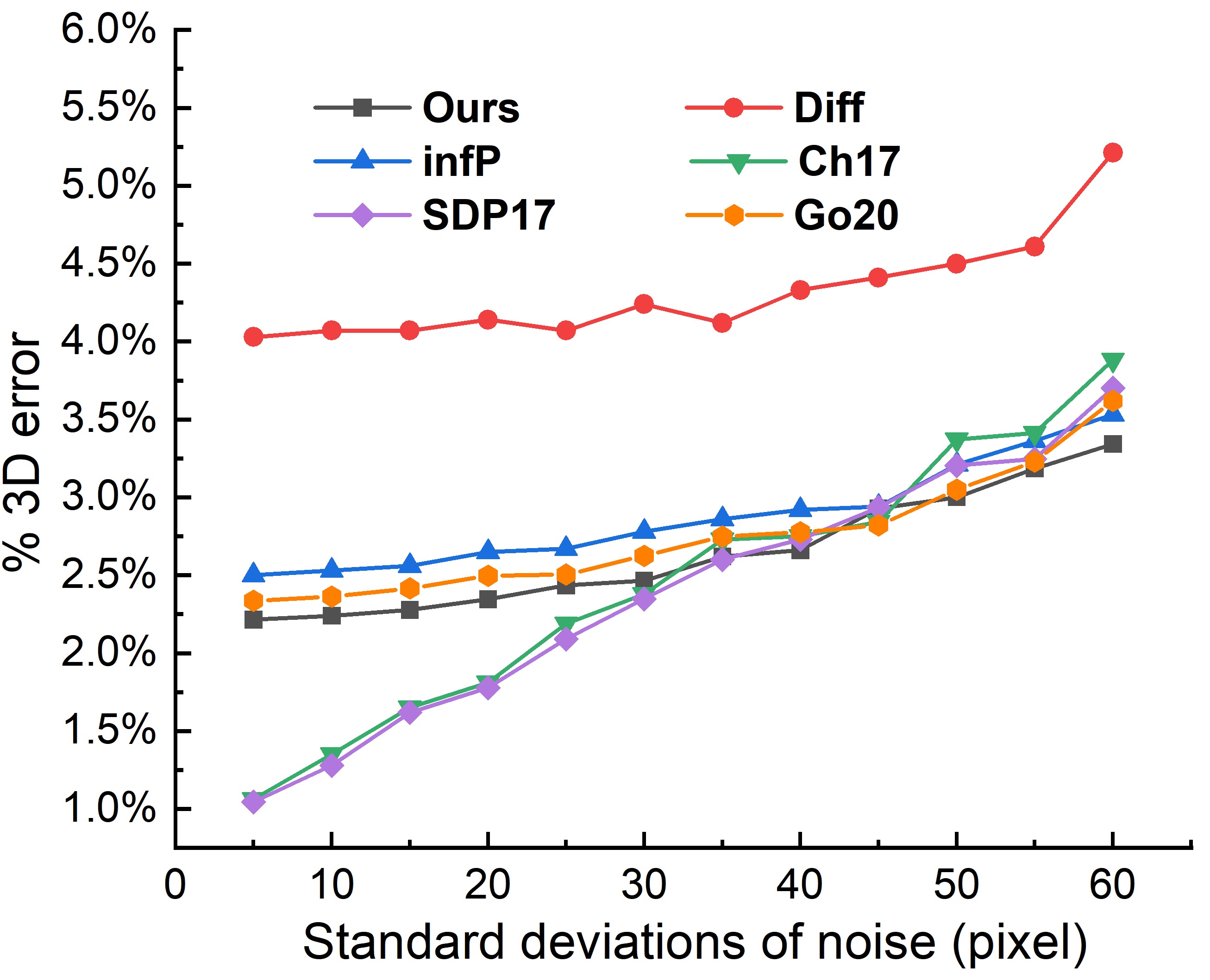}\label{fig:graph10}}
	\hspace{.1cm}
	\subfloat[\footnotesize{Shape error (degrees)}]{\includegraphics[width=0.48\columnwidth,height=1.4in]{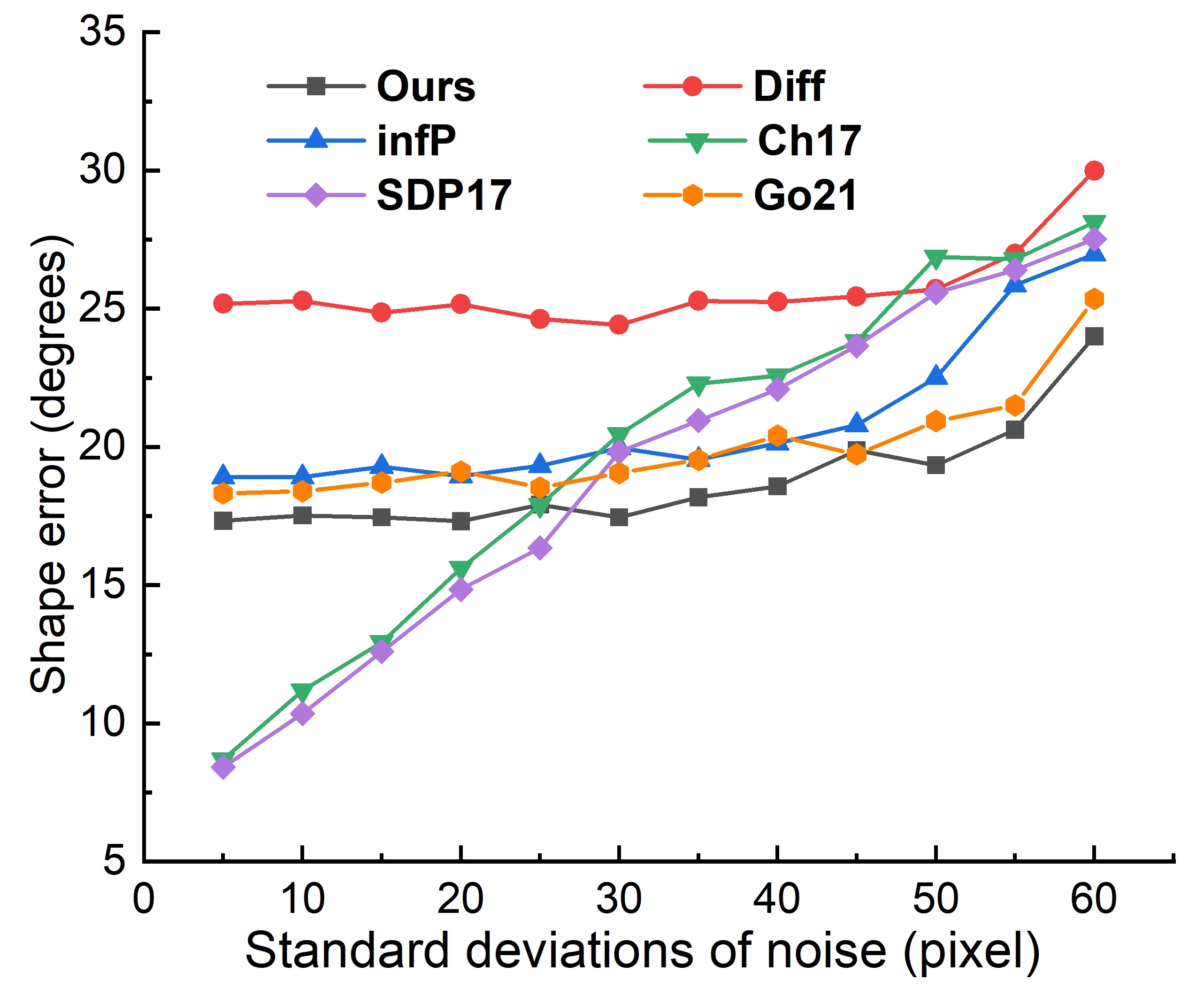}\label{fig:graph11}}
	\caption{Comparison results of the mean shape error and $\%$ 3D error using the T-shirt dataset with different levels of noise. }
	\label{fig:graph10and11}
	\endminipage
\end{figure}

In T-shirt dataset, we also show the accuracy of the obtained sparse point cloud using the integrating way. The mean \%3D errors of these methods are respectively 2.2401\%(\textbf{Ours}), 4.0234\%(\textbf{Diff}), 2.4932\%(\textbf{infP}), 0.8789\%(\textbf{Ch17}),  \textbf{0.7979}\% (\textbf{SDP17}), and 2.3296\%(\textbf{Go20}). The average shape errors are respectively  $17.5239^{\circ}$(\textbf{Ours}), $25.0743^{\circ}$(\textbf{Diff}), $7.2102^{\circ}$(\textbf{infP}), $6.9904^{\circ}$(\textbf{Ch17}), $\textbf{6.4235}^{\circ}$(\textbf{SDP17}), and $18.3721^{\circ}$(\textbf{Go20}). For this dataset, the \textbf{SDP17} and \textbf{Ch17} methods show the best performance. Our method, \textbf{Go20}, and  \textbf{infP} show a stable and fair performance.  The  \textbf{Diff} method does not work well on this dataset. T-shirt dataset is almost isometric (a special case of conformal). \textbf{Ch17} and \textbf{SDP17} use only point correspondences whereas \textbf{Ours} also uses their first and second-order derivatives. T-shirt has only 85 manually clicked point correspondences on 10 images, thus the computation of the first- and (especially) second-order derivatives is inaccurate. There is practically no noise on this dataset. This is why  \textbf{Ch17} and \textbf{SDP17} perform better than \textbf{Ours} on T-shirt.

When 50\% point data are missing from every image, including the first image, the mean \%3D errors (in \%) of these methods are respectively 3.05\% (\textbf{infP}), \textbf{2.77}\% (\textbf{Ours}), 3.82\% (\textbf{Diff}), 7.02\% (\textbf{Ch17}), 6.68\% (\textbf{SDP17}), and 2.78\% (\textbf{Go20}). The mean shape error are $21.05^\circ$ (\textbf{infP}), $\textbf{20.95}^\circ$ (\textbf{Ours}), $25.67^\circ$ (\textbf{Diff}), $34.20^\circ$ (\textbf{Ch17}), $30.12^\circ$ (\textbf{SDP17}), and ${20.98}^\circ$ (\textbf{Go20}). Our results indicate that the \textbf{Ch17} and \textbf{SDP17} methods are not reliable in the presence of many missing points, which is very common in  cases of the tracking lost in the real applications of NRSfM, such as deformable SLAM~\cite{6}.

We add Gaussian noises, of which the standard deviations range from 5 pixels to 50 pixels, to this dataset. The image size in this dataset is $4800\times3200$, so these levels of noise are reasonable. These six methods are compared with each other. The changes of the mean \%3D error (in \%) and the mean shape error (in $^\circ$) are shown in \figurename~\ref{fig:graph10and11}. 
We can find that our method is robust to noisy data and the \textbf{SDP17} and \textbf{Ch17} methods are sensitive to the noises.

\textbf{Flag dataset}: Based on a real flag and a virtual perspective camera, the Flag dataset~\cite{41} is a self-synthetic
data with 250 tracking features. We randomly select 30 images from the whole dataset with 450 frames and add Gaussian noises (with a standard deviation of 4 pixels) to each feature. The comparison results of the \% 3D error are shown in \figurename~\ref{fig:graph12and13}. The examples of the reconstructed results are shown in \figurename~\ref{fig:graph6and7_addx}.

\begin{figure}[!htb]
	\minipage{\columnwidth}
	\begin{center}
		\includegraphics[width=0.70\linewidth,height=1.8in]{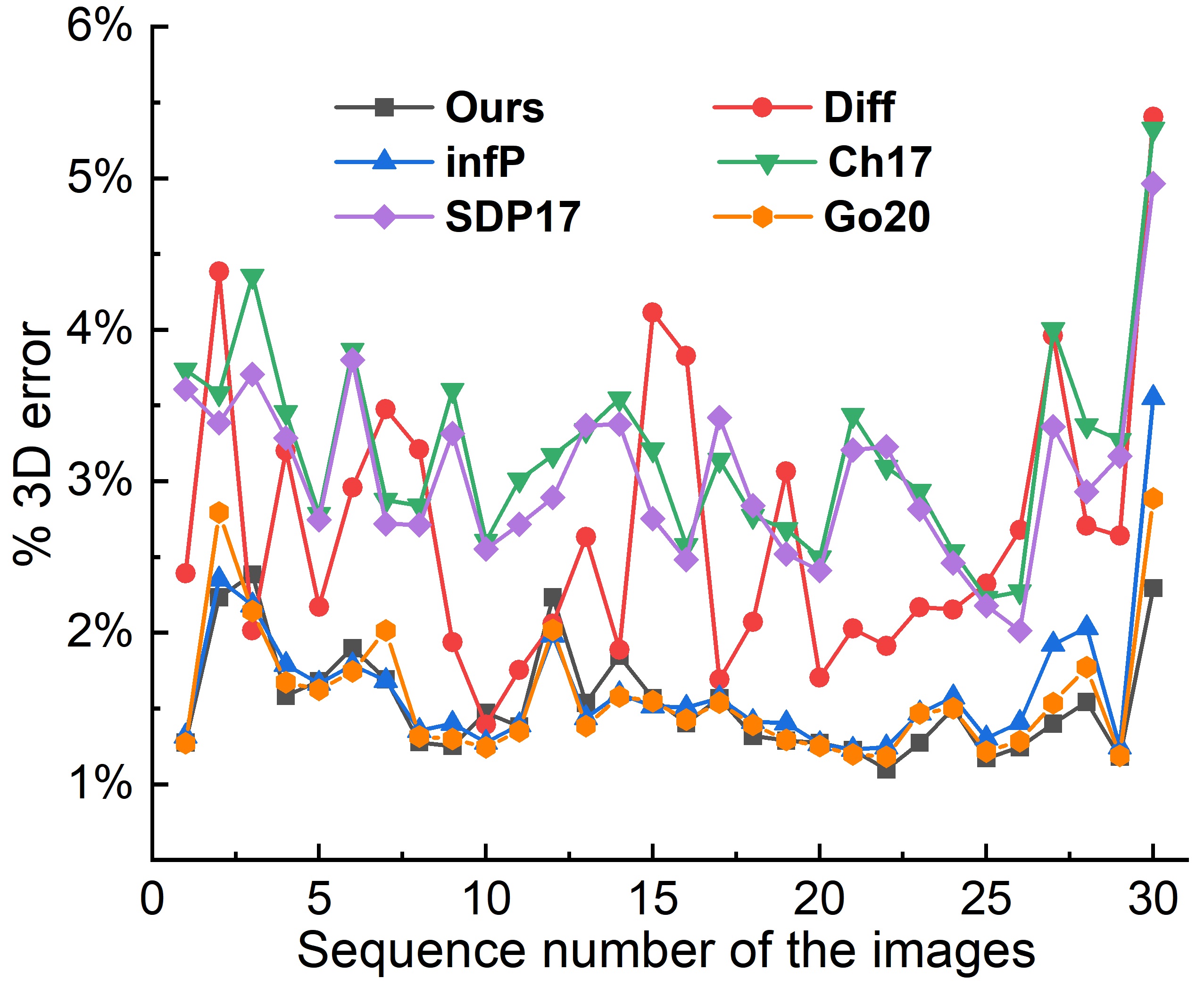}
	\end{center}
	\caption{Comparison results of the mean $\%$ 3D error using the Flag dataset.}
	\label{fig:graph12and13}
	\endminipage\\
	\minipage{\columnwidth}
	\begin{center}
	\includegraphics[width=\linewidth]{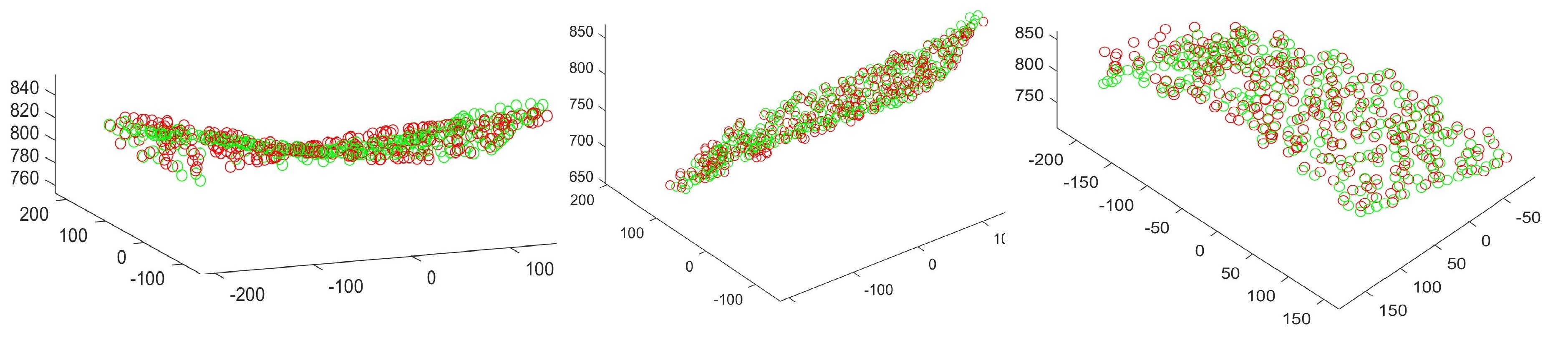}
	\end{center}
	\caption{The ground truth (green) and reconstructed results (red) of the Flag dataset.}
	\label{fig:graph6and7_addx}
	\endminipage\\
\end{figure}


The mean  \% 3D errors are respectively \textbf{1.5364}\% (\textbf{Ours}),  2.6631\% (\textbf{Diff}), 1.6289\% (\textbf{infP}), 3.0996\% (\textbf{Ch17}), 3.0309\% (\textbf{SDP17}), and 1.5710\% (\textbf{Go20}). These results show that our method has the best performance on
this dataset. \textbf{Ch17} and \textbf{SDP17} have the largest errors as compared to other methods. 

The following two long-term datasets, including Rug and KinectPaper datasets, are used to further verify the performance of our method in the isometric cases. Because the SOCP and SDP formulations are solved using the inner-point method, \textbf{Ch17} and \textbf{SDP17} take more than 3 hours for 60 images with 300 features. Therefore, we split the long-term sequences to sets of 30 images (Rug dataset) and 15 images (KinectPaper dataset), and then evaluated these two methods. The other methods, including  \textbf{infP}, \textbf{Diff}, \textbf{Go20}, and \textbf{Ours}, are based on the full datasets and they can be solved within 1 hour given the image warps. The analysis of the computational time will be further discussed in Section~\ref{s651}.

\textbf{Rug dataset}: {The Rug dataset~\cite{6} is a public isometric data with 159 images and 300 features showing the deformed rug from different views. This is a challenging dataset with outliers and poor given correspondences, due to the low frame-rate of the recorded sequences. With no missing features, the comparison results of the \% 3D error are shown in \figurename~\ref{fig:graph6and7_add}.} Reconstructed results by \textbf{Ours} are illustrated in \figurename~\ref{fig:graph6and7_add1x}.

\begin{figure}[!htb]
	\minipage{\columnwidth}
	\begin{center}
		\includegraphics[width=0.85\linewidth]{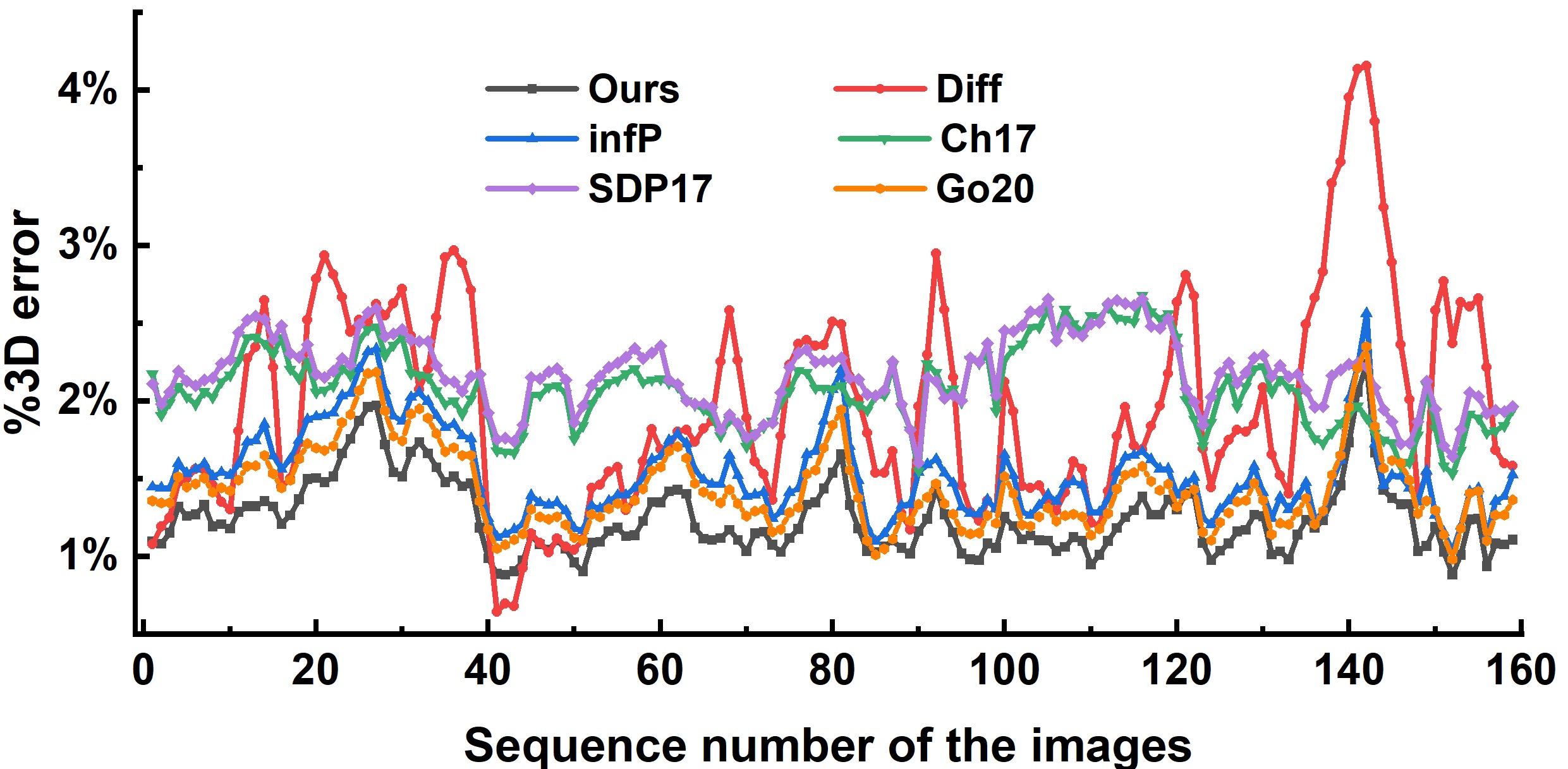}
	\end{center}
	\caption{The comparison results of the Rug dataset.}
	\label{fig:graph6and7_add}
	\endminipage\\
	\minipage{\columnwidth}
	\begin{center}
		\includegraphics[width=0.9\linewidth]{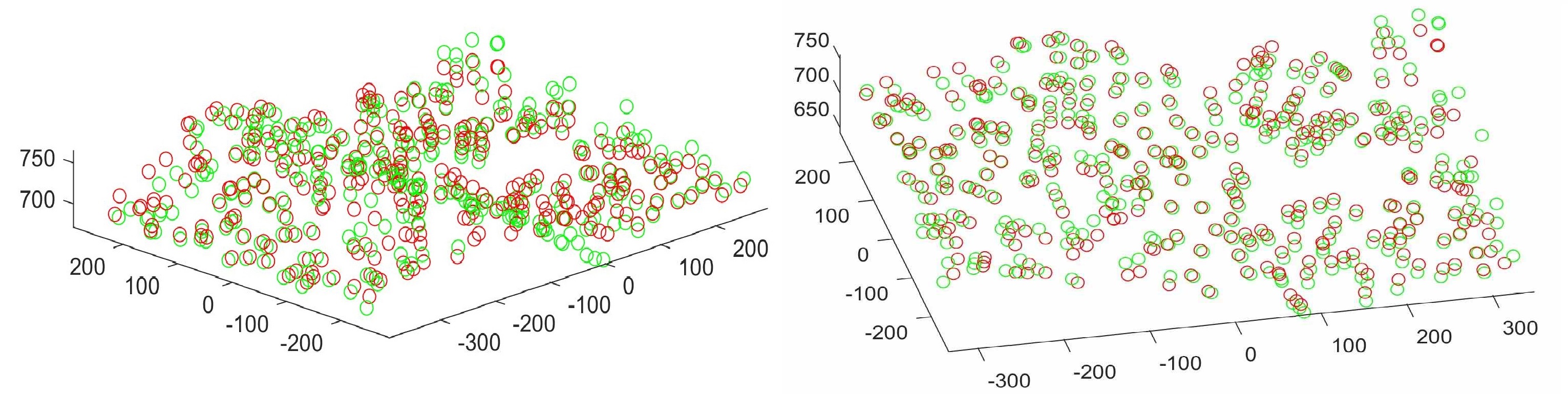}
	\end{center}
	\caption{The ground truth (green) and reconstructed results (red) of the Rug dataset.}
	\label{fig:graph6and7_add1x}
	\endminipage
\end{figure}

The mean \% 3D errors of all the images are respectively
\textbf{1.2511}\% (\textbf{Ours}),  1.970\% (\textbf{Diff}), 1.5286\% (\textbf{infP}), 2.082\% (\textbf{Ch17}), 2.1748\% (\textbf{SDP17}),  and 1.4296\% (\textbf{Go20}). It is clearly visible that our method performs better than the other methods.

\textbf{KinectPaper dataset} \cite{6}: It is a long sequence with 191 images and 1503 features. It shows the isometric deformations of a paper. Similar to the Rug dataset, this sequence also contains outliers. Six methods are applied to this dataset in order to compare their performance. Based on all the visible features, the comparison results of the \% 3D error are shown in \figurename~\ref{fig:graph6and7_add3}. The ground truth and the reconstructed results for the KinectPaper dataset are presented in \figurename~\ref{fig:graph6and7_add3x}.

\begin{figure}[!htb]
	\minipage{\columnwidth}
	\begin{center}
		\includegraphics[width=0.85\linewidth]{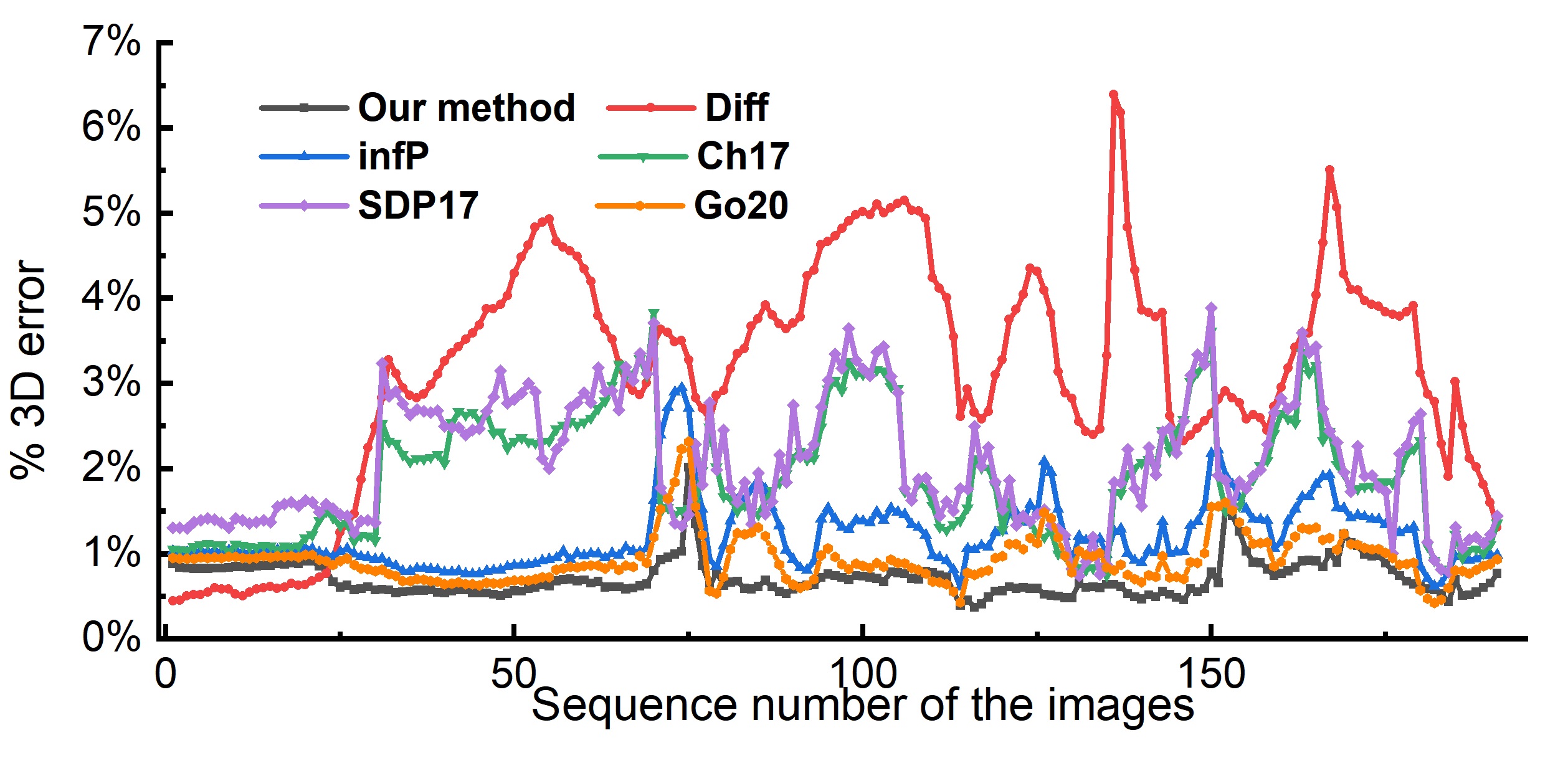}
	\end{center}
	\caption{The comparison results of the KinectPaper dataset.}
	\label{fig:graph6and7_add3}
	\endminipage\\
	\minipage{\columnwidth}
	\begin{center}
		\includegraphics[width=0.9\linewidth]{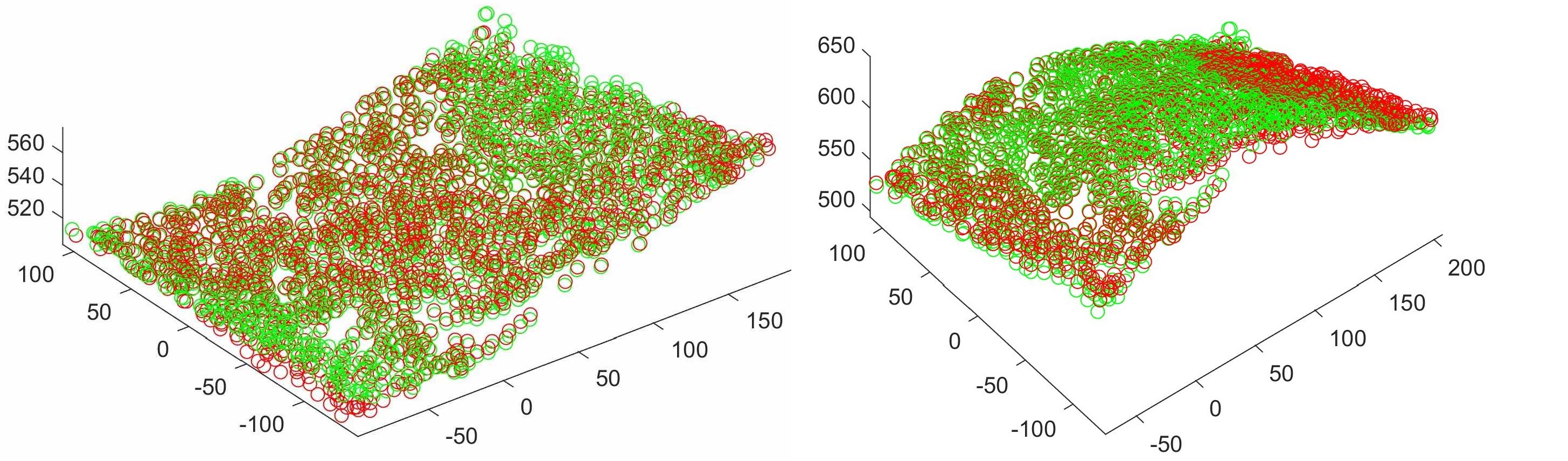}
	\end{center}
	\caption{The reconstructed results of the KinectPaper dataset.}
	\label{fig:graph6and7_add3x}
	\endminipage
\end{figure}
The mean \% 3D errors are respectively
\textbf{0.7011}\% (\textbf{Ours}),   3.1577\% (\textbf{Diff}), 1.2106\% (\textbf{infP}), 1.9164\% (\textbf{Ch17}), 2.0804\% (\textbf{SDP17}), and 0.9294\% (\textbf{Go20}). These results demonstrate that our method evidently outperforms the others.
\subsection{Real Datasets with More Generic Deformations}
	A key advantage of our method is its applicability to both isometric and conformal deformations. We therefore test it on the NRSfM Challenge Dataset~\cite{42} and a self-collected conformal dataset.

\textbf{NRSfM Challenge Dataset}. We apply our method in the NRSfM Challenge Dataset, which includes 5 image sequences called \emph{Articulated Joints}, \emph{Balloon Deflation}, \emph{Paper Bending}, \emph{Rubber Stretching}, and \emph{Paper Tearing}. They represent 5 different classical non-rigid deformations: articulated (piecewise-rigid), balloon (conformal), paper bending (isometric), rubber (elastic), and  paper being torn. Because of the virtual camera, the vision range is commonly outside the vision range $\mathcal{R}_v$ and only the integrating method is applied in this part. The image features are obtained using 6 different camera motions with orthogonal and perspective projections and the ground-truth for one frame is provided for each sequence. Our method is based on the assumption that the deforming surface is a manifold, for the stereoscopic datasets (like balloon), the back-side features will break the one-by-one mapping (image embedding) between our image and the visible surface. Therefore, we test the balloon and two Paper datasets without using missing (invisible) features and the Articulated and Stretching datasets using full features. The comparison results, including score, ground truth (green), and reconstructed shapes (red), among \textbf{Ours}, \textbf{Ch17}, \textbf{Pa21-R}~\cite{43}, \textbf{Pa21-S}~\cite{43}, \textbf{An17}~\cite{44}, \textbf{Lee16}~\cite{45}, \textbf{Diff}, \textbf{Closed}~\cite{31}, and \textbf{Best} (\textbf{Best} means the one that does best as reported in the benchmark statistics provided on the website~\cite{31}) are presented in Table~\ref{NRSfM}. The reconstructed results are shown in \figurename~\ref{fig:graph6and7_add2}.
\begin{table}[!ht]
	\small
	\caption{Results on the NRSfM challenge datasets.}
	\label{NRSfM}
	\begin{center}
		\begin{tabular}{|c|c|c|c|c|c|}
			\hline 
			Camera&\multicolumn{5}{c|}{NRSfM Challenge Dataset (Perspective projection)}\\
			\hline
			Method&Articulated&Balloon&Bending&Rubber &Tearing\\
			\hline
   			Case &full&missing&missing&full &missing\\
			\hline
			\textbf{Ch17}&91.6&53.5&63.8&62.5& 51.9
			\\
			\hline
			\textbf{An17}&65.1&55.2& 64.7 &48.1  &  50.8
			\\
			\hline
			\textbf{Lee16}&105.5&-& -&70.3&-\\
			\hline
			 \textbf{Pa21-R}&25.1&41.6&40.5&20.6&28.7\\
			\hline
			\textbf{Pa21-S}&26.0&40.7&40.4&20.6 &27.8\\
			\hline
			\textbf{Diff}&\textbf{21.3}& 37.8& 39.0 &30.3 &23.8\\
			\hline
			\textbf{Best}&40.7& 35.7& 39.0 &30.3 &24.9\\
   			\hline
			\textbf{Closed}&21.8& 38.1& 38.2 &17.2 &23.1\\
			\hline 
			\textbf{Ours}&22.8& \textbf{17.9}& \textbf{21.8} &\textbf{16.9} &\textbf{15.3}\\
			\hline 
			Camera&\multicolumn{5}{c|}{NRSfM Challenge Dataset (Orthographic projection)}\\
			\hline 
			\textbf{Ch17}&{88.7}& 52.6 & 65.0 & 66.3 & 57.2 \\
			\hline 
			\textbf{An17}&58.1&46.4 & 50.3 & 38.9 & 38.4\\
			\hline 
			\textbf{Lee16}&105.3&-&-&69.2 &-\\
			\hline 
			\textbf{Pa21-R}&21.8& 27.2&32.3 &23.1 & 20.7 \\
			\hline 
			\textbf{Pa21-S}&22.0& 27.3& 32.2& 33.0&  21.0\\
			\hline 
			\textbf{Diff}&18.7&33.6 & 34.0&\textbf{17.1} & 18.8\\
			\hline 
			\textbf{Best}&35.5&33.8 & 37.0& 22.9& 18.3\\
            \hline
			\textbf{Closed}&20.1& 26.8& 32.1 &17.2 &18.3\\
			\hline 
			\textbf{Ours}&\textbf{13.2}& \textbf{17.1}& \textbf{21.3} &19.3 &\textbf{17.0}\\
			\hline
		\end{tabular}\\
	\end{center}
\footnotesize{$-$	indicates that method failed to return a result due to missing data. Most of these compared results are obtained from Table 5 in Reference \cite{31} considering full or missing cases.}
\end{table}

\begin{figure}[!ht]
	\begin{center}
		\includegraphics[width=0.85\linewidth]{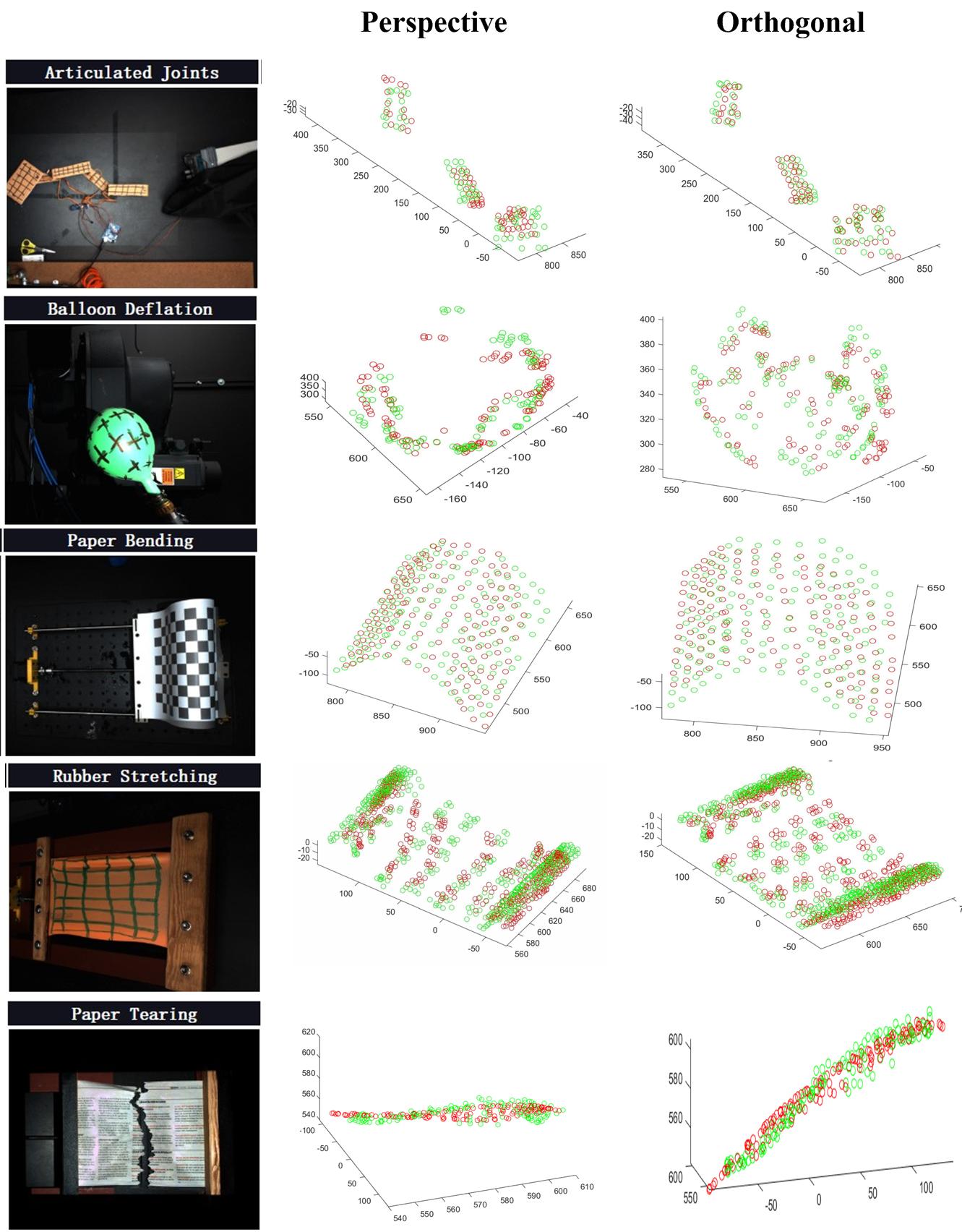}
	\end{center}
	\caption{NRSfM challenge dataset, ground truth (green), and reconstructed results (red) using \textbf{Ours}.}
	\label{fig:graph6and7_add2}
\end{figure}

Results show our method performs robustly and competitively with SOTA in both perspective and orthographic cases on datasets with generic deformations.

\textbf{Self-Collected Conformal Dataset}. Existing real-world datasets capturing conformal deformation are limited. To evaluate the performance of our framework under elastic deformation, we collected a custom dataset using the Intel RealSense D435i depth camera. This dataset consists of 13 images of a deforming balloon. To enhance texture and facilitate feature detection, visual markers were affixed to the balloon’s surface. Sparse 2D features were extracted using the SIFT, and feature correspondences were established across frames. The corresponding depth values from the depth images were used as ground truth\footnote{The D435i depth camera offers about 5–40 mm accuracy within 1–2 meter range, limiting ground truth precision.}. In this dataset, the mean 3D reconstruction errors for each method are as follows: \textbf{2.2634}\% (\textbf{Ours}),  3.2822\% (\textbf{infP}), N/A (\textbf{Ch17}), 2.9578\% (\textbf{Pa21-S}), and 3.4588\% (\textbf{Go20}) and the corresponding average shape errors are: $\textbf{8.0563}^{\circ}$ (\textbf{Ours}), $ 26.3880^{\circ}$ (\textbf{infP}), N/A (\textbf{Ch17}), $23.4654^{\circ}$ (\textbf{Pa21-S}), and $37.1828^{\circ}$ (\textbf{Go20})\footnote{\textbf{Ch17} method encounters memory limitations on a desktop with 32 GB RAM, and thus its results are unavailable. The result for \textbf{infP} is obtained by excluding features that are observed fewer than two times.}. Our method achieves the best performance on this dataset. Fig.~\ref{fig:graphdsss} shows the reconstructed results for the first image alongside the ground truth.

\begin{figure}[!htb]
	\minipage{\columnwidth}
	\centering
	\subfloat[\footnotesize{{The first frame in self-collected conformal dataset with detected SIFT features.}}]{\includegraphics[width=0.45\columnwidth,height=1.3in]{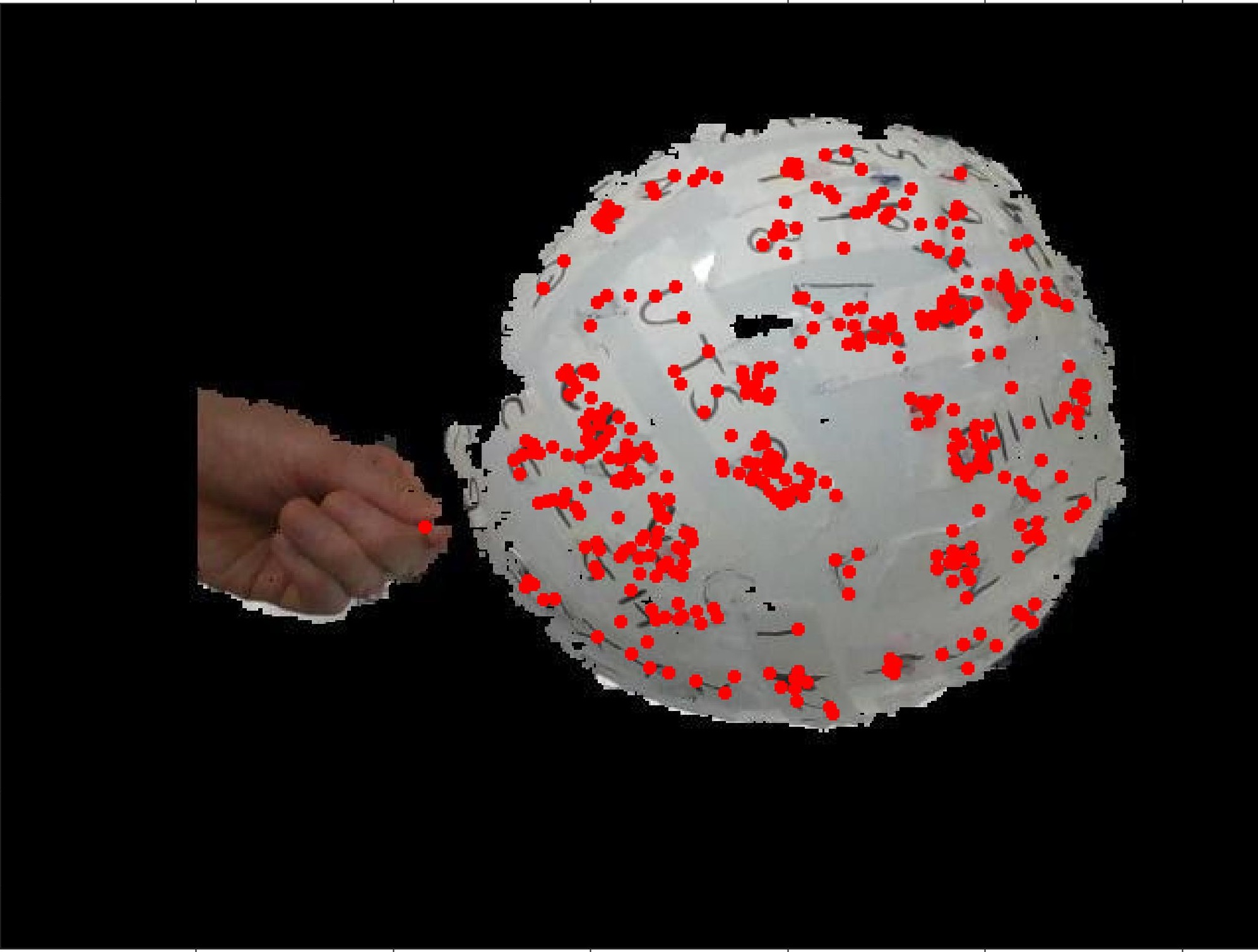}\label{fig:grapha}}
	\hspace{.1cm}
	\subfloat[\footnotesize{{Ground truth (green) and reconstructed result (red) using \textbf{Ours}.}}]{\includegraphics[width=0.48\columnwidth,height=1.4in]{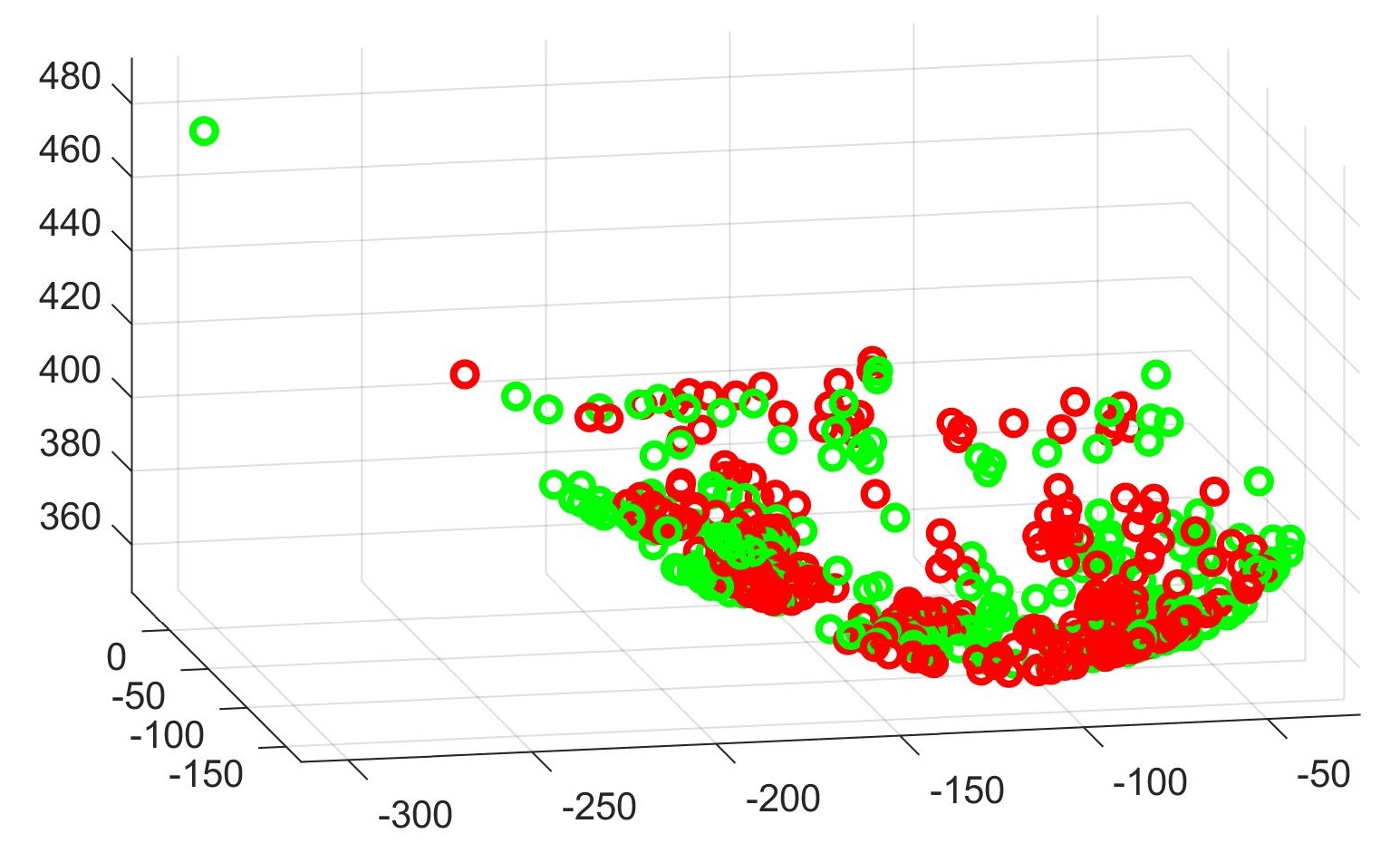}\label{fig:graphb}}\\
	\subfloat[\footnotesize{{Textured result using network}}]{\includegraphics[width=0.8\columnwidth,height=1.3in]{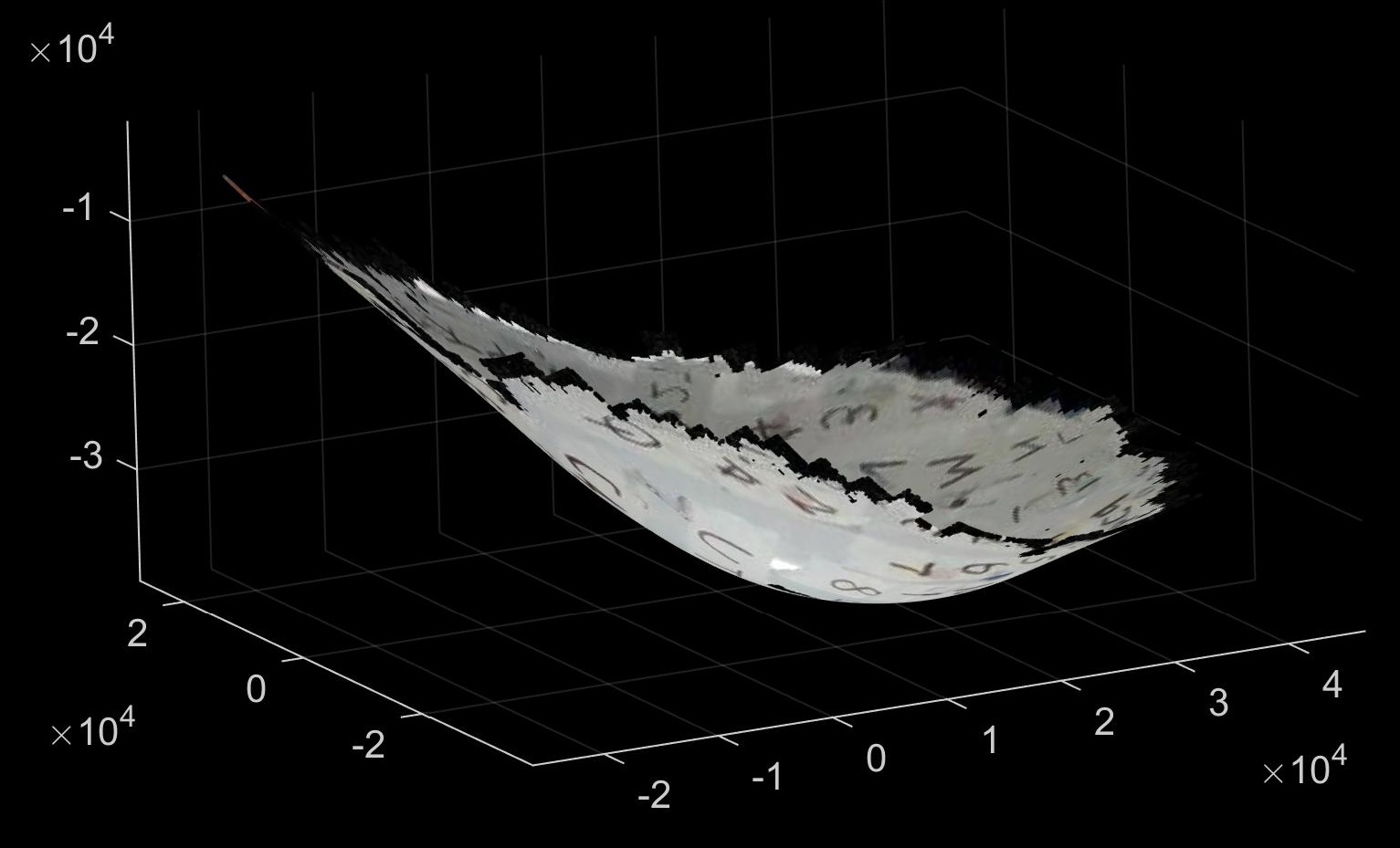}\label{fig:graphf}}
	\caption{{Reconstructed results of the self-collected dataset.}}
	\label{fig:graphdsss}
	\endminipage\\
\end{figure}

\subsection{Deformable SLAM Datasets}\label{s64}
The deformable SLAM is currently considered as an important application of the NRSfM method. In~\cite{8}, the NRSfM method is used to build the sparse local map fusing the tracking camera and the dense map combining with the template mesh. In this part, in order to further compare with other methods, the datasets, including Mandala dataset~\cite{8} and Hamlyn dataset~\cite{46,47}, are introduced using our proposed method. We use the front-end of the Def-SLAM~\cite{8} with ORB descriptor~\cite{48} to pick out and track the 2D features in the deformable environment. The ground truth depths of both datasets are obtained from the corresponding stereo sequence.

\textbf{Mandala dataset} The Mandala dataset is composed of 5 sequences ($640\times480$ pixels at 30 fps) with exploratory trajectories observing a textured kerchief deforming near-isometrically. Using the Mandala3 dataset, we generate an NRSfM dataset based on the 9-th keyframe consisting of 9 images with 411 tracking normalized features. Some features are missing during the deforming process and the example image with tracking features is shown in \figurename~\ref{fig:grapha}. The corresponding reconstructed results for the 9-th keyframe are shown in \figurename~\ref{fig:graphb} and \figurename~\ref{fig:graphf}. This dataset is very challenging with much missing data and unreliable data association. The mean \% 3D errors are respectively \textbf{1.2768}\% (\textbf{Ours}), 1.7442\% (\textbf{Diff}), 1.7389\% (\textbf{infP}), and 1.4955\% (\textbf{Go20}). The average shape errors are $\textbf{16.14972}^\circ$ (\textbf{Ours}),  $20.0113^\circ$ (\textbf{Diff}), $19.2663^\circ$ (\textbf{infP}), and $ 17.2663^\circ$ (\textbf{Go20}). The \textbf{Ch17} and \textbf{SDP17} methods generate very poor results with more than 10\% 3D error and $30^\circ$ shape errors and we can consider them as a failure to deal with this dataset. These experimental results demonstrate that our method evidently outperforms the others for this dataset.

\begin{figure}[!htb]
	\minipage{\columnwidth}
	\centering
	\subfloat[\footnotesize{The 9-th keyframe in Mandala3 dataset. }]{\includegraphics[width=0.38\columnwidth,height=1.2in]{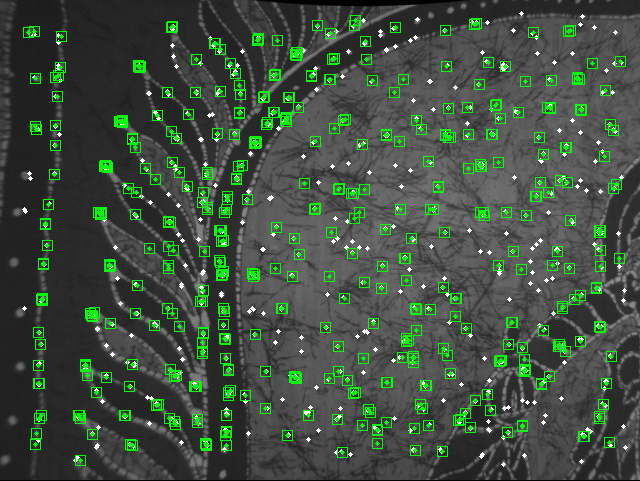}\label{fig:grapha}}
	\hspace{.1cm}
	\subfloat[\footnotesize{Ground truth (green) and reconstructed result (red) using \textbf{Ours}.}]{\includegraphics[width=0.58\columnwidth,height=1.2in]{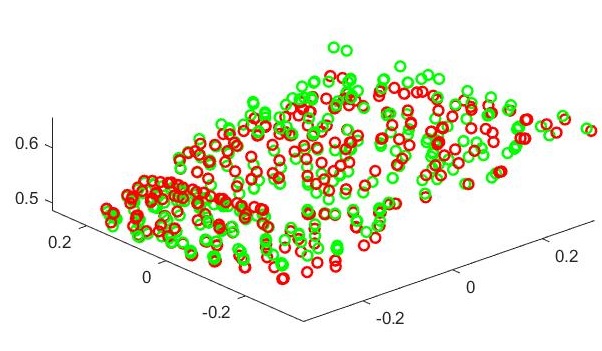}\label{fig:graphb}}\\
	\subfloat[\footnotesize{Dense result with texture using network}]{\includegraphics[width=0.8\columnwidth,height=1.3in]{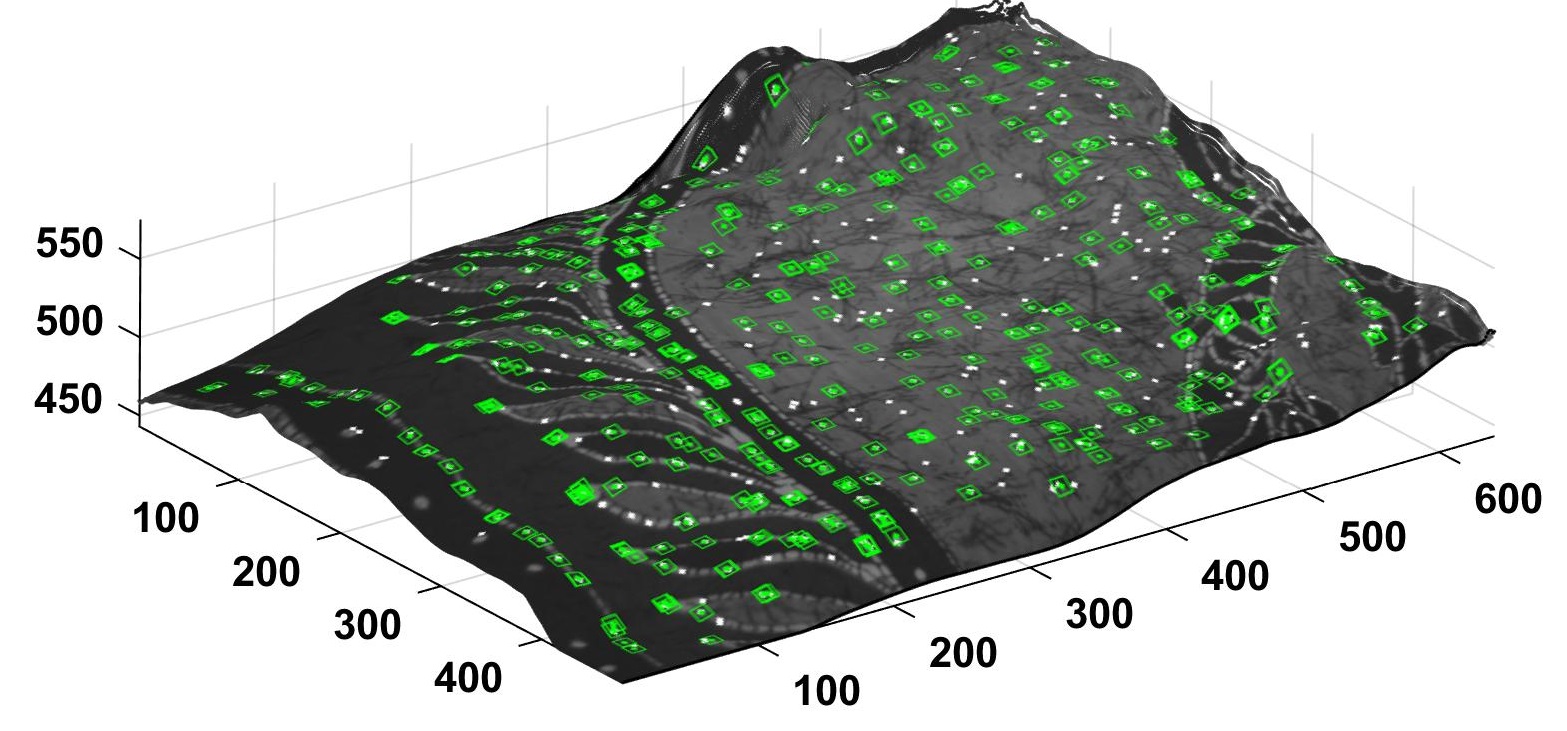}\label{fig:graphf}}
	\caption{Reconstructed results of Mandala3 dataset.}
	\label{fig:graphd}
	\endminipage\\
\end{figure}

\textbf{Hamlyn dataset} The Hamlyn dataset includes six different sequences and we pick out the SeqHeart sequence to evaluate our algorithm~\cite{8}. This sequence shows a non-rigid beating heart observed by a fixed camera. Using the deformable SLAM front-end, we generate an NRSfM dataset consisting of 7 images with 325 tracking normalized features. A lot of features are missing during the deforming process and the 4-th keyframe image with tracking features is shown in \figurename~\ref{fig:grapha}. This dataset is also very challenging with a lot of missing data and unreliable data association. The mean \% 3D errors are respectively 2.5431\% (\textbf{Ours}),  2.5415\% (\textbf{infP}), 3.4061\% (\textbf{Ch17}), 3.1733\% (\textbf{SDPP7}),  and 2.5322\% (\textbf{Go20}). The average shape errors are $27.9817^\circ$ (\textbf{Ours}), $27.9921^\circ$ (\textbf{infP}), $30.4817^\circ$ (\textbf{Ch17}), $29.4817^\circ$ (\textbf{SDP17}), and $28.1123^\circ$ (\textbf{Go20}). The results of this image sequence show that our method, the \textbf{infP} method, and the \textbf{Go20} method perform better than the others.
\begin{figure}[!htb]
	\minipage{\columnwidth}
	\centering
	\subfloat[\footnotesize{The 4-th keyframe in SeqHeart dataset. }]{\includegraphics[width=0.38\columnwidth,height=1.2in]{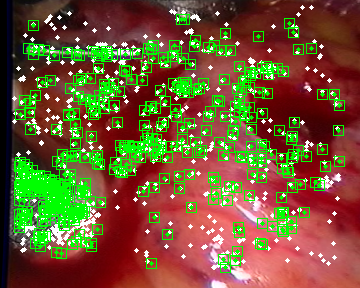}\label{fig:grapha1}}
	\hspace{.1cm}
	\subfloat[\footnotesize{Ground truth (green) and reconstructed result (red) using \textbf{Ours}.}]{\includegraphics[width=0.58\columnwidth,height=1.2in]{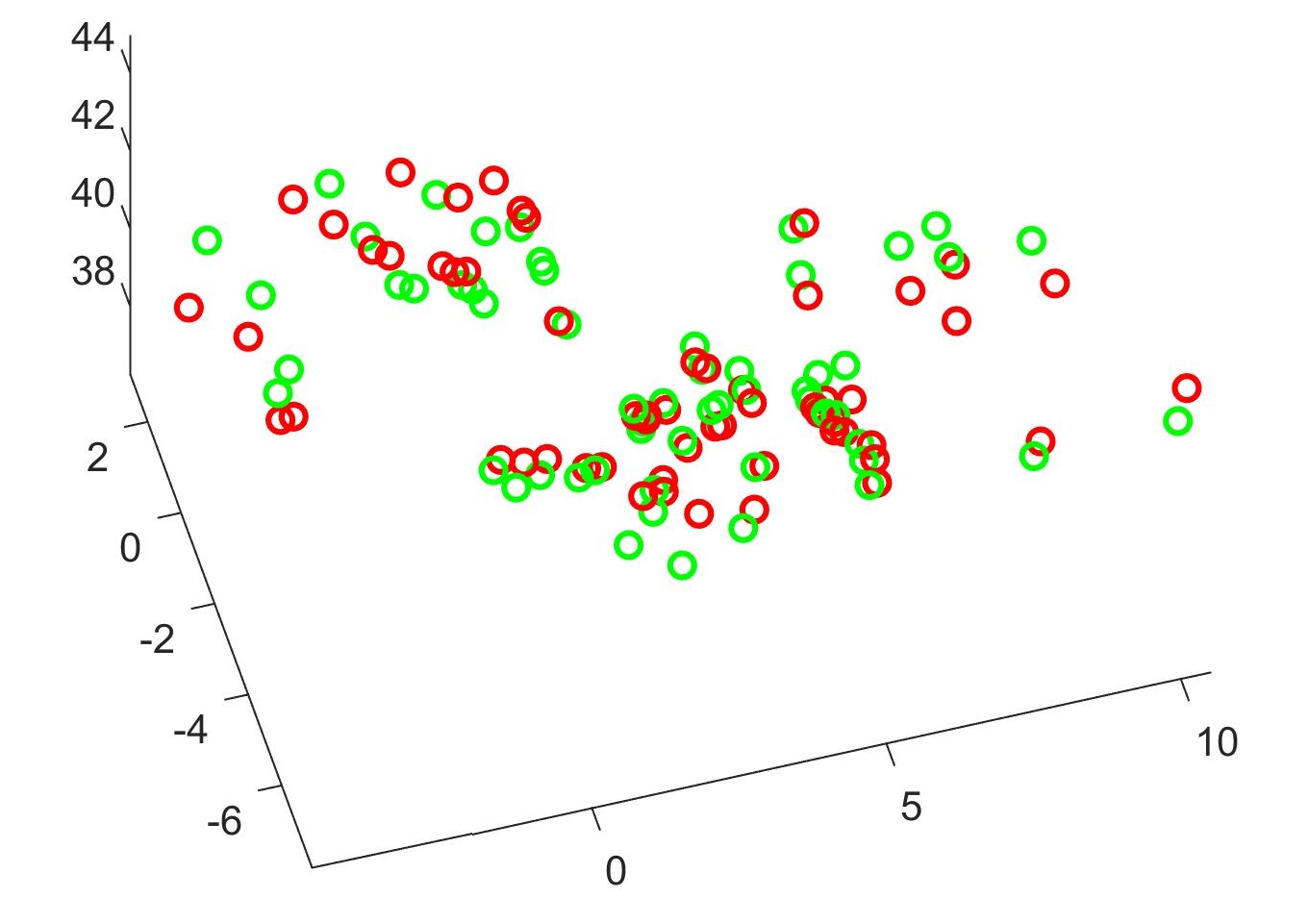}\label{fig:graphb1}}\\
	\subfloat[\footnotesize{Dense point cloud with texture using network}]{\includegraphics[width=0.8\columnwidth,height=1.3in]{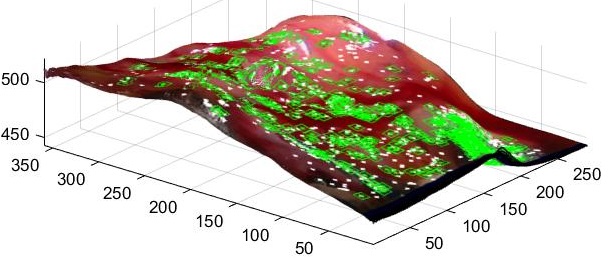}\label{fig:graphf1}}
	\caption{Reconstructed results of the 4-th keyframe of the SeqHeart sequence in the Hamlyn dataset.}
	\label{fig:graphd1}
	\endminipage\\
\end{figure}

\subsection{Dense datasets}

{Recently, unsupervised neural networks such as $\textbf{Si20}$~\cite{21add}, \textbf{LASR}~\cite{add_new_1}, and \textbf{ViSER}~\cite{add_new_2} have demonstrated superior performance compared with several classical dense NRSfM methods, including \textbf{TB}~\cite{51}, \textbf{MP}~\cite{52}, \textbf{DSTA}~\cite{53}, \textbf{GM}~\cite{54}, and \textbf{JM}~\cite{55}. To demonstrate the efficiency of the proposed method, we compare its performance with these state-of-the-art learning-based approaches using two publicly available dense datasets: Paper (Dense) and Tshirt (Dense)~\cite{21add}. Both datasets contain a large number of feature correspondences or images, which leads to high computational complexity. We evaluate the performance of \textbf{Ours} against \textbf{An17}~\cite{44}, \textbf{Pa19}~\cite{6}, \textbf{Pa21-R}~\cite{43}, \textbf{Pa21-S}~\cite{43}, \textbf{Si20}, \textbf{LASR}, and \textbf{ViSER}, and report the mean 3D error\footnote{Note that the mean 3D error used here differs from the previously used mean \% 3D error. Its definition is given in~\cite{21add}. The reported results of \textbf{LASR} and \textbf{ViSER} are obtained using their default configurations without additional tuning, and we exclude some non-front, backside points for a fair comparison. If we consider the normal field, which represents the shape, the disadvantage of these two methods becomes even more pronounced. Our method is evaluated on a downsampled dataset with 2000 features.} in Table~\ref{NRSfMx}. Fig.~\ref{fig:graphdxx} visually illustrates the reconstruction results.}


\begin{table}[!ht]
	\small
	\caption{Mean 3D error results on the dense datasets (\%).}
	\label{NRSfMx}
	\begin{center}
		\begin{tabular}{|c|c|c|c|c|c|}
			\hline
			&Paper&Tshirt&&Paper&Tshirt\\
			\hline
			\textbf{An17}&4.48&2.76&\textbf{Pa19}&5.47& 3.91 
			\\
			\hline
			\textbf{Pa21-R}&5.33&3.99&\textbf{Pa21-S}&5.52 & 4.02  
			\\
			\hline
			\textbf{Si20}&3.32&3.09 &\textbf{LASR}& 5.56  & 8.49
			\\
			\hline
			\textbf{ViSER}&3.74& 6.33& \textbf{Ours}&\textbf{3.01}   &   \textbf{2.33}
			\\
			\hline
		\end{tabular}\\
	\end{center}
\end{table}

\begin{figure}[!htb]
	\minipage{\columnwidth}
	\centering
	\subfloat[\footnotesize{{Obtained meshes (example) for Paper dataset.}}]{\includegraphics[width=0.9\columnwidth]{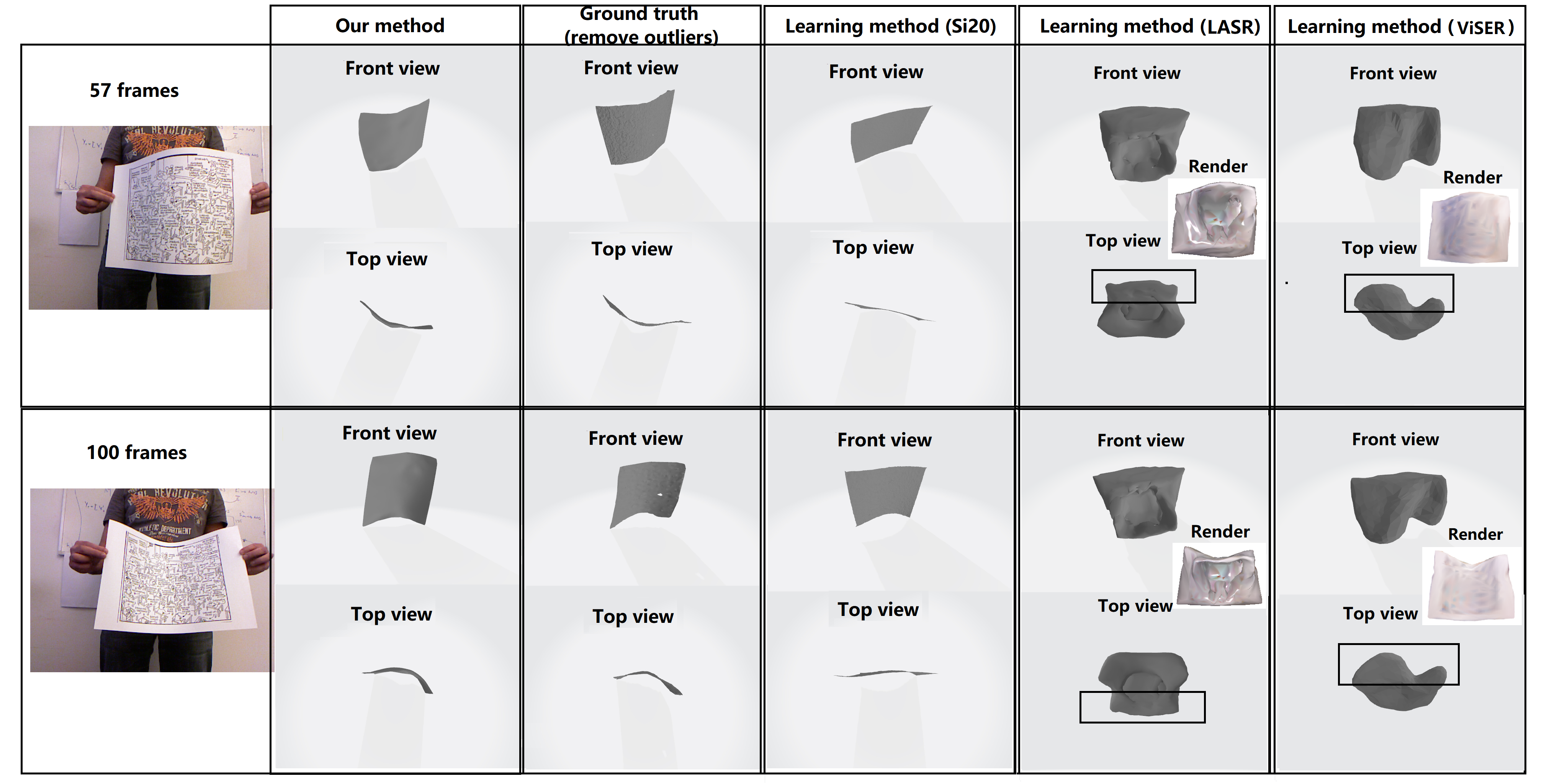}\label{fig:graphfxx}}\\
	\subfloat[\footnotesize{{Obtained meshes (example) for Tshirt dataset.}}]{\includegraphics[width=0.9\columnwidth]{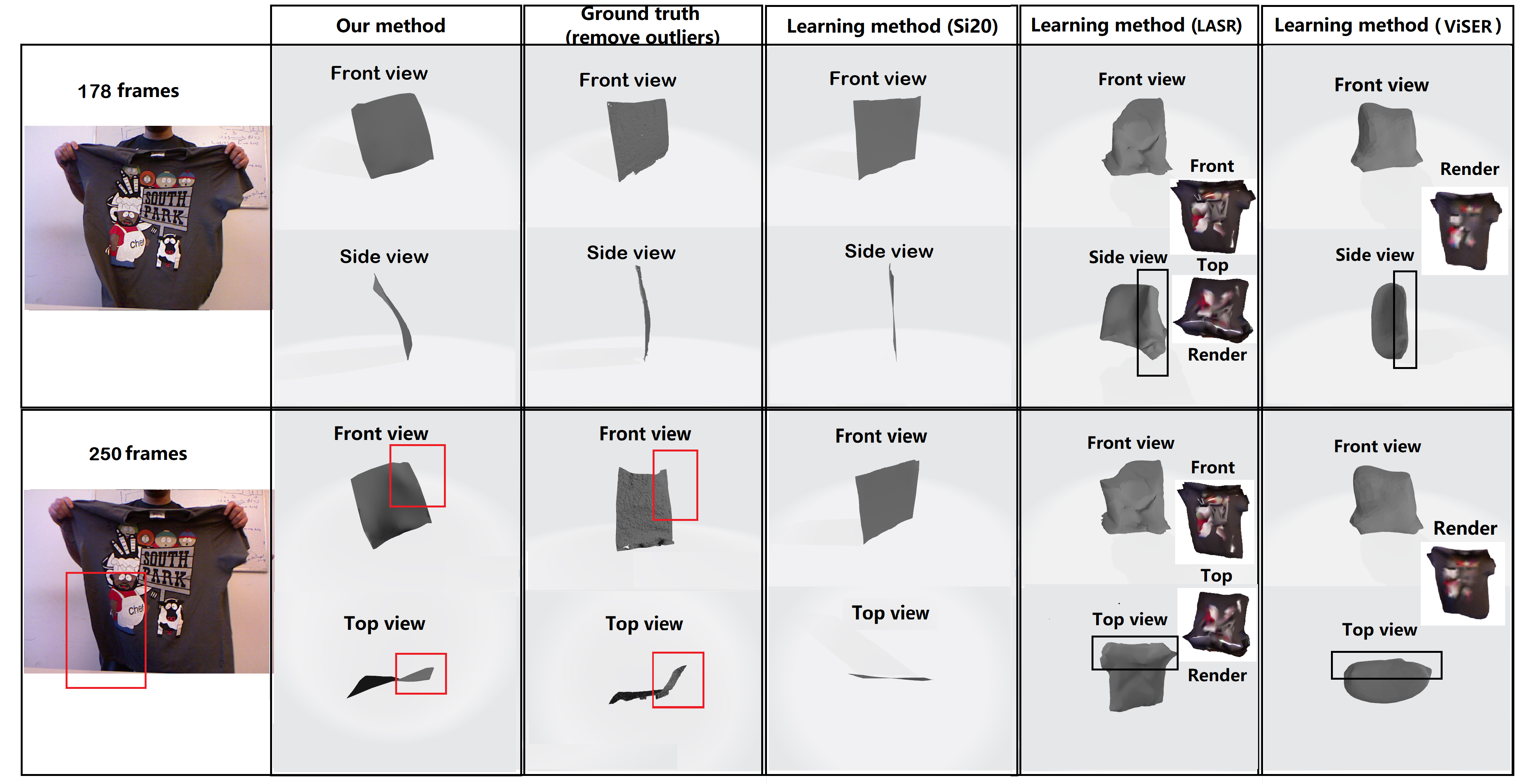}\label{fig:graphfxxx}}
\caption{Recovered results of two datasets.}
	\label{fig:graphdxx}
	\endminipage
\end{figure}


{The results show that our method achieves the best overall performance. Based on the reported results in their papers and our findings in Table~\ref{NRSfMx}, \textbf{ViSER} and \textbf{LASR} are better suited for spatially convex, deforming objects (e.g., animals and humans) that undergo large motions and are observed from diverse viewpoints, as they reconstruct shapes from a spherical prior. In other words, the target shape is topologically equivalent to a sphere. However, these methods are ill-suited for reconstructing thin Riemannian manifolds or surfaces, particularly when camera motion is small and restricted, which is the typical setting considered in this paper.} In fact, in the Paper (Dense) dataset, its ground truth includes a few outliers. All the above results in Table~\ref{NRSfMx} only remove the outliers with too large distance and the threshold is set as 100~\footnote{We reuse datasets, evaluation Matlab function, and results offered by~\cite{21add}, so the threshold is officially given.}. We find that this threshold is not enough to remove all outliers. When the threshold is reduced to 70, which helps to remove more outliers, the mean 3D error of our method will be 2.82\%. Fig.~\ref{fig:graphd1xs} shows an example of outliers in the 75th frame.
\begin{figure}[!htb]
	\minipage{\columnwidth}
	\centering
	\subfloat[\footnotesize{remain outliers}]{\includegraphics[width=0.31\columnwidth,height=1in]{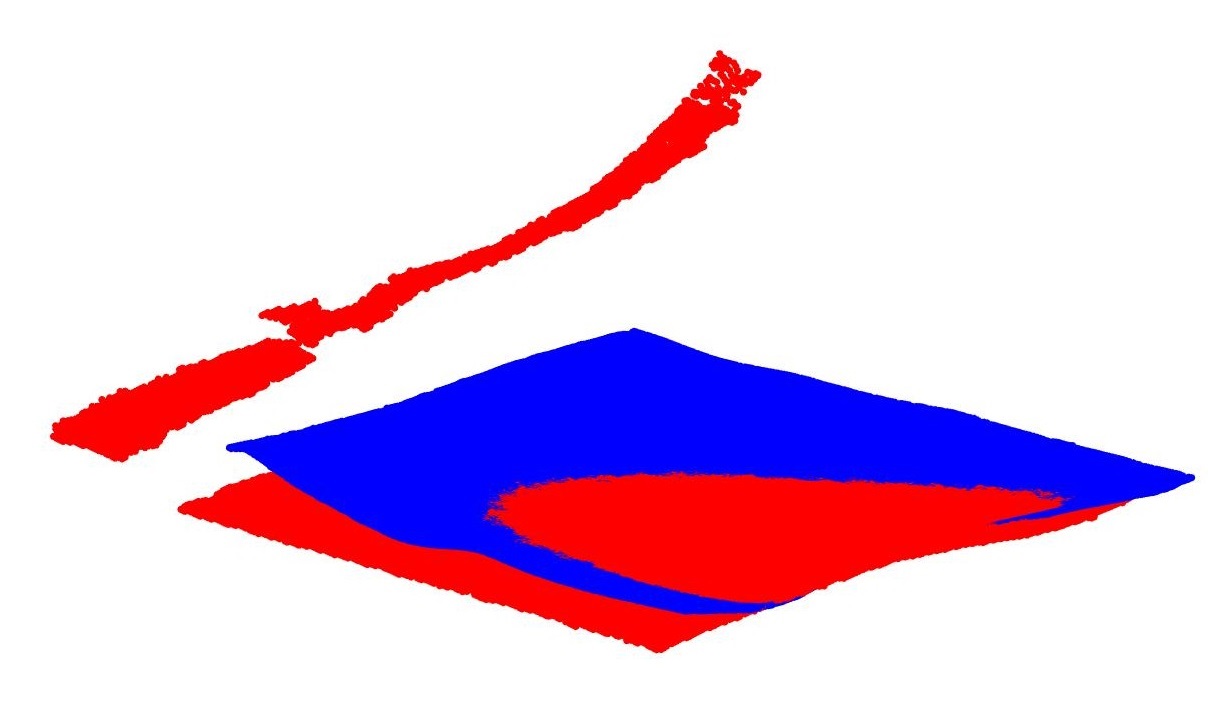}\label{fig:grapha1}}
	\hspace{.1cm}
	\subfloat[\footnotesize{threshold= 100}]{\includegraphics[width=0.31\columnwidth,height=1in]{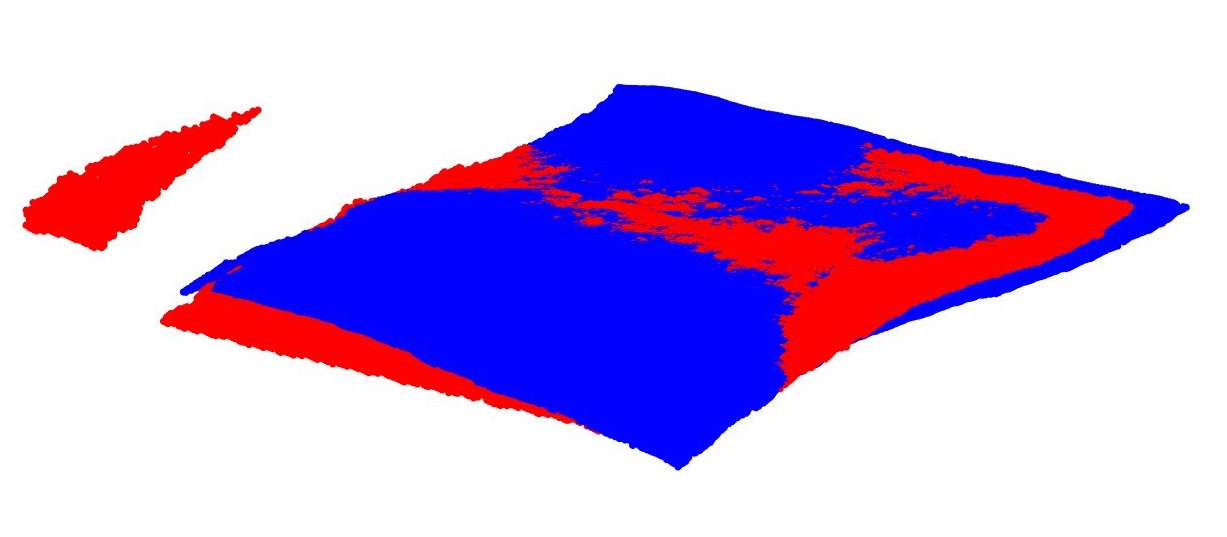}\label{fig:graphb1}}	
	\hspace{.1cm}
	\subfloat[\footnotesize{threshold= 70}]{\includegraphics[width=0.31\columnwidth,height=1in]{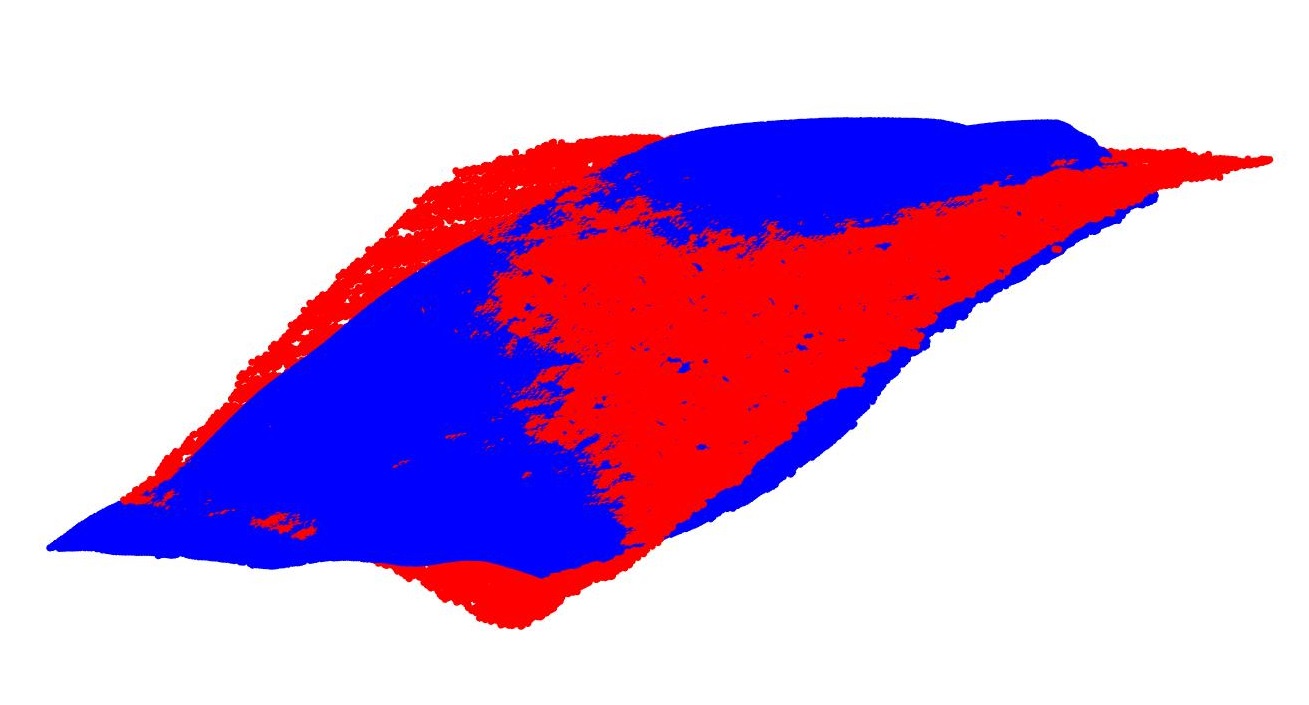}\label{fig:graphb1}}\\
	\caption{An example in 75-th frame (Paper dataset). Ground truth (red), including outliers, and reconstructed results (blue).}
	\label{fig:graphd1xs}
	\endminipage
\end{figure}

We note that these unsupervised methods differ from conventional learning, where a network is trained once and applied to new datasets. They use the networks to represent the deformation, which is more similar to the parametrization. Hence, they need to compute again for each new dataset, and the computational time will be discussed, shown in Section~\ref{s651}.

\subsection{Further Analysis}
\subsubsection{Network Result}
In this part, we evaluate the performance of our DenseNet-169-based U-Net with data augmentation against two alternatives: (1) the same network without data augmentation, and (2) a similar U-Net with a ResNet-50 encoder~\cite{55add}. In our approach, we augment the normalized normal field by injecting Gaussian noise with a mean of 0.2 and a variance of 0.1—equivalent to approximately 20\% noise. The noise-free subset is identical for both the augmented and non-augmented models. We report the training losses over 100 epochs in Fig.~\ref{fig:xxxzssq}, illustrating the convergence behavior of each model. Corresponding prediction results for a testing data point are shown in Fig.~\ref{fig:xxxzssq1}. For the clean testing dataset, the relative prediction error ratios (with our method as baseline) are 0.9571 for the non-augmented DenseNet model and 2.6069 for the ResNet-50 U-Net, where a smaller ratio indicates better accuracy. For the noisy testing dataset, the error ratios are 6.7416 (non-augmented) and 2.3806 (ResNet-50). Our trained network achieves comparable prediction accuracy to the version without data augmentation and significantly outperforms U-Nets with ResNet encoders on noise-free data. Under noisy conditions, our method demonstrates superior robustness compared to the other approaches.
 
\begin{figure}[!ht]
	\begin{center}
		\includegraphics[width=0.9\linewidth]{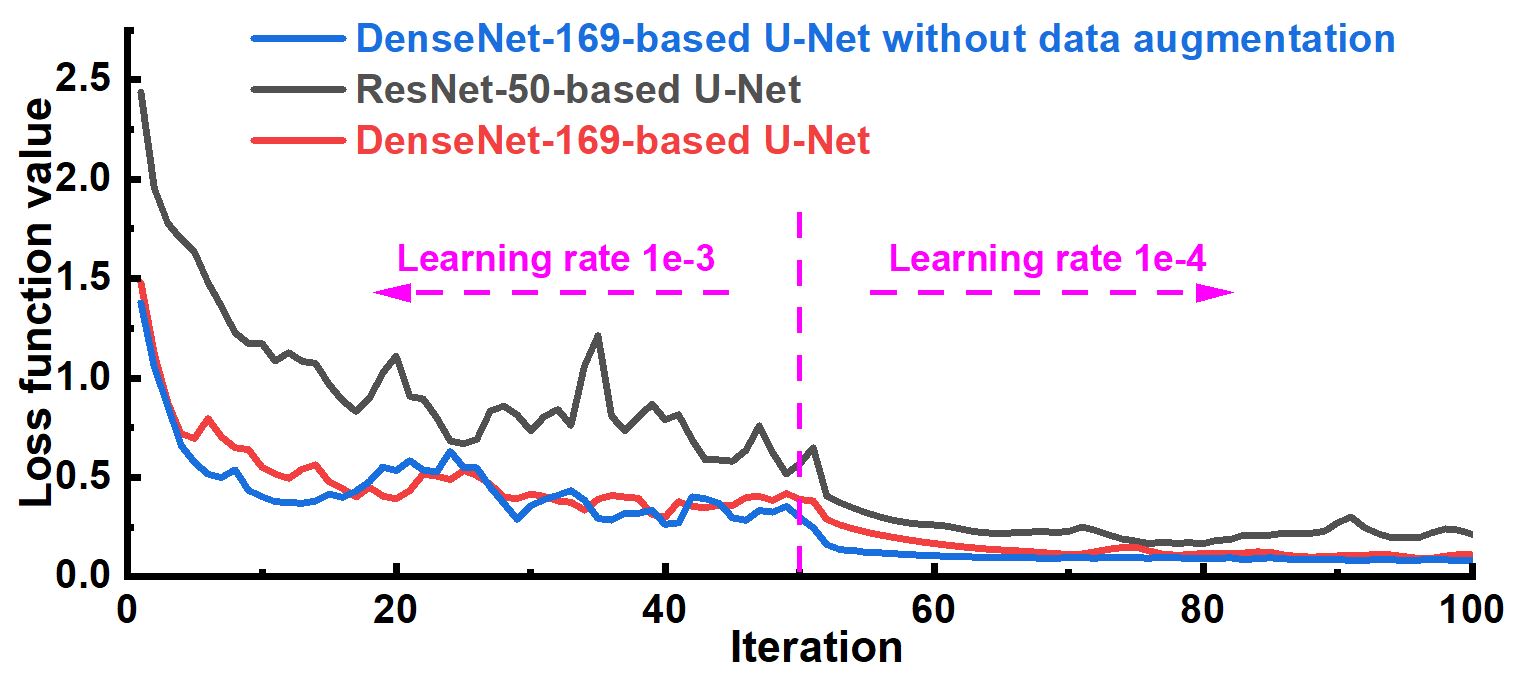}
	\end{center}
	\caption{Training loss over epochs.}
	\label{fig:xxxzssq}
\end{figure}

\begin{figure}[!ht]
	\begin{center}
		\includegraphics[width=0.7\linewidth]{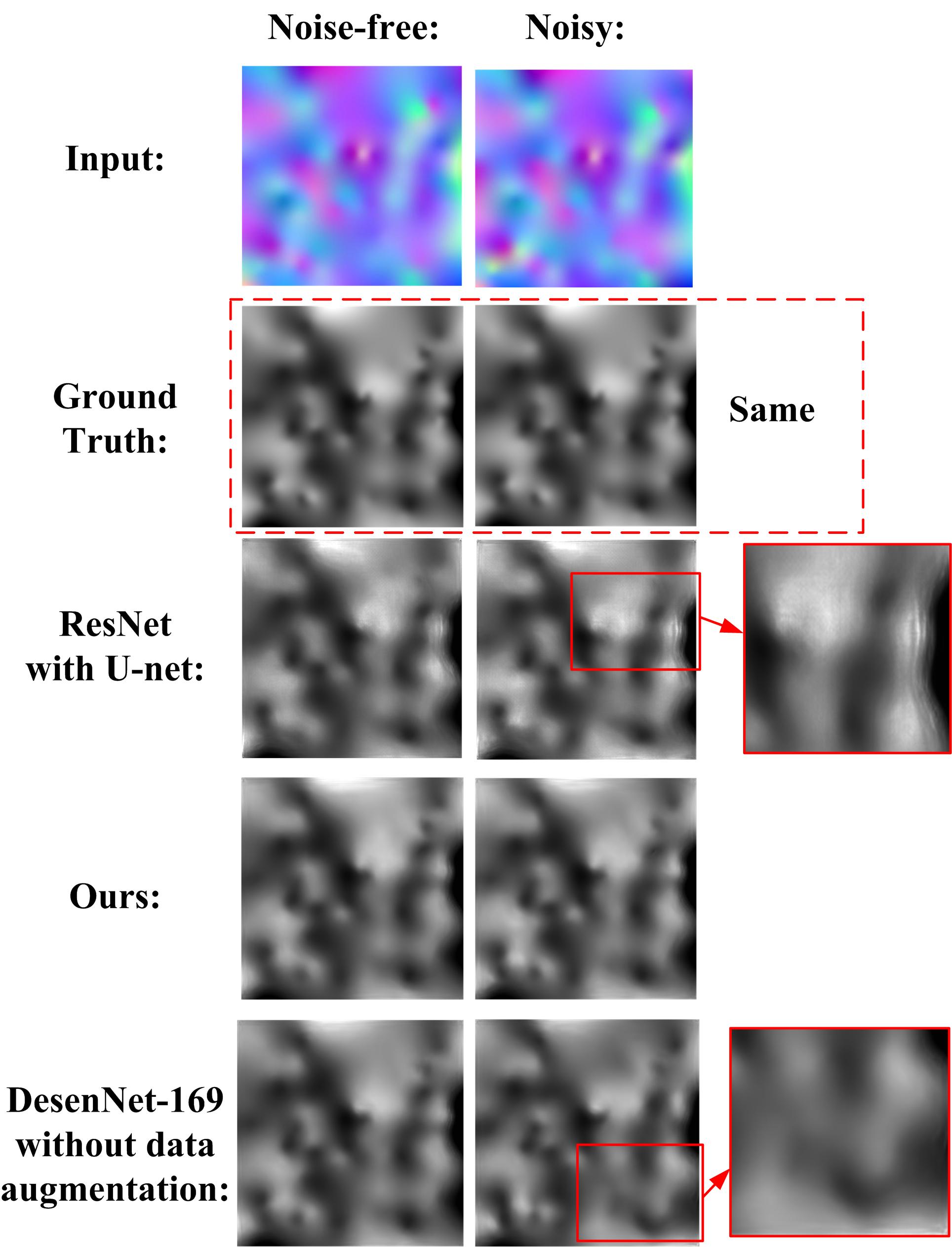}
	\end{center}
	\caption{Network prediction results.}
	\label{fig:xxxzssq1}
\end{figure}

\subsubsection{Computational Time Study}\label{s651}
Based on the hardware setting (Dell G5-5500 laptop with NVIDIA GeForce RTX 2070 MQ), we compare our method with the \textbf{infP} method, the \textbf{Diff} method, the \textbf{Ch17} method, and the \textbf{Go20} method in the running ability. The running times with the given image warps (without considering the texturing) using MATLAB R2021a are shown in  Table~\ref{NRSfM1}. The result shows that our method has an acceptable computational ability. The \textbf{infP} method is the fastest one and the \textbf{Ch17} method shows a much slower running ability compared with the others. In fact, in our implementation, Our results indicate that the \textbf{SDP17} commonly takes 2-3 times longer than the computational time of the \textbf{Ch17} method. For the learning method \textbf{Si20}, it costs more than 200 seconds/frame to recover the Paper (Dense) dataset within $5\times10^4$ iterations. \textbf{LASR} and \textbf{ViSER} cost about 45.10 seconds/frame and 50.36 seconds/frame to recover the Paper (Dense) dataset, respectively. \textbf{Ours} costs 11.50 seconds/frame using a downsample Paper (Dense) dataset with 2000 features/frame and reaching a better performance.

\begin{table}[!ht]
	\small
	\caption{Computational ability comparison (s/frame)}
	\label{NRSfM1}
	\begin{center}
		\begin{tabular}{|c|c|c|c|c|c|c|}
			\hline
			Datasets&{$N_m/N_p$}&\textbf{Ours}&\textbf{infP}&\textbf{Diff}&\textbf{Ch17}&\textbf{Go20}\\
			\hline
			Flag&30/250&2.08&0.41 &3.41&101.10& 0.33
			\\
			\hline
			Rug&159/300&1.69&0.31 &4.50&$>200$&1.48\\
			\hline
			Rug$\times 5$$^{**}$&791/300&1.61 &0.20  &4.66 &$\gg200$ &2.06\\
			\hline
			Rug$\times 10$&1581/300&1.62 &0.18  &4.67 &$\gg200$&2.89\\
			\hline
		\end{tabular}\\
		\end{center}
		\footnotesize{$^{**}$Rug$\times \star$ means to directly combine $\star$ Rug datasets with the same feature tracking as the normal Rug dataset.~{Our results indicate that our code can deal with the middle size dataset, like Rug$\times 50$ with 7901 frames and 300 feature/frame, and keep the similar running speed (1.398 s/frame) based on common desktop configuration with i7-13700k and 32 GB RAM.}}
\end{table}

 All these results reported in Table~\ref{NRSfM1} limit the recovered result to be sparse point features, instead of a dense point cloud. If we would like to perform the colorful 3D dense reconstruction, using our network combining the depth recovery and the texturing, our method will show some computational advantages. It commonly takes about 0.6s to recover the 3D dense point cloud with texture for one frame. For the other sparse methods, the depth recovery from the normal field and the texturing cost about 1.4s to operate one frame~\cite{3} with almost the same accuracy.

To further evaluate the effectiveness of the proposed parallel strategy, we compare our framework with a baseline version in which all parallel operations are disabled. The comparison is conducted using the Rug$\times\star$ datasets, where $\star = 2,~4,~6,~8,10$. Since the parallelization does not affect the estimation results, we report only the ratio of computational time between the parallel and non-parallel versions, as illustrated in Fig.~\ref{fig:xxxzq}. All experiments were conducted on a desktop with an Intel Core i7-13700K CPU, 32 GB of RAM, and MATLAB R2021a. The results clearly demonstrate that the parallel strategy significantly enhances computational efficiency on a multi-core CPU platform.

\begin{figure}[!ht]
	\begin{center}
    \includegraphics[width=0.8\linewidth]{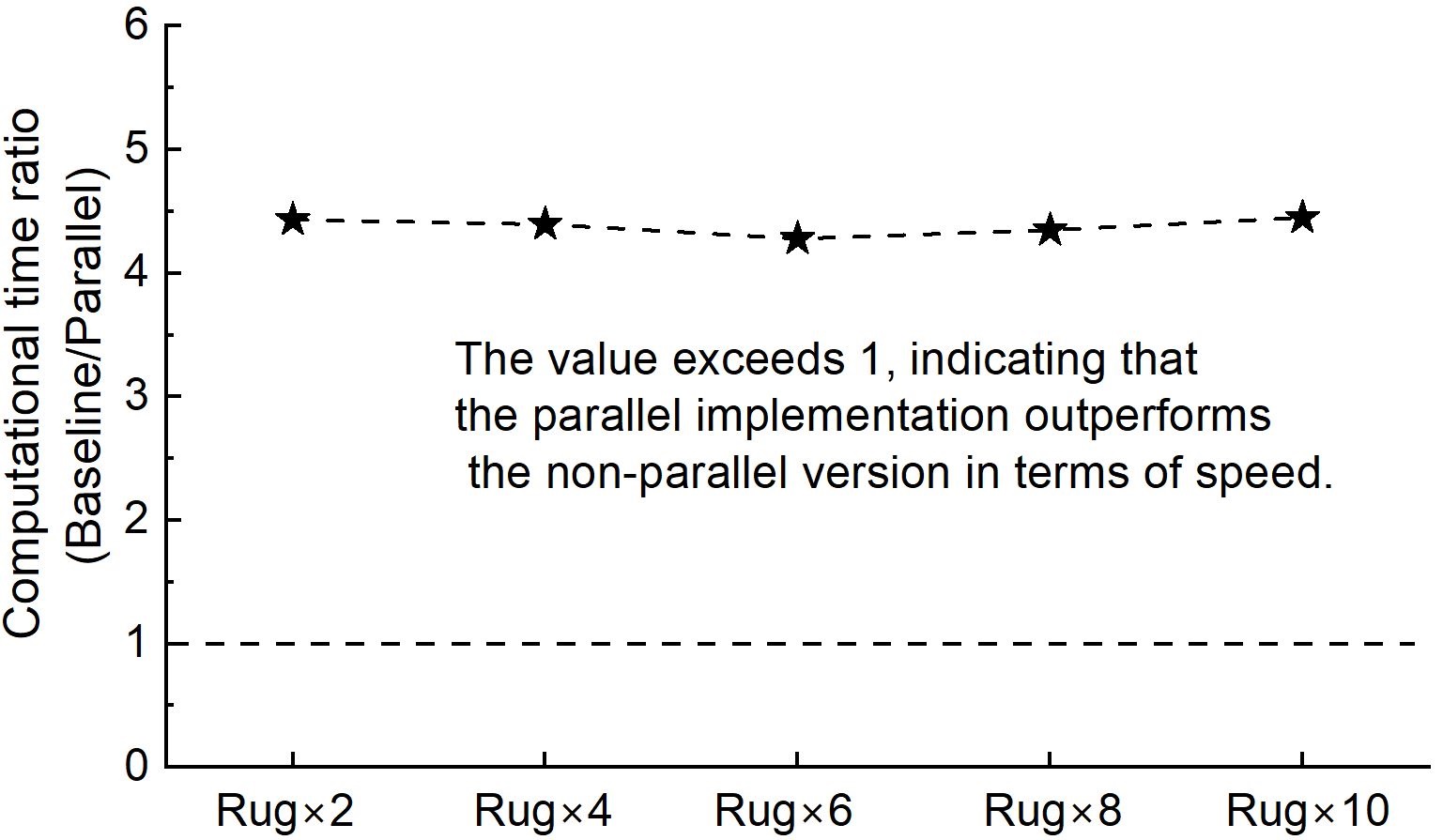}
	\end{center}
	\caption{Computational time ratio between two versions.}
	\label{fig:xxxzq}
\end{figure}

\subsubsection{Ablation Study}
In order to verify the importance of each component of our proposed method, we complete an ablation study in this sub-section. In our method, the computation of second-order derivatives helps to improve the reconstruction quality of surfaces, especially with high curvature.
The application of conformal scale helps to estimate a good scale of reconstruction with wider application scope. The separable framework improves the robustness and avoids the local minima. Table~\ref{table_4} shows the ablation study of each component.

\begin{table}[h]
	\footnotesize
	\renewcommand{\arraystretch}{1.3}
	\caption{Ablation study of each component}
	\label{table_4}
	\centering
	\begin{tabular}{|c||c|c|c|c|}
		\hline
		\% 3D and Shape errors& \multicolumn{2}{ c |}{Synthetic 1}& \multicolumn{2}{ c |}{T-shirt}\\
		\hline
		\textbf{Our framework} &0.81\%&$6.08^{\circ}$&2.24\%&$17.52^{\circ}$
		\\
		\hline
		Without Pre-Step&7.82\%&$41.35^{\circ}$&4.52\%&$44.83^{\circ}$\\
		\hline
		Without separable opt. &5.54\%&$33.50^{\circ}$&3.41\%&$27.87^{\circ}$\\
		\hline
		Without conformal  &4.53\%&$29.06^{\circ}$&2.73\%&$21.42^{\circ}$\\
		\hline
		Without second  &1.22\%&$10.47^{\circ}$&2.60\%&$20.48^{\circ}$\\
		\hline
	\end{tabular}
\end{table}

In the table, ``Our framework" means to use the complete Con-NRSfM method shown in Algorithm~\ref{alg:euclid}. ``Without Pre-Step" means to ignore the operations Pre-Step in Section~\ref{s52} and the others following the complete method. ``Without separable opt.” means joint optimization of all the variables, instead of the operations Step 1-4 in Section~\ref{s52}. ``Without conformal” means the method setting all the conformal scales as 1. ``Without second”  represents the method setting all the second-order derivatives as 0. Our experiments show that the introduction of the separable optimization framework, the conformal scale, and second-order term all help a lot in improving the result accuracy. Synthetic 1 dataset benefits more from the conformal scale than the near-isometric T-shirt dataset, as it follows the conformal deformation.

\subsubsection{Parameter Sensitivity}
In this section, we evaluate the parameter sensitivity of our algorithm to demonstrate its robustness with respect to key parameter settings. As a representative example, we examine the impact of the optimization iteration count $N_t$ used in Step 1, 2, and 4. Since our method follows the idea of alternating optimization, it is not necessary to fully optimize each variable in every iteration. Instead, the design encourages updating variables in a small number of steps while keeping others fixed.  We vary $N_t$ from 1 to 5 and report the corresponding results on the T-shirt and Flag datasets in Fig.~\ref{fig:xxxzqsss}. The results indicate that our method maintains stable performance across different values of $N_t$, highlighting its robustness to this parameter. To strike a balance between computational efficiency and accuracy, we set the default value of $N_t$ to 3 in our experiments\footnote{This experiment was conducted on a new desktop equipped with an Intel Core i7-13700K CPU. As a result, very small numerical differences may exist compared to the results presented in Section~\ref{62}.}.

\begin{figure}[!ht]
	\begin{center}
	\includegraphics[width=0.75\linewidth,height=1.6in]{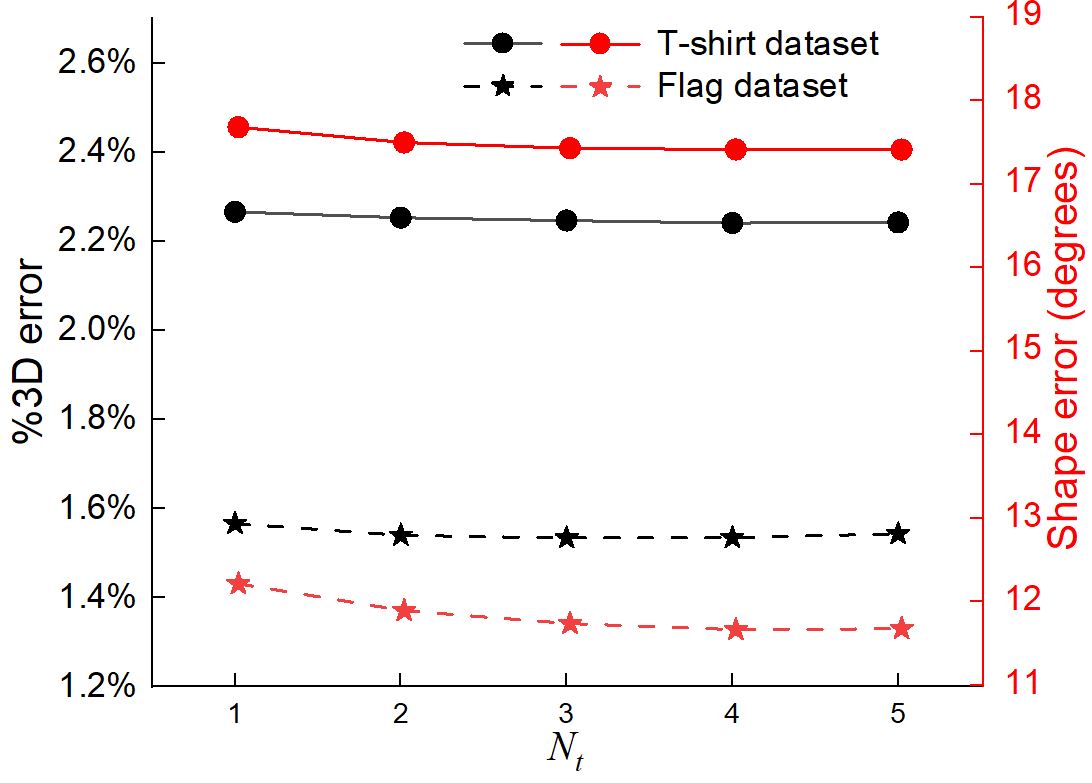}
	\end{center}
	\caption{Estimation results for two datasets with different $N_t$.}
	\label{fig:xxxzqsss}
\end{figure}

\subsubsection{Stability and Failure Cases}
Because of the accurate image warp, the point-wise graph optimization, and the separable optimization framework, our complete method shown in Algorithm~\ref{alg:euclid} is  very stable and robust with the initial value (Pure zeros/ones or random values), noises (\figurename~\ref{fig:graph10and11}), and missing data (Table~\ref{t1zz}). In all our experiments, the initial values of the variables are all zeros/ones without any specific prior. Our method can handle most NRSfM scenarios and the failure case did not happen in all our experiments. Because our method relies on image correspondences to formulate reconstruction constraints, our method will fail when the existing image matching methods fail to detect the image feature correspondences, which is also challenging for other feature-based approaches. 

\subsubsection{Limitations}
As mentioned in Section~\ref{s3b}, a key assumption of this work is that the deforming object is modeled as a Riemannian manifold, implying a continuous and smooth surface without discontinuities or self-intersections. Consequently, this method is not well-suited for reconstructing highly non-surface or large-scale deforming objects, such as complex human limb motions (Human3.6M~\cite{56}) or articulated hand movements (InterHand2.6M~\cite{57}). It is important to note that, although the point-wise formulation improves flexibility and robustness in handling sparse correspondences and moderate deformations, it still fundamentally assumes that all features lie on a continuous surface. That is, each feature is expected to have a valid mapping to a local patch of the underlying manifold. As a result, if the features originate from  regions with large-area self-intersections, occluded folds, or detached components, this assumption breaks down, and the reconstruction may become very inaccurate. Another limitation stems from the current implementation, which relies heavily on MATLAB's TRR-based optimization tools. This reliance results in memory constraints when processing large datasets such as Rug$\times 300$. The implementation handles medium-sized datasets like Rug$\times 50$ effectively, achieving comparable computational performance to other SOTA methods (Table~\ref{NRSfM1}). However, for larger datasets, a practical way is to divide the data into smaller subsets and perform reconstruction separately on each to avoid memory overflow.

\section{Conclusion and future work}
This paper introduces a novel theoretical framework, Con-NRSfM, for addressing conformal NRSfM problems. The framework leverages the rotational invariance of connections to express conformal constraints without relying on assumptions about surface geometry. The proposed formulation is solved using selected image warps, a separable parallel graph optimization strategy, and a self-supervised convolutional network. Our method has been extensively tested on a wide range of synthetic and real datasets featuring diverse baseline viewpoints and deformations. The results demonstrate that it outperforms SOTA methods in both accuracy and robustness. 

This paper marks an initial step toward applying differential geometric modeling in NRSfM. We aim to extend the framework to address broader deformation cases, including topological changes and complex surface dynamics. Future work will focus on key challenges such as faster reconstruction, novel constraints, denser outputs, and self-occlusion handling. Furthermore, we plan to integrate advanced pose estimation techniques, including learning-based and keypoint-tracking methods, to develop a comprehensive online visual deformable SLAM system. In parallel, 3D Gaussian Splatting (3DGS)–based SLAM~\cite{58} has surged in popularity because, by enforcing photometric consistency across (nearly) all pixels rather than sparse keypoints, it leverages far more of the image signal and typically achieves higher accuracy. Future work will introduce the novel local physical constraints into the 4DGS SLAM system~\cite{59}, combining image rendering with assumed physical information to enhance accuracy.

\section*{Acknowledgments}
We are grateful to Dr. Jose Lamarca for his support with~\cite{8} and datasets in Section~\ref{s64}, and to the reviewers for their constructive comments.

\section{Proof of Claim 1}\label{App_1}
\begin{proof}
	Let's consider the composite function $\Phi_1=\Psi_{21}\circ\Phi_2\circ\eta_{12}$, for the locally diffeomorphic mapping $\Psi_{21}$,  we have $(\bar{\bm{e}}_1^*,\bar{\bm{e}}_2^*)=\Upsilon\mathbf{R}(\bm{e}_1^*,\bm{e}_2^*),~\Upsilon=\mbox{diag}(\lambda_1,\lambda_2,\lambda_3)$. For the cross basis $\bar{\bm{e}}_3^*=\bar{\bm{e}}_1^*\times \bar{\bm{e}}_2^*$, we have: $\bar{\bm{e}}_3^*=\Upsilon^*\mathbf{R}\bm{e}_3^*,~\Upsilon^*=\mbox{diag}(\lambda_2\lambda_3,\lambda_1\lambda_3,\lambda_1\lambda_2)$. In short, the moving frame of the composite mapping $\Phi_1$, can be written as:
	\begin{equation}\label{eq11}
		\begin{aligned}
			&\overline{E}(\Phi_1)=E(\Psi_{21}\circ\Phi_2\circ\eta_{12})=(\bar{\bm{e}}_1^*,\bar{\bm{e}}_2^*,\bar{\bm{e}}_3^*)\\
			&=(\Upsilon\mathbf{R} (\bm{e}_1^*,\bm{e}_2^*),\Upsilon^*\mathbf{R}{\bm{e}}_3^*)=\Upsilon\mathbf{R}( \bm{e}_1^*, \bm{e}_2^*,\bm{C} \bm{e}_3^*)\\
			&=\Upsilon\mathbf{R}(\bm{e}_1,\bm{e}_2,\bm{C}\bm{e}_3)\mathbf{J}_{\eta 3\times 3},
		\end{aligned}
	\end{equation}
	where $\bm{C}=\left(\Upsilon\mathbf{R}\right)^{-1}\Upsilon^*$ $\mathbf{R}$.
	
	Based on the definition of the connections, we apply the derivative to the moving frame. The differential vectors are:
	\begin{equation}\label{eq12}
		\begin{aligned}
			&(d\bar{\bm{e}}_1^*,d\bar{\bm{e}}_2^*,d\bar{\bm{e}}_3^*)=\Upsilon\mathbf{R}(d\bm{e}_1, d\bm{e}_2,\bm{C} d\bm{e}_3)\mathbf{J}_{\eta 3\times 3}\\&+\Upsilon\mathbf{R}(e_1, e_2,\bm{C} \bm{e}_3)\mbox{diag}(d\mathbf{J}_{\eta_{12}},d\det(\mathbf{J}_{\eta_{12}})).
		\end{aligned}
	\end{equation}
	
	For the left hand side of the differential vectors~\eqref{eq12}, we can represent the connection by the coordinate axises. Let's consider
    \begin{equation}
		\begin{aligned}
			&d\bar{\bm{e}}_j^*=\frac{\partial~\bar{\bm{e}}_j^*}{\partial~\bar{u}}d\bar{u}+\frac{\partial~\bar{\bm{e}}_j^*}{\partial~\bar{v}}d\bar{v}\\
            &\frac{\partial~\bar{\bm{e}}_j^*}{\partial~\bar{u}}=\Gamma^1_{j1}(\Phi_1)\bar{\bm{e}}_1^*+\Gamma^2_{j1}(\Phi_1)\bar{\bm{e}}_2^*+\Gamma^3_{j1}(\Phi_1)\bar{\bm{e}}_3^*\\
            &\frac{\partial~\bar{\bm{e}}_j^*}{\partial~v}=\Gamma^1_{j2}(\Phi_1)\bar{\bm{e}}_1^*+\Gamma^2_{j2}(\Phi_1)\bar{\bm{e}}_2^*+\Gamma^3_{j2}(\Phi_1)\bar{\bm{e}}_3^*,
		\end{aligned}
	\end{equation}we have:
    \begin{equation}
		\begin{aligned}
			&d\bar{\bm{e}}_j^*=(\Gamma^1_{j1}(\Phi_1)\bar{\bm{e}}_1^*+\Gamma^2_{j1}(\Phi_1)\bar{\bm{e}}_2^*+\Gamma^3_{j1}(\Phi_1)\bar{\bm{e}}_3^*)d\bar u\\&+(\Gamma^1_{j2}(\Phi_1)\bar{\bm{e}}_1^*+\Gamma^2_{j2}(\Phi_1)\bar{\bm{e}}_2^*+\Gamma^3_{j2}(\Phi_1)\bar{\bm{e}}_3^*)d\bar v.
		\end{aligned}
	\end{equation}
	
    Rewrite $\bar \omega^i_j=\Gamma^i_{j1}(\Phi_1)d\bar u+\Gamma^i_{j2}(\Phi_1)d\bar v$, we have: $d\bar{\bm{e}}_j^*=\bar\omega^1_j\bar{\bm{e}}_1^*+\bar\omega^2_j\bar{\bm{e}}_2^*+\bar\omega^3_j\bar{\bm{e}}_3^*$. Then, we can get:
	
	\begin{equation}\label{eq_98}
		\begin{aligned}
			&(d\bar{\bm{e}}_1^*,d\bar{\bm{e}}_2^*,d\bar{\bm{e}}_3^*)=(\bar{\bm{e}}_1^*,\bar{\bm{e}}_2^*,\bar{\bm{e}}_3^*)\left(\bar{\omega}^i_j\right)_{3\times 3}\\&=\Upsilon\mathbf{R}(\bm{e}_1,\bm{e}_2,\bm{C}\bm{e}_3)\mathbf{J}_{\eta 3\times 3}\left(\bar{\omega}^i_j\right)_{3\times 3}.
		\end{aligned}
	\end{equation}
	
	Similarly, for the right hand side of the differential vectors~\eqref{eq_98}, based on $(d\bm{e}_1,d\bm{e}_2,d\bm{e}_3)=(\bm{e}_1,\bm{e}_2,\bm{e}_3)\left({\omega}^i_j\right)_{3\times 3}$, we also have:
	\begin{equation}\label{eq_99}
		\begin{aligned}
			&(d\bm{e}_1, d\bm{e}_2,\bm{C} d\bm{e}_3)\\&=\left((\bm{e}_1,\bm{e}_2,\bm{e}_3)\left({\omega}^i_j\right)_{3\times 2},\bm{C}(\bm{e}_1,\bm{e}_2,\bm{e}_3)\left({\omega}^i_3\right)_{3\times 1}\right)\\
			&=({E}(\Phi_2)\left({\omega}^i_j\right)_{3\times 2},~\bm{C}{E}(\Phi_2)\left({\omega}^i_j\right)_{3\times 1}),
		\end{aligned}
	\end{equation} where $\left(\cdot\right)_{\bullet\times \circ}$, $\left({\omega}^i_j\right)_{3\times 2}$ and $\left({\omega}^i_3\right)_{3\times 1}$ denote the sub-matrices of ${\omega}^i_j$ corresponding to its first two columns and its last column, respectively.
    
	Introducing \eqref{eq_98} and \eqref{eq_99} into \eqref{eq12}, we have:
	\begin{equation}\label{eq_97}
		\begin{aligned}
			&\Upsilon\mathbf{R}(\bm{e}_1,\bm{e}_2,\bm{C} \bm{e}_3)\mathbf{J}_{\eta 3\times 3}\left(\bar{\omega}^i_j\right)_{3\times 3}=\Upsilon\mathbf{R}\\&({E}(\Phi_2)\left({\omega}^i_j\right)_{3\times 2},\bm{C}{E}(\Phi_2)\left({\omega}^i_3\right)_{3\times 1})\mathbf{J}_{\eta 3\times 3}\\&+\Upsilon\mathbf{R}(\bm{e}_1, \bm{e}_2,\bm{C} \bm{e}_3)\mbox{diag}(d\mathbf{J}_{\eta_{12}},d\det(\mathbf{J}_{\eta_{12}})).
		\end{aligned}
	\end{equation}
	
	Let's delete $\Upsilon\mathbf{R}$, we have:
	\begin{equation}\label{eq_100}
		\begin{aligned}
			&\left(\bar{\omega}^i_j\right)_{3\times 3}=f_1(\Upsilon,\mathbf{R},\left({\omega}^i_j\right)_{3\times 3})+\\&\mathbf{J}_{\eta 3\times 3}^{-1}\mbox{diag}(d\mathbf{J}_{\eta_{12}},d\det(\mathbf{J}_{\eta_{12}})),
		\end{aligned}
	\end{equation}
	where, 
	\begin{equation}
		\begin{aligned}
			&f_1(\Upsilon,\mathbf{R},\left({\omega}^i_j\right)_{3\times 3})=\mathbf{J}_{\eta 3\times 3}^{-1}(\bm{e}_1, \bm{e}_2,\bm{C} \bm{e}_3)^{-1}\\&({E}(\Phi_2) \left({\omega}^i_j\right)_{3\times 2}, 
	\bm{C}{E}(\Phi_2)\left({\omega}^i_3\right)_{3\times 1})\mathbf{J}_{\eta 3\times 3}.
		\end{aligned}
	\end{equation}
	
	Introduce {$\omega^i_j=\Gamma^i_{j1}(\Phi_2)du+\Gamma^i_{j2}(\Phi_2)dv$, $\bar \omega^i_j=\Gamma^i_{j1}(\Phi_1)$ $d\bar u+\Gamma^i_{j2}(\Phi_1)d\bar v$} into the above equation~\eqref{eq_100}.	We can find that the function is greatly related to the conformal scale (matrix form) $\Upsilon$ and the rotation matrix $\mathbf{R}$, which means that the connections of the mappings $\Phi_2$ and $\Phi_2\circ\eta_{12}$ do not have the invariance property. So Claim~1 is proved.
\end{proof}


\section{Proof of Theorem 1}\label{App_2}

	\begin{proof}
 
	Let's consider the composite function $\Phi_1=\Psi_{21}\circ\Phi_2\circ\eta_{12}$, for the conformal case $\lambda_1=\lambda_2=\lambda_3=\lambda$, the moving frame of the composite mapping, $\Phi_1$, can be written as: $\overline{E}(\Phi_1)=\mathbf{R}E(\Phi_2){\mathbf{J}_{\eta}}_{3\times 3}\Lambda$. Its differential vector is:
	\begin{equation}\label{eq_79}
		\begin{aligned}
			&(d\bar{\bm{e}}_1^*,d\bar{\bm{e}}_2^*,d\bar{\bm{e}}_3^*)=\mathbf{R}(d\bm{e}_1,d\bm{e}_2,d\bm{e}_3){\mathbf{J}_{\eta}}_{3\times 3}\Lambda\\&+\mathbf{R}E(\Phi_2)\mbox{diag}(d\mathbf{J}_{\eta_{12}},d\det(\mathbf{J}_{\eta_{12}}))\Lambda.
		\end{aligned}
	\end{equation}
	
	For the left-hand side of the differential vectors~\eqref{eq_79}, we can represent the connection by the coordinate axes. Let's consider:
	\begin{equation}\label{eq_79add}
		\begin{aligned}
		&d\bar{\bm{e}}_j^*=\frac{\partial~\bar{\bm{e}}_j^*}{\partial~\bar{u}}d\bar{u}+\frac{\partial~\bar{\bm{e}}_j^*}{\partial~\bar{v}}d\bar{v},\\
		&\frac{\partial~\bar{\bm{e}}_j^*}{\partial~\bar{u}}=\Gamma^1_{j1}(\Phi_1)\bar{\bm{e}}_1^*+\Gamma^2_{j1}(\Phi_1)\bar{\bm{e}}_2^*+\Gamma^3_{j1}(\Phi_1)\bar{\bm{e}}_3^*,\\
	&\frac{\partial~\bar{\bm{e}}_j^*}{\partial~v}=\Gamma^1_{j2}(\Phi_1)\bar{\bm{e}}_1^*+\Gamma^2_{j2}(\Phi_1)\bar{\bm{e}}_2^*+\Gamma^3_{j2}(\Phi_1)\bar{\bm{e}}_3^*,
		\end{aligned}
	\end{equation}
we have:
	\begin{equation}\label{eq_79add1}
		\begin{aligned}
		&d\bar{\bm{e}}_j^*=(\Gamma^1_{j1}(\Phi_1)\bar{\bm{e}}_1^*+\Gamma^2_{j1}(\Phi_1)\bar{\bm{e}}_2^*+\Gamma^3_{j1}(\Phi_1)\bar{\bm{e}}_3^*)d\bar u\\&+(\Gamma^1_{j2}(\Phi_1)\bar{\bm{e}}_1^*+\Gamma^2_{j2}(\Phi_1)\bar{\bm{e}}_2^*+\Gamma^3_{j2}(\Phi_1)\bar{\bm{e}}_3^*)d\bar v.
		\end{aligned}
	\end{equation}
	
	 Rewrite:
	\begin{equation}\label{eq_79add1}
		\begin{aligned}
		&\bar \omega^i_j=\Gamma^i_{j1}(\Phi_1)d\bar u+\Gamma^i_{j2}(\Phi_1)d\bar v,
		\end{aligned}
	\end{equation}
	we have: 
	\begin{equation}\label{eq_79add2}
		\begin{aligned}
		d\bar{\bm{e}}_j^*=\bar\omega^1_j\bar{\bm{e}}_1^*+\bar\omega^2_j\bar{\bm{e}}_2^*+\bar\omega^3_j\bar{\bm{e}}_3^*.
		\end{aligned}
	\end{equation}

	Further, based on:
	\begin{equation}\label{eq_79add2}
		\begin{aligned}
		(d\bm{e}_1,d\bm{e}_2,d\bm{e}_3)&=(\bm{e}_1,\bm{e}_2,\bm{e}_3)({\omega}^i_j)_{3\times 3}=E(\Phi_2)({\omega}^i_j)_{3\times 3},\\
		(d\bar{\bm{e}}_1^*,d\bar{\bm{e}}_2^*,d\bar{\bm{e}}_3^*)&={E}(\Psi_{21}\circ\Phi_2\circ\eta_{12})(\bar{\omega}^i_j)_{3\times 3}\\&=(\bar{\bm{e}}_1^*,\bar{\bm{e}}_2^*,\bar{\bm{e}}_3^*)(\bar{\omega}^i_j)_{3\times 3}=\overline{E}(\Phi_1)(\bar{\omega}^i_j)_{3\times 3},
		\end{aligned}
	\end{equation}
and \eqref{eq7}, the previous equation~\eqref{eq_79} can be written as:
	\begin{equation}\label{eq_29}
		\begin{aligned}
			&\mathbf{R}E(\Phi_2){\mathbf{J}_{\eta}}_{3\times 3}\Lambda\left(\bar{\omega}^i_j\right)_{3\times 3}
			=\mathbf{R}E(\Phi_2)\left({\omega}^i_j\right)_{3\times 3}{\mathbf{J}_{\eta}}_{3\times 3}\Lambda\\&+\mathbf{R}E(\Phi_2)\mbox{diag}(d\mathbf{J}_{\eta_{12}},d\det(\mathbf{J}_{\eta_{12}}))\Lambda.
		\end{aligned}
	\end{equation}
	
	Multiplying $E(\Phi_2)^{-1}\mathbf{R}^{-1}$ on both sides, because of:
	\begin{equation}\label{eq_29add}
		\begin{aligned}
			&\bar \omega^i_j=\Gamma^i_{j1}(\Phi_2)du+\Gamma^i_{j2}(\Phi_2)dv,\\
			&\bar \omega^i_j=\Gamma^i_{j1}(\Phi_1)d\bar u+\Gamma^i_{j2}(\Phi_1)d\bar v,\\
			&\mbox{diag}(d\mathbf{J}_{\eta_{12}},d\det(\mathbf{J}_{\eta_{12}}))=\mbox{diag}(\frac{\partial \mathbf{J}_{\eta_{12}}}{\partial \bar{u}},\\&\frac{\partial \det(\mathbf{J}_{\eta_{12}})}{\partial \bar{u}})d\bar{u}+\mbox{diag}(\frac{\partial \mathbf{J}_{\eta_{12}}}{\partial \bar{v}},\frac{\partial \det(\mathbf{J}_{\eta_{12}})}{\partial \bar{v}})d\bar{v},\\
		\end{aligned}
	\end{equation}
and the orthogonality of $d\bar{u}$ and $d\bar{v}$, we have: 
	\begin{equation}\label{eq_20}
		\begin{aligned}
			&(1)~{\mathbf{J}_{\eta}}_{3\times 3}\Lambda\Gamma^i_{j1}(\Phi_1)=\Gamma^i_{j1}(\Phi_2)\frac{\partial u}{\partial \bar{u}}{\mathbf{J}_{\eta}}_{3\times 3}\Lambda+\\&\Gamma^i_{j2}(\Phi_2)\frac{\partial v}{\partial \bar{u}}{\mathbf{J}_{\eta}}_{3\times 3}\Lambda+\mbox{diag}(\frac{\partial \mathbf{J}_{\eta_{12}}}{\partial \bar{u}},\frac{\partial \det(\mathbf{J}_{\eta_{12}})}{\partial \bar{u}})\Lambda,\\
			&(2)~{\mathbf{J}_{\eta}}_{3\times 3}\Lambda\Gamma^i_{j2}(\Phi_1)=\Gamma^i_{j1}(\Phi_2)\frac{\partial u}{\partial \bar{v}}{\mathbf{J}_{\eta}}_{3\times 3}\Lambda+\\&\Gamma^i_{j2}(\Phi_2)\frac{\partial v}{\partial \bar{v}}{\mathbf{J}_{\eta}}_{3\times 3}\Lambda+\mbox{diag}(\frac{\partial \mathbf{J}_{\eta_{12}}}{\partial \bar{v}},\frac{\partial \det(\mathbf{J}_{\eta_{12}})}{\partial \bar{v}})\Lambda.
		\end{aligned}
	\end{equation}
	
	So Theorem~\ref{t2} is proved.
\end{proof}

\section{Proof of Corollary 1}\label{App_3}

\begin{proof}
	As an example, because the deduction process of the first equation of the connection~(7) is very similar to the second one, we only show the deduction process of the first one. For the conformal deformation $\Psi_{21}$, writing $\Gamma^i_{j1}(\Phi_1)=(\bar{\bm{T}}^1_{kl})$, $\Gamma^i_{j1}(\Phi_2)=({\bm{T}}^1_{kl})$, and $\Gamma^i_{j2}(\Phi_2)=({\bm{T}}^2_{kl}),~k,l=1,2$ as $2\times 2$ block matrices, we can re-write \eqref{eq_861} in Theorem~\ref{t2} as the block matrix formulation:
	\begin{equation}\label{eq_22}
		\begin{aligned}
			&\left(\begin{array}{cc}
				\mathbf{J}_{\eta_{12}}&\bm{0}\\\bm{0}&\det(\mathbf{J}_{\eta_{12}})\end{array}\right)\left(\begin{array}{cc}
				\lambda\bm{I}_{2\times 2}&\bm{0}\\\bm{0}&\lambda^2
			\end{array}\right)\left(\begin{array}{cc}
				\bar{\bm{T}}^1_{11}&\bar{\bm{T}}^1_{12}\\\bar{\bm{T}}^1_{21}&\bar{\bm{T}}^1_{22}
			\end{array}\right)\\&=\left(\begin{array}{cc}
				{\bm{T}}^1_{11}&{\bm{T}}^1_{12}\\{\bm{T}}^1_{21}&{\bm{T}}^1_{22}
			\end{array}\right)\frac{\partial u}{\partial \bar{u}}{\mathbf{J}_{\eta}}_{3\times 3}\left(\begin{array}{cc}
				\lambda\bm{I}_{2\times 2}&\bm{0}\\\bm{0}&\lambda^2
			\end{array}\right)\\
			&+\left(\begin{array}{cc}
				{\bm{T}}^2_{11}&{\bm{T}}^2_{12}\\{\bm{T}}^2_{21}&{\bm{T}}^2_{22}
			\end{array}\right)\frac{\partial v}{\partial \bar{u}}{\mathbf{J}_{\eta}}_{3\times 3}\left(\begin{array}{cc}
				\lambda\bm{I}_{2\times 2}&\bm{0}\\\bm{0}&\lambda^2
			\end{array}\right)\\
&+\mbox{diag}(\frac{\partial \mathbf{J}_{\eta_{12}}}{\partial \bar{u}},\frac{\partial \det(\mathbf{J}_{\eta_{12}})}{\partial \bar{u}})\mbox{diag}(\lambda,\lambda,\lambda^2)\\
			&\Rightarrow\left(\begin{array}{cc}
				\lambda\mathbf{J}_{\eta_{12}}\bar{\bm{T}}^1_{11}&\lambda\mathbf{J}_{\eta_{12}}\bar{\bm{T}}^1_{12}\\\lambda^2\det(\mathbf{J}_{\eta_{12}})\bar{\bm{T}}^1_{21}&\lambda^2\det(\mathbf{J}_{\eta_{12}})\bar{\bm{T}}^1_{22}\end{array}\right)\\&=\frac{\partial u}{\partial \bar{u}}\left(\begin{array}{cc}
				\lambda{\bm{T}}^1_{11}\mathbf{J}_{\eta_{12}}&\lambda^2{\bm{T}}^1_{12}\det(\mathbf{J}_{\eta_{12}})\\\lambda{\bm{T}}^1_{21}\mathbf{J}_{\eta_{12}}&\lambda^2{\bm{T}}^1_{22}\det(\mathbf{J}_{\eta_{12}})\end{array}\right)\\
			\end{aligned}
\end{equation}
\begin{equation}\nonumber
\begin{aligned}
			&+\frac{\partial v}{\partial \bar{u}}\left(\begin{array}{cc}
				\lambda{\bm{T}}^2_{11}\mathbf{J}_{\eta_{12}}&\lambda^2{\bm{T}}^2_{12}\det(\mathbf{J}_{\eta_{12}})\\\lambda{\bm{T}}^2_{21}\mathbf{J}_{\eta_{12}}&\lambda^2{\bm{T}}^2_{22}\det(\mathbf{J}_{\eta_{12}})\end{array}\right)\\&+\left(\begin{array}{cc}
				\lambda\frac{\partial \mathbf{J}_{\eta_{12}}}{\partial \bar{u}}&\bm{0}\\\bm{0}&\lambda^2\frac{\partial \det(\mathbf{J}_{\eta_{12}})}{\partial \bar{u}}\end{array}\right)\Rightarrow\\
			&\left\{
			\begin{array}{l}
				{\lambda}\mathbf{J}_{\eta_{12}}\bar{\bm{T}}^1_{11}=\frac{\partial u}{\partial \bar{u}}\lambda{\bm{T}}^1_{11}\mathbf{J}_{\eta_{12}}+\frac{\partial v}{\partial \bar{u}}\lambda{\bm{T}}^2_{11}\mathbf{J}_{\eta_{12}}+\lambda\frac{\partial \mathbf{J}_{\eta_{12}}}{\partial \bar{u}}\\ \lambda^2\det(\mathbf{J}_{\eta_{12}})\bar{\bm{T}}^1_{22}=\frac{\partial u}{\partial \bar{u}}\lambda^2{\bm{T}}^1_{22}\det(\mathbf{J}_{\eta_{12}})\\+\frac{\partial v}{\partial \bar{u}}\lambda^2{\bm{T}}^2_{22}\det(\mathbf{J}_{\eta_{12}})+\lambda^2\frac{\partial \det(\mathbf{J}_{\eta_{12}})}{\partial \bar{u}}\\
				\lambda\mathbf{J}_{\eta_{12}}\bar{\bm{T}}^1_{12}=\frac{\partial u}{\partial \bar{u}}\lambda^2{\bm{T}}^1_{12}\det(\mathbf{J}_{\eta_{12}})+\frac{\partial v}{\partial \bar{u}}\lambda^2{\bm{T}}^2_{12}\det(\mathbf{J}_{\eta_{12}})\\
				\lambda^2\det(\mathbf{J}_{\eta_{12}})\bar{\bm{T}}^1_{21}=\frac{\partial u}{\partial \bar{u}}\lambda{\bm{T}}^1_{21}\mathbf{J}_{\eta_{12}}+\frac{\partial v}{\partial \bar{u}}\lambda{\bm{T}}^2_{21}\mathbf{J}_{\eta_{12}}
			\end{array},
			\right.
		\end{aligned}
	\end{equation}
where $\bar{\bm{T}}^1_{11},~{\bm{T}}^1_{11},~{\bm{T}}^2_{11}\in\mathbb{R}_{2\times2}$, $\bar{\bm{T}}^1_{12},~{\bm{T}}^1_{12},~{\bm{T}}^2_{12}\in\mathbb{R}_{2\times1}$, $\bar{\bm{T}}^1_{21},~{\bm{T}}^1_{21},~{\bm{T}}^2_{21}\in\mathbb{R}_{1\times2}$, and  $\bar{\bm{T}}^1_{22},~{\bm{T}}^1_{22},~{\bm{T}}^2_{22}\in\mathbb{R}_{1\times1}$.

	Deleting $\lambda$ or $\lambda^2$, in the final 4 equations, we see that, the front two are not related to the conformal scale $\lambda$, which means that the second-order leading principal minors and the last elements of the connections $\Gamma^i_{jk}(\Phi_1)$ and  $\Gamma^i_{jk}(\Phi_2\circ\eta_{12})$ are invariant. So  Corollary~\ref{c2} is proved.
\end{proof}

\section{Proof of Corollary~\ref{c3}}\label{App_4}

\begin{proof}
Based on the block matrix formulation in  Corollary~\ref{c3}, viewing the conformal scale $\lambda$ as the variable, in~\eqref{eq_22}, we find that the relation between the (1,3)-th, (2,3)-th, (3,1)-th, and (3,2)-th elements of the connections $\Gamma^i_{jk}(\Phi_1)$ and  $\Gamma^i_{jk}(\Phi_2\circ\eta_{12})$ of the mapping $\Phi_1$ and $\Phi_2\circ\eta_{12}$ are linear functions, satisfying:
\begin{equation}\label{eq_13xx}
	\begin{aligned}
		&\lambda\det(\mathbf{J}_{\eta_{12}})\bar{\bm{T}}^1_{21}=\frac{\partial u}{\partial \bar{u}}{\bm{T}}^1_{21}\mathbf{J}_{\eta_{12}}+\frac{\partial v}{\partial \bar{u}}{\bm{T}}^2_{21}\mathbf{J}_{\eta_{12}},\\
		&(\frac{\partial u}{\partial \bar{u}}{\bm{T}}^1_{12}+\frac{\partial v}{\partial \bar{u}}{\bm{T}}^2_{12})\det(\mathbf{J}_{\eta_{12}})\lambda=\mathbf{J}_{\eta_{12}}\bar{\bm{T}}^1_{12},\\
		&\lambda\det(\mathbf{J}_{\eta_{12}})\bar{\bm{T}}^2_{21}=\frac{\partial u}{\partial \bar{v}}{\bm{T}}^1_{21}\mathbf{J}_{\eta_{12}}+\frac{\partial v}{\partial \bar{v}}{\bm{T}}^2_{21}\mathbf{J}_{\eta_{12}},\\
		&(\frac{\partial u}{\partial \bar{v}}{\bm{T}}^1_{12}+\frac{\partial v}{\partial \bar{v}}{\bm{T}}^2_{12})\det(\mathbf{J}_{\eta_{12}})\lambda=\mathbf{J}_{\eta_{12}}\bar{\bm{T}}^2_{12},\\
	\end{aligned}
\end{equation}
For these 8 linear functions, moving other terms into the right-hand side, they can all be written as the following formulation $\lambda-\alpha_i=0,~i=1,\cdots,8$.
Hence, the sum of the squares problem of the conformal scale can be formulated as:
\begin{equation}\label{eq_24aa}
	\begin{aligned}
	\min_{\lambda}	&\sum_{i=1}^{8}(\lambda-\alpha_i)^2,~s.t.~\lambda\in \mathbb{R}.
	\end{aligned}
\end{equation}

Ignoring the constraint $\lambda\in \mathbb{R}$, its closed-form solution is:
\begin{equation}\label{eq_24a}
	\begin{aligned}
	\lambda=\frac{\sum_{i=1}^{n}\alpha_i\pm \sqrt{\left(\sum_{i=1}^{n}\alpha_i\right)^2-n\sum_{i=1}^{n}\alpha_i^2}}{n},~n=8.
	\end{aligned}
\end{equation}

Based on the AM-QM inequality (arithmetic mean, and quadratic mean), we have $\left(\sum_{i=1}^{n}\alpha_i\right)^2-n\sum_{i=1}^{n}\alpha_i^2\leq 0$. Hence, the closest real value solution is $\lambda={\sum_{i=1}^{n}\alpha_i}/{n}$. By introducing all relationships in~\eqref{eq_13xx}, we can get the solution in~\eqref{eq_23add2}. So  Corollary~\ref{c3} is proved.
\end{proof}

\newpage

\vspace{-22pt}

\begin{IEEEbiography}[{\includegraphics[width=1in,height=1.25in,clip,keepaspectratio]{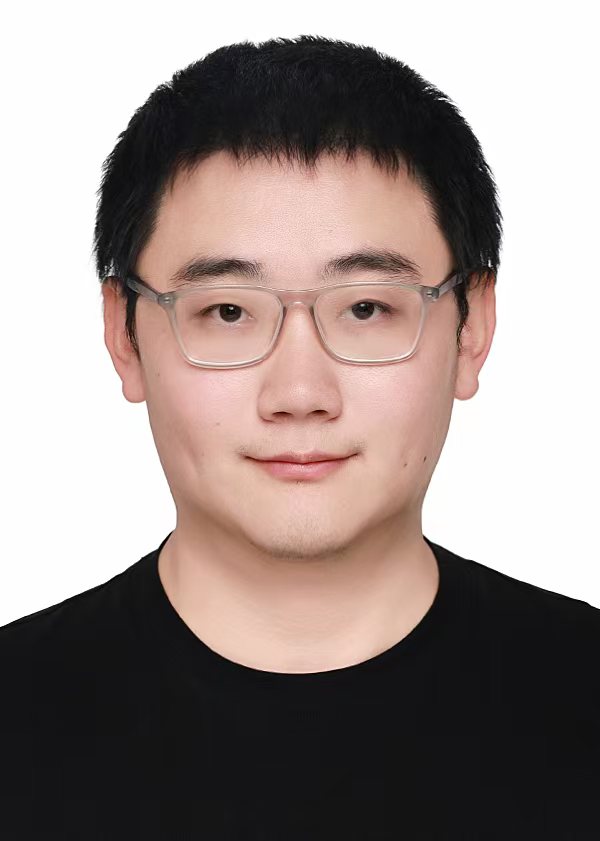}}]{Yongbo Chen} received the Ph.D. degree in Robotics from the Robotics Institute, UTS, Sydney, Australia, in August 2021. From 2021--2024, he was a Postdoctoral Research Fellow in the School of Computing, The Australian National University (ANU), Canberra, Australia. He is now an Associate Professor with the School of Automation and Intelligent Sensing, Shanghai Jiao Tong University, Shanghai, China. His research interests are simultaneous localization and mapping (SLAM), partially observable Markov decision process (POMDP), optimization and learning techniques in Robotics.
\end{IEEEbiography}

\vspace{-22pt}

\begin{IEEEbiography}[{\includegraphics[width=1in,height=1.25in,clip,keepaspectratio]{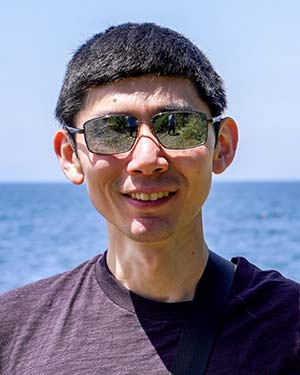}}]{Yanghao Zhang} received the Ph.D.
degree in robotics from UTS,
Sydney, Ultimo, NSW, Australia, in 2022.
From 2022-2024, he was a Postdoctoral Research Fellow
with the College of Engineering and Computer Science,
Australian National University, Canberra, ACT,
Australia. He is now a researcher on Embodied AI and Robotics with the Robotics Institute, UTS, Sydney, Australia. His research interests include visual simultaneous
localization and mapping and deformation
reconstruction.
\end{IEEEbiography}

\vspace{-22pt}

\begin{IEEEbiography}[{\includegraphics[width=1in,height=1.25in,clip,keepaspectratio]{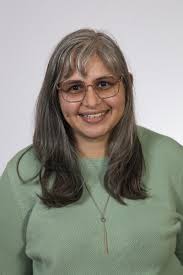}}]{Shaifali Parashar}
 received Ph.D. degree in computer vision from the Universit\'{e} Auvergne,
Clermont-Ferrand, France, in 2017. She was a
Postdoctoral Researcher with Computer Vision Laboratory (CVLab), \'{E}cole Polytechnique F\'{e}d\'{e}rale de Lausanne (EPFL)
in Lausanne, Switzerland. She is currently a
Centre national de la recherche scientifique (CNRS) Research scientist at Institut National des Sciences Appliqu\'{e}es de Lyon (LIRIS, INSA-Lyon).
Her research interests include 3D computer vision
including nonrigid 3D reconstruction, deformation
modeling and deformable SLAM.
\end{IEEEbiography}

\vspace{-22pt}

\begin{IEEEbiography}[{\includegraphics[width=1in,height=1.25in,clip,keepaspectratio]{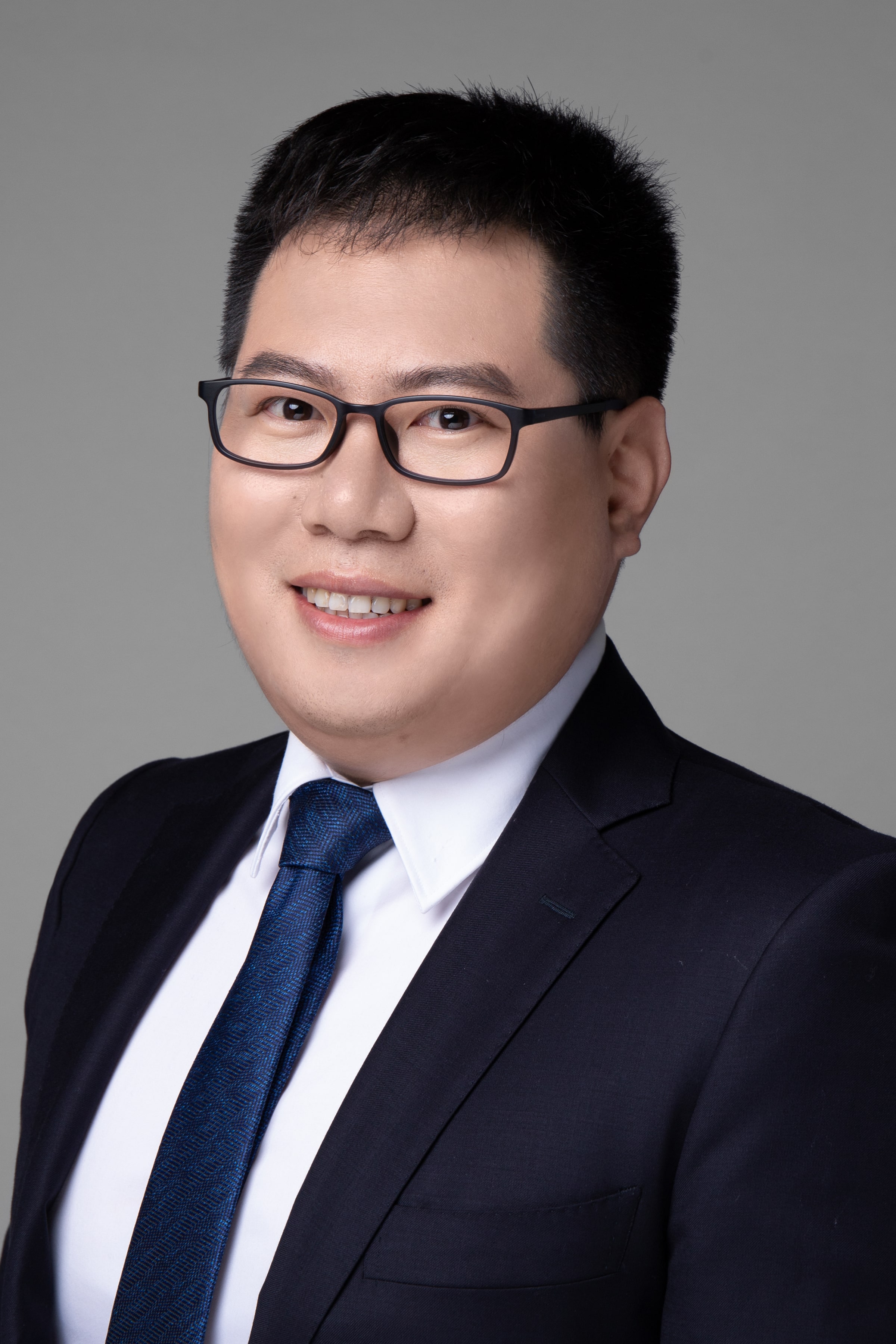}}]
	{Liang Zhao} received the Ph.D.
degree in photogrammetry and remote sensing from
Peking University, Beijing, China, in 2013.
From 2014 to 2016, he worked as a Postdoctoral
Research Associate with the Hamlyn Centre for
Robotic Surgery, Department of Computing, Faculty
of Engineering, Imperial College London, London,
U.K. From 2016 to 2024, he was a Senior Lecturer
 at the UTS Robotics Institute, UTS, Sydney, Australia. He is now a Reader in Robot Systems in the School of Informatics, The University of Edinburgh, United Kingdom. His research
interests include surgical robotics, autonomous robot SLAM, monocular SLAM, aerial photogrammetry, optimization
techniques in mobile robot localization and mapping and image-guided
robotic surgery.
\end{IEEEbiography}

\vspace{-22pt}

\begin{IEEEbiography}[{\includegraphics[width=1in,height=1.25in,clip,keepaspectratio]{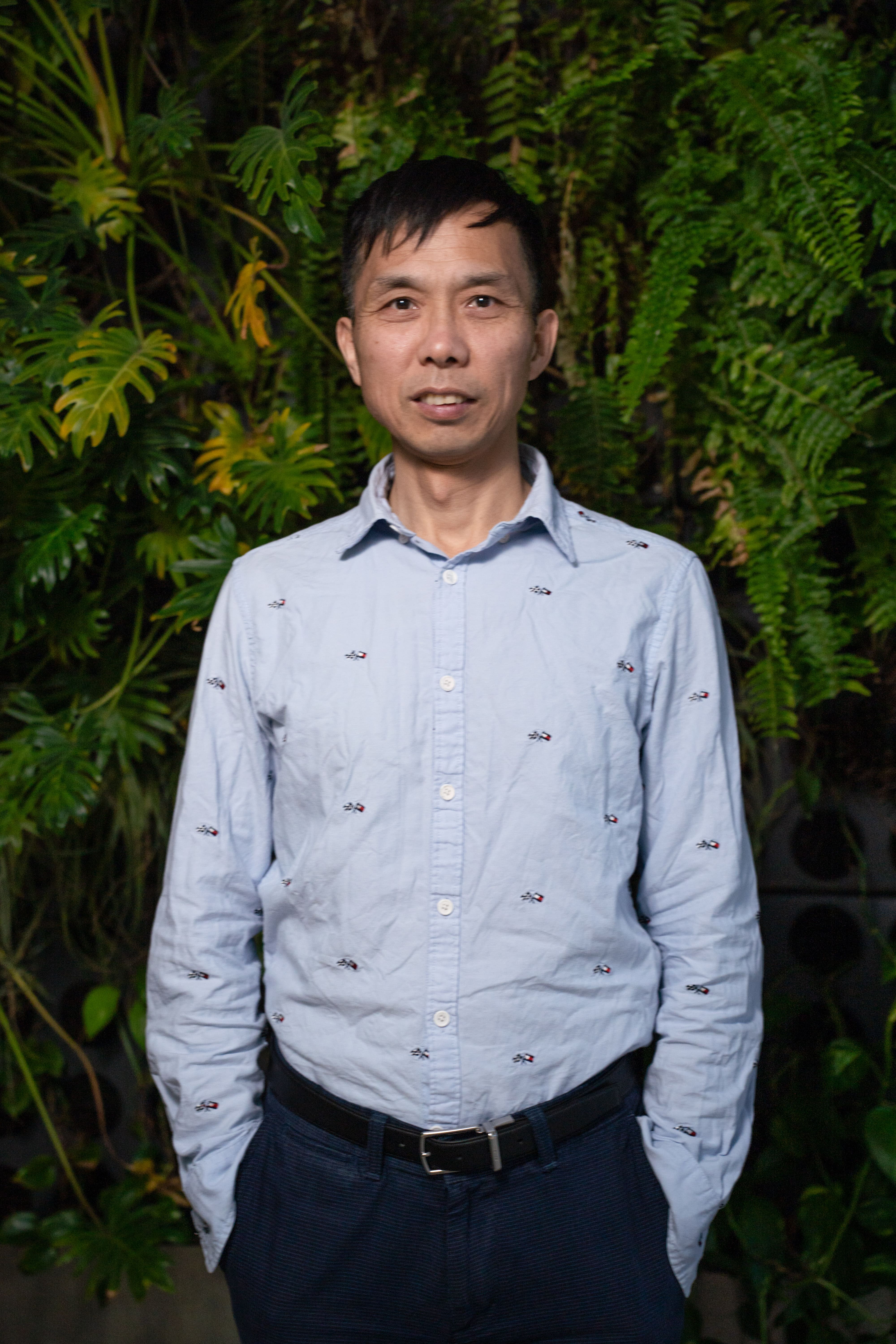}}]{Shoudong Huang} received his Ph.D. in Automatic Control from Northeastern University, Shenyang, China, in 1998. He is currently a Professor and Deputy Director of the UTS Robotics Institute, as well as Deputy Head of School (Research) in the School of Mechanical and Mechatronic Engineering at the University of Technology Sydney. His research interests
include nonlinear system state estimation and control, mobile robot SLAM, and surgical robotics.
\end{IEEEbiography}

\vfill


\begin{thebibliography}{1}
\bibliographystyle{IEEEtran}
\bibitem{1add}
S. Parashar and A. Bartoli, ``3DVFX: 3D video editing using non-rigid structure-from-motion,"  \textit{Eurographics} (Short Papers), pp. 29-32, 2019.
\bibitem{8}
J. Lamarca, S. Parashar, A. Bartoli, and J. M. M. Montiel, ``DefSLAM: Tracking and mapping of deforming scenes from monocular sequences," \textit{IEEE Trans. on Robot.}, vol. 37, no. 1, pp. 291-303, 2021.
\bibitem{8add}
J. J. G. Rodríguez, J. M. M. Montiel, and J. D. Tardós. ``NR-slam: Non-rigid monocular slam." \textit{IEEE Trans. on Robot.}, vol. 40, pp. 4252-4264, 2024.
\bibitem{8add1}
T. Deng, G. Shen, C. Xun, S. Yuan, T. Jin, H. Shen, Y. Wang, J. Wang, H. Wang, D. Wang, and W. Chen, ``Mne-slam: Multi-agent neural slam for mobile robots," in \textit{Proc. IEEE Conf. on Comput. Vis. Patt. Recog. (CVPR)}, 2025, pp. 1485-1494.
\bibitem{1}
L. Torresani, A. Hertzmann, and C. Bregler, ``Nonrigid structure-from-motion: Estimating shape and motion with hierarchical priors," \textit{IEEE Trans. Patt. Anal. Mach. Intell.}, vol. 30, no. 5, pp. 878-892, 2008.
\bibitem{2}
I. Akhter, Y. Sheikh, S. Khan, and T. Kanade, ``Trajectory space: A dual representation for nonrigid structure from motion," \textit{IEEE Trans. Patt. Anal.
	Mach. Intell.}, vol. 33, no. 7, pp. 1442-1456, 2010.
\bibitem{3}
Y. Chen, L. Zhao, Y. Zhang, and S. Huang, ``Dense isometric non-rigid shape-from-motion based on graph optimization and edge selection,"  \textit{IEEE Robot. and Autom. Lett.}, vol. 5, no. 4, pp. 5889-5896, 2020.
\bibitem{4}
A. Chhatkuli, D. Pizarro, T. Collins, and A. Bartoli, ``Inextensible non-rigid structure-from-motion by second-order cone programming," \textit{IEEE Trans. Patt. Anal.
	Mach. Intell.}, vol. 40, no. 10, pp. 2428-2441, 2017.
\bibitem{4add}
A. Sengupta, K. Makki and A. Bartoli, ``Using specularities to boost non-rigid structure-from-motion," in \textit{Proc. IEEE Int. Conf. on Robot.  Auto. (ICRA)}, 2024, pp. 2652-2659.
\bibitem{5}
A. Chhatkuli, D. Pizarro, and A. Bartoli, ``Nonrigid shape-from-motion for isometric surfaces using infinitesimal planarity," in \textit{Proc. British Mach. Vis. Conf. (BMVC)}, 2014, pp. 1-12.
\bibitem{6}
S. Parashar, D. Pizarro, and A. Bartoli, ``Isometric non-rigid shapefrom-motion with riemannian geometry solved in linear time," \textit{IEEE Trans. Patt. Anal.
	Mach. Intell.}, vol. 40, no. 10, pp. 2442-2454, 2018.
\bibitem{7}
S. Parashar, D. Pizarro, and A. Bartoli, ``Local deformable 3d reconstruction with cartan's connections," \textit{IEEE Trans. Patt. Anal.
	Mach. Intell.}, vol. 42, no. 12, pp. 3011-3026, 2021.
\bibitem{10}
S. Parashar, M. Salzmann, and P. Fua, ``Local non-rigid structure-from-motion from diffeomorphic mappings," in \textit{Proc. IEEE Conf. Comput. Vis. Patt. Recog. (CVPR)}, 2020, pp. 2059-2067.
\bibitem{21add}
V. Sidhu, E. Tretschk, V. Golyanik, A. Agudo, and C. Theobalt, ``Neural dense non-rigid structure from motion with latent space constraints," in \textit{Eur. Conf. Comput. Vis (ECCV)}, 2020, pp. 204-222.
\bibitem{13}
A. D. Bue, ``A factorization approach to structure from motion with shape priors," in \textit{Proc. IEEE Conf. on Comput. Vis. Patt. Recog. (CVPR)}, 2008, pp. 1-8.
\bibitem{19}
Y. Dai, H. Li, and M. He, ``A simple prior-free method for non-rigid structure-from-motion factorization," \textit{Int. J. Comput. Vis.}, vol. 107, no. 2, pp. 101-122, 2014.
\bibitem{19add}
S. Graßhof and S. S. Brandt. ``Tensor-based non-rigid structure from motion." in \textit{Proc. of the IEEE/CVF Winter Conf. on Appl. of Comput. Vis. (WACV)}, 2022, pp. 3011-3020.
\bibitem{18}
I. Akhter, Y. Sheikh, S. Khan, and T. Kanade, ``Nonrigid structure from motion in trajectory space," in \textit{Proc.  Neural Inf. Process. Syst. (NIPS)}, 2009, pp. 1-8.
\bibitem{21}
A. Agudo and F. M. Noguer, ``Force-based representation for non-rigid shape and elastic model estimation," \textit{IEEE Trans. Patt. Anal.
	Mach. Intell.}, vol. 40, no. 9, pp. 2137-2150, 2018.
\bibitem{20}
A. Agudo, F. M. Noguer, B. Calvo, and J. Montiel, ``Sequential nonrigid structure from motion using physical priors," \textit{IEEE Trans. Patt. Anal.
	Mach. Intell.}, vol. 38, no. 5, pp. 979-994, 2016.
\bibitem{14}
A. D. Bue, F. Smeraldi, and L. Agapito, ``Non-rigid structure from motion using non- parametric tracking and non-linear optimization," in \textit{Proc. IEEE Conf. on Comput. Vis. Patt. Recog. (CVPR)}, 2004, pp. 1-8.
\bibitem{15}
L. Torresani, A. Hertzmann, and C. Bregler, ``Nonrigid structure-frommotion: Estimating shape and motion with hierarchical priors," \textit{IEEE Trans. Patt. Anal.
	Mach. Intell.}, vol. 30, no. 5, pp. 878-892, 2008.
\bibitem{16}
P. F. U. Gotardo and A. M. Mart\'{i}nez, ``Kernel non-rigid structure from motion," in \textit{IEEE  Int. Conf. Mach. Lear (ICCV)}, 2011, pp. 802-809.
\bibitem{16add}
Y. Wang, D. Xu, W. Huang, X. Ye, and M. Jiang, ``Temporal-aware neural network for dense non-rigid structure from motion." \textit{Electronics}, vol. 12, no. 18, 3942, 2023.
\bibitem{17}
J. Fayad, A. D. Bue, L. Agapito, and P. M. Aguiar, ``Non-rigid structure
from motion using quadratic deformation models," in \textit{ Proc. British Mach. Vis. Conf. (BMVC)}, 2009, pp. 1-11.
\bibitem{56}
D. Novotny, N. Ravi, B. Graham, N. Neverova, and A. Vedaldi. ``C3dpo: Canonical 3d pose networks for non-rigid structure from motion," in \textit{Proc. IEEE Int. Conf. Comput. Vis. (ICCV)}, 2019, pp. 7688-7697.
\bibitem{57}
C. Wang and S. Lucey, ``PAUL: Procrustean autoencoder for unsupervised lifting," in \textit{Proc. IEEE Conf. on Comput. Vis. Patt. Recog. (CVPR)}, 2021, pp. 434-443.
\bibitem{57add}
H. Deng, T. Zhang, Y. Dai, J. Shi, and H. Li, ``Deep non-rigid structure-from-motion: A sequence-to-sequence translation perspective," in \textit{IEEE Trans. Patt. Anal.
	Mach. Intell.}, vol. 46, no. 12, pp. 10814-10828, 2024.
\bibitem{25}
F. Bookstein, ``Principal warps: thin-plate splines and the decomposition of deformations," \textit{IEEE Trans. Patt. Anal. Mach. Intell.}, vol. 11, no. 6, pp. 567-585, 1989.
\bibitem{23}
J. Taylor, A. D. Jepson, and K. N. Kutulakos, ``Non-rigid structure from locally-rigid motion," in \textit{Proc. IEEE Conf. on Comput. Vis. Patt. Recog. (CVPR)}, 2010, pp. 2761-2768.
\bibitem{24}
C. Russell, R. Yu, and L. Agapito, ``Video pop-up: Monocular 3d reconstruction of dynamic scenes," in \textit{Eur. Conf. Comput. Vis. (ECCV)}, 2014, pp. 583-598.
\bibitem{22}
S. Parashar, D. Pizarro, and A. Bartoli. ``Isometric non-rigid shape-from-motion in linear time," in \textit{Proc. IEEE Conf. on Comput. Vis. Patt. Recog. (CVPR)}, 2016, pp. 4679-4687.
\bibitem{9}
J. J. G. Rodríguez, J. Lamarca, J. Morlana, J. D. Tardós, and J. M. Montiel, ``Sd-defslam: Semi-direct monocular slam for deformable and intracorporeal scenes," in \textit{Proc. IEEE Int. Conf. on Robot. Auto. (ICRA)}, 2021, pp. 5170-5177.
\bibitem{9adds}
R. Mur-Artal, J. M. M. Montiel, and J. D. Tardos, ``ORB-SLAM: A versatile and accurate monocular SLAM system," \textit{IEEE Trans. on Robot.}, vol. 31, no. 5, pp. 1147-1163, 2015.
\bibitem{9adds1111}
J. Lamarca and J. M. M. Montiel, ``Camera tracking for SLAM in
deformable maps,” in \textit{Eur. Conf. on Comput. Vis. (ECCV) Workshops}, 2018.
\bibitem{9adds3333}
E. Wang, Y. Liu, J. Xu, and X. Chen, ``Non-rigid scene reconstruction of deformable soft tissue with monocular endoscopy in minimally invasive surgery," \textit{Int. J. Comput. Assist. Radiol. Surg.}, vol. 19, no. 12, pp. 2433-2443, 2024.
\bibitem{31}
S. Parashar, Y. Long, M. Salzmann, and P. Fua, ``A closed-form, pairwise solution to local non-rigid structure-from-motion," \textit{IEEE Trans. Patt. Anal. Mach. Intell.}, vol. 46, no. 11, pp. 7027-7040, 2024.
\bibitem{25a}
P. Ji, H. Li, Y. Dai, and I. Reid, ``“Maximizing rigidity” revisited: A convex programming approach for generic 3D shape reconstruction from multiple perspective views," in \textit{Proc. IEEE Int. Conf. Comput. Vis. (ICCV)}, 2017, pp. 929-937.
\bibitem{26}
N. Sundaram, T. Brox, and K. Keutzer, ``Dense point trajectories by gpu-accelerated large displacement optical flow," in \textit{Eur. Conf. Comput. Vis. (ECCV)}, 2010, pp. 438-451.
\bibitem{27}
D. G. Lowe, ``Object recognition from local scale-invariant features,"  in \textit{Proc. IEEE Int. Conf. Comput. Vis. (ICCV)}, 1999, pp. 1150-1157.
\bibitem{28}
 H. Bay, ``Surf: Speeded up robust features," \textit{Comput. Vis. Image Underst.}, vol. 110, no. 3, pp. 404-417, 2006.
\bibitem{29}
 Y. Chen, S. Huang, L. Zhao, and G. Dissanayake, ``Cramér–rao bounds and optimal design metrics for pose-graph SLAM," \textit{IEEE Trans. on Robot.}, vol. 37, no. 2, pp. 627-641, 2021.
\bibitem{38}
 Y. Chen, K. M. B. Lee, C. Yoo, and R. Fitch, ``Broadcast your
weaknesses: cooperative active pose-graph SLAM for multiple robots,"
\textit{IEEE Robot. and Autom. Lett.}, vol. 5, no. 2, pp. 2200-2007, 2020.
\bibitem{30}
J. M. Lee. Riemannian manifolds: an introduction to curvature. Springer, 1997.
\bibitem{41add}
J. C. Bezdek and R. J. Hathaway, ``Convergence of alternating optimization," \textit{Neural Parallel Sci. Comput.}, vol. 11, no. 4, pp. 351-368, 2003.
\bibitem{32}
I. Alhashim and P. Wonka, ``High quality monocular depth estimation via transfer learning," arXiv preprint arXiv:1812.11941, 2018.
\bibitem{33}
G. Huang, Z. Liu, L. van der Maaten, and K. Q. Weinberger, ``Densely connected convolutional networks," in \textit{Proc. IEEE Conf. on Comput. Vis. Patt. Recog. (CVPR)}, 2017, pp. 2261-2269.
\bibitem{34}
J. Deng, W. Dong, R. Socher, L. J. Li, K. Li, and F. F. Li,
``Imagenet: A large-scale hierarchical image database," in 
\textit{Proc. IEEE Conf. on Comput. Vis. Patt. Recog. (CVPR)}, 2009, pp. 248–255.
\bibitem{35}
S. Ioffe and C. Szegedy, ``Batch normalization: Accelerating
deep network training by reducing internal covariate shift," in \textit{Int. Conf. Mach. Lear (ICML)}, 2015, pp. 448-456.
\bibitem{36}
H. Fu, M. Gong, C. Wang, N. Batmanghelich, and D. Tao,
``Deep ordinal regression network for monocular depth estimation,"in 
\textit{Proc. IEEE Conf. on Comput. Vis. Patt. Recog. (CVPR)}, 2018, pp. 2002-2011.
\bibitem{41_new}
J. Matt, Absolute Orientation - Horn's method, MATLAB Central File Exchange, 2023.
\bibitem{41xxxxxx}
D. Henrion, J. B. Lasserre, and J. Löfberg, ``GloptiPoly 3: moments, optimization and semidefinite programming," \textit{Optim. Methods Softw.}, vol. 24, no. 4–5, pp. 761-779.
\bibitem{41}
S. Vicente and L. Agapito, ``Soft inextensibility constraints for template-free non-rigid reconstruction," in 
\textit{Eur. Conf. Comput. Vis. (ECCV)}, 2012, pp. 426-440.
\bibitem{42}
S. H. N. Jensen, M. E. B. Doest, H. Aanæs, and A. Del Bue, ``A benchmark and evaluation of non-rigid structure from motion," 
\textit{Int. J. Comput. Vis.}, vol. 129, no. 4, pp. 882-899, 2021.
\bibitem{43}
S. Parashar, A. Bartoli, and D. Pizarro, ``Robust isometric non-rigid
structure-from-motion," \textit{IEEE Trans. Patt. Anal. Mach. Intell.}, vol. 44, no. 10, pp. 6409-6423, 2022.
\bibitem{44}
M. D. Ansari, V. Golyanik, and D. Stricker, ``Scalable dense monocular
surface reconstruction," in \textit{Int. Conf. on 3D Vis. (3DV)}, 2017, pp. 78-87.
\bibitem{45}
M. Lee, S. Cho, and J. Oh, ``Consensus of non-rigid reconstructions,"
in \textit{Proc. IEEE Conf. on Comput. Vis. Patt. Recog. (CVPR)}, 2016, pp. 4670-4678.
\bibitem{46}
P. Mountney, D. Stoyanov, and G. Z. Yang, ``Three-dimensional tissue deformation recovery and tracking," \textit{IEEE Signal Process. Mag.}, vol. 27, no. 4, pp. 14-24, 2010.
\bibitem{47}
D. Stoyanov, G. P. Mylonas, F. Deligianni, A. Darzi, and G. Z. Yang,
``Soft-tissue motion tracking and structure estimation for robotic assisted mis procedures," in \textit{Proc. Med. Image Comput. Comput.-Assisted Intervention (MICAI)}, 2005, pp. 139-146.
\bibitem{48}
C. Campos, R. Elvira, J. J. G. Rodríguez, J. M. M. Montiel, and J. D. Tardós, ``ORB-SLAM3: An accurate open-source library for visual, visual-inertial and multi-map SLAM," \textit{IEEE Trans. on Robot.}, vol. 37, no. 6, pp. 1874-1890, 2021.
\bibitem{add_new_1}
G. Yang, D. Sun, V. Jampani, D. Vlasic, F. Cole, H. Chang, D. Ramanan, W.T. Freeman, and C. Liu, ``Lasr: Learning articulated shape reconstruction from a monocular video," in \textit{Proc. IEEE Conf. on Comput. Vis. Patt. Recog. (CVPR)}, 2021, pp. 15980-15989.
\bibitem{add_new_2}
G. Yang, D. Sun, V. Jampani, D. Vlasic, F. Cole, C. Liu, and D. Ramanan, ``Viser: Video-specific surface embeddings for articulated 3d shape reconstruction," in \textit{Proc. Adv. Neural Inf. Process. (NeurIPS)}, 34, pp. 19326-19338.
\bibitem{51}
I. Akhter, Y. Sheikh, S. Khan, and T. Kanade, ``Trajectory space: A dual representation for nonrigid structure from motion," \textit{IEEE Trans. Patt. Anal. Mach. Intell.}, vol. 33, no. 7, pp. 1442-1456, 2011.
\bibitem{52}
M. Paladini, A. Del Bue, J. Xavier, L. Agapito, M. Stosi\'{c}, and M. Dodig, ``Optimal metric projections for deformable and articulated structure-from-motion," \textit{Int. J. Comput. Vis.}, vol. 96, no. 2, pp. 252-276, 2012.
\bibitem{53}
Y. Dai, H. Deng, and M. He, ``Dense non-rigid structure-from-motion made easy – a spatial-temporal smoothness based solution," in \textit{Proc. Int. Conf. on Image Processing (ICIP)}, 2017, pp. 4532-4536.
\bibitem{54}
S. Kumar, A. Cherian, Y. Dai, and H. Li, ``Scalable dense non-rigid structure-from-
motion: A grassmannian perspective," in \textit{Proc. IEEE Conf. on Comput. Vis. Patt. Recog. (CVPR)}, 2018, pp. 254-263.
\bibitem{55}
S. Kumar, ``Jumping manifolds: Geometry aware dense non-rigid structure from
motion," in \textit{Proc. IEEE Conf. on Comput. Vis. Patt. Recog. (CVPR)}, 2019, pp. 5346-5355.
\bibitem{55add}
K. He, X. Zhang, S. Ren, and J. Sun, ``Deep residual learning for image recognition," in \textit{Proc.  IEEE Conf. on Comput. Vis. Patt. Recog. (CVPR)}, 2016, pp. 770-778.
\bibitem{56}
C. Ionescu, D. Papava, V. Olaru and C. Sminchisescu, ``Human3.6M: Large scale datasets and predictive methods for 3D human sensing in natural environments," \textit{IEEE Trans. on Patt. Anal. and Mach. Intell.}, vol. 36, No. 7, pp. 1325-1339, 2014.
\bibitem{57}
G. Moon, S. Yu, H. Wen,  and K. M. Lee, ``InterHand2.6M: A dataset and baseline for 3D interacting hand pose estimation from a single RGB image," in \textit{ Eur. Conf. on Comput. Vis. (ECCV)}, 2020, pp. 548-564.
\bibitem{58}
H. Matsuki, R. Murai, P.H. Kelly, and A.J. Davison, ``Gaussian splatting slam," in Proceedings of the IEEE/CVF Conference on Computer Vision and Pattern Recognition (pp. 18039-18048).
\bibitem{59}
Y. Li, Y. Fang, Z. Zhu, K. Li, Y. Ding, F. Tombari, 4D Gaussian Splatting SLAM, arXiv preprint arXiv:2503.16710, 2025.
\end{thebibliography}
\end{document}